\documentclass[letterpaper]{article} 
\usepackage{aaai2026}  
\makeatletter
\@ifpackageloaded{hyperref}{}{%
  \@ifundefined{AAAIhyperrefhack}{%
    \let\old@AAAIerror\PackageError
    \def\PackageError#1#2#3{%
      \def\tempa{#1}\def\tempb{aaai}%
      \ifx\tempa\tempb
        \typeout{(AAAI hyperref warning suppressed for arXiv version)}%
      \else
        \old@AAAIerror{#1}{#2}{#3}%
      \fi
    }%
  }{}%
}
\makeatother
\usepackage{times}  
\usepackage{helvet}  
\usepackage{courier}  
\usepackage[hyphens]{url}  
\usepackage{graphicx} 
\usepackage{natbib}  
\usepackage{caption} 
\usepackage{algorithm}
\usepackage{algorithmic}

\usepackage{amsmath}
\usepackage{amssymb}
\usepackage{mathtools}
\usepackage{amsthm}

\usepackage{microtype}
\usepackage{graphicx}
\usepackage{subfigure}
\usepackage{booktabs} 

\usepackage[utf8]{inputenc} 
\usepackage{url}            
\usepackage{amsfonts}       
\usepackage{nicefrac}       
\usepackage{xcolor}         
\usepackage{natbib}
\usepackage[utf8]{inputenc}
\usepackage{graphicx}  
\usepackage[thicklines]{cancel}
\usepackage{dsfont}
\usepackage{enumitem}
\usepackage{pifont}
\usepackage{multicol}
\usepackage{caption}
\usepackage{tikz}
\usepackage{multirow}
\usepackage{array}
\usepackage{circledsteps}
\usepackage{soul}
\usepackage{subcaption}

\DeclareMathOperator*{\argminA}{arg\,min}
\DeclareMathOperator*{\argmaxA}{arg\,max}

\newtheorem{theorem}{Theorem}
\newtheorem{proposition}{Proposition}
\newtheorem{lemma}{Lemma}
\newtheorem{corollary}{Corollary}
\theoremstyle{definition}
\newtheorem{definition}{Definition}
\newtheorem{assumption}{Assumption}
\theoremstyle{remark}
\newtheorem{remark}{Remark}
\usepackage{newfloat}
\usepackage{listings}
\DeclareCaptionStyle{ruled}{labelfont=normalfont,labelsep=colon,strut=off} 
\floatstyle{ruled}
\newfloat{listing}{tb}{lst}{}
\floatname{listing}{Listing}
\usepackage[hidelinks]{hyperref}

\title{Provably Efficient Multi-Objective Bandit Algorithms under Preference-Centric Customization}
\author{
    Linfeng Cao\textsuperscript{\rm 1},
    Ming Shi\textsuperscript{\rm 2},
    Ness B. Shroff\textsuperscript{\rm 1,\rm 3}
}
\affiliations{
    \textsuperscript{\rm 1}Department of Computer Science and Engineering, The Ohio State University\\
    \textsuperscript{\rm 2}Department of Electrical Engineering, University at Buffalo\\
    \textsuperscript{\rm 3}Department of Electrical and Computer Engineering, The Ohio State University\\

    cao.1378@osu.edu,
    mshi24@buffalo.edu,
    shroff.11@osu.edu
}

\usepackage{bibentry}

\begin{document}

\maketitle

\begin{abstract}
Multi-objective multi-armed bandit (MO-MAB) problems traditionally aim to achieve Pareto optimality. However, real-world scenarios often involve users with varying preferences across objectives, resulting in a Pareto-optimal arm that may score high for one user but perform quite poorly for another. This highlights the need for \emph{customized learning}, a factor often overlooked in prior research.
To address this, we study a \emph{preference-aware} MO-MAB framework in the presence of explicit user preference.
It shifts the focus from achieving Pareto optimality to further optimizing within the Pareto front under preference-centric customization. To our knowledge, this is the first theoretical study of customized MO-MAB optimization with explicit user preferences.
Motivated by practical applications, we explore two scenarios: unknown preference and hidden preference, each presenting unique challenges for algorithm design and analysis. At the core of our algorithms are \emph{preference estimation} and \emph{preference-aware optimization} mechanisms to adapt to user preferences effectively. We further develop novel analytical techniques to establish near-optimal regret of the proposed algorithms. Strong empirical performance confirm the effectiveness of our approach.
\end{abstract}


\section{Introduction}
\label{sec: intro}

Multi-objective multi-armed bandit (MO-MAB) is an important extension of standard MAB~\citep{drugan2013designing}. In MO-MAB problems each arm is associated with a $D$-dimensional reward vector.
In this environment, objectives could conflict, leading to arms that are optimal in one dimension, but suboptimal in others.
A natural solution is utilizing Pareto ordering to compare arms based on their rewards~\citep{drugan2013designing}. Specifically, for any arm $i \in [K]$, if its expected reward $\boldsymbol{\mu}_i$ is non-dominated by that of any other arms, arm $i$ is deemed to be Pareto optimal. 
The set containing all Pareto optimal arms is denoted as Pareto front $\mathcal{O}^*$.
Formally, $\mathcal{O}^* = \{i \mid \boldsymbol{\mu}_j \not\succ \boldsymbol{\mu}_i, \forall j \in [K]\setminus i \}$, where 
$\boldsymbol{u} \succ \boldsymbol{v}$ holds if and only if $\boldsymbol{u}(d) > \boldsymbol{v}(d), \forall d \in [D]$.
The performance is then evaluated by Pareto regret, which measures the cumulative minimum distance between the learner's obtained rewards and rewards of arms within $\mathcal{O}^*$ \citep{drugan2013designing}. 
However, achieving low Pareto regret alone overlooks that users ultimately seek options aligned with their individual preferences.
As the example depicted in Fig. \ref{fig:intro}, given multiple Pareto optimal restaurants, one user may give a higher preference to quality, while another user may give a higher preference to affordibility. 
This means that \emph{user preferences} need to be accounted for in the MO-MAB problem set up in order to choose the right solution on the Pareto front $\mathcal{O}^*$. This is the focus of this paper. 

\begin{figure*}[t]
    \centering    
    \includegraphics[width=0.83\textwidth]{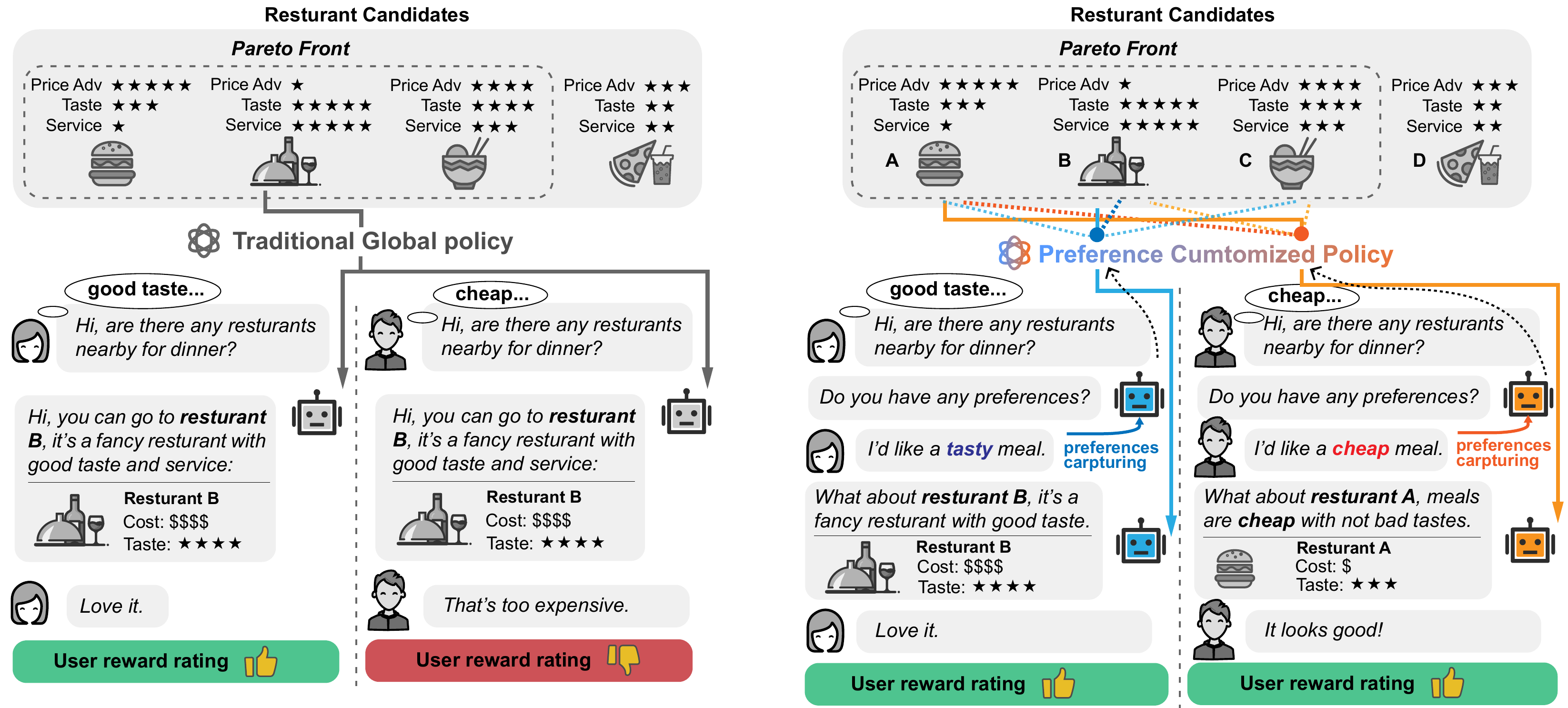}
    \caption{
    A scenario of users interacting with a conversational recommender for restaurant recommendation. 
    (a) Recommender achieves Pareto optimality but receives low rating from user. 
    (b) Recommendations with high users' ratings when the recommender captures users' preferences and aligns optimization with preferences.
    }
    \label{fig:intro}
\end{figure*}

Numerous MO-MAB studies have been conducted but \textbf{most of them achieve Pareto optimality via an arm selection policy that is uniform across all users}, which we refer to as a \emph{global policy}. 
One representative line of research focuses on efficiently estimating the entire Pareto front $\mathcal{O}^*$, and the action is \emph{randomly} chosen on the estimated Pareto front \citep{drugan2013designing, turgay2018multi, lu2019multi, drugan2018covariance, balef2023piecewise}.
Another line of research transforms the $D$-dimensional reward into a scalar using a scalarization function, which targets a specific Pareto optimal arm solution without the costly estimation of entire Pareto front \citet{drugan2013designing, busa2017multi, mehrotra2020bandit, xu2023pareto}. These studies construct the scalarization function in a user-agnostic manner, causing the target arm solution to remain the same across different users (see Appendix~\ref{sec:related_work} for a more detailed related work discussion).
However, \emph{simply achieving Pareto optimality using a global policy may not yield favorable outcomes, since, as mentioned earlier, users often have diverse preferences across different objectives. }
Consider Fig.~\ref{fig:intro}(a), where two users with different preferences interact with a conversational recommender to choose a restaurant based on multi-dimensional rewards (e.g., price, taste, service).
Clearly, restaurants A, B, and C are Pareto optimal, as none of their rewards are dominated by others. 
Previous research using a global policy would either randomly recommend a restaurant from A, B, or C, or select one based on a fixed global criterion to achieve Pareto optimality. 
However, while recommending a restaurant like B might lead to positive feedback from user-1, it is likely to result in a low reward rating from user-2, who prefers an economical meal, since restaurant B is expensive.
In contrast, Fig.~\ref{fig:intro}(b) illustrates that when the system accurately captures user preferences (e.g., user-1 prefers a tasty meal, while user-2 prefers a cheap meal), it can select options more likely to receive positive reward ratings from both users.
\emph{Therefore, we argue that optimizing MO-MAB should be customized based on the user preferences rather than solely aiming for Pareto optimality with a global policy.}

To fill this gap, we introduce a formulation of MO-MAB problem, where each user is associated with a $D$-dimensional \emph{preference vector}, with each element representing the user's preference for the corresponding objective. 
Formally, in each round $t$, user incurs a stochastic preference $\boldsymbol{c}_t \!\in\! \mathbb{R}^D$.
The player selects an arm $a_t$ and observes a stochastic reward $\boldsymbol{r}_{a_t,t} \!\in\! \mathbb{R}^D$.
We define the scalar \emph{overall-reward} as the inner product of arm reward $\boldsymbol{r}_{a_t,t}$ and user preference $\boldsymbol{c}_t$. The learner's goal is to maximize the overall-reward accrued over a given time horizon.
We term this problem as \emph{Preference-Aware MO-MAB} (PAMO-MAB).

While interactive user modeling and customized optimization cross multiple objectives present promising experimental results in some areas including recommendation \citep{xie2021personalized}, ranking \citep{wanigasekara2019learning}, and more \citep{reymond2024interactively}, there are no theoretical studies on MO-MAB customization under explicit user preferences. 
Particularly, two open questions remain:
\emph{(1) how to develop provably efficient algorithms for customized optimization under different preference structure (e.g., unknown preference, hidden preference)?}
\emph{(2) how does the additional user preferences impact the overall performance?}

Our contributions are summarized as follows.
\begin{itemize}[leftmargin=*]
\item
We make the first effort to address the open questions above. 
Motivated by real applications, we consider PAMO-MAB under two preference structures: unknown preference with feedback, and unknown preference without feedback (hidden preference), with tailored algorithms that are proven to achieve near-optimal regret in each case.
These approaches is built on a designed general algorithmic backbone that introduces two key components: preference estimation and preference-aware optimization, to enable effective learning and decision-making under preference-centric customization. 
The expressions of our results are in an explicit form that capture a clear dependency on preference.
\emph{To the best of our knowledge, this is the first work that explicitly showcases the fundamental impact of user preference in the regret optimization of MO-MAB problems.}

\item 
For the general hidden preference case, we propose a novel near-optimal algorithm PRUCB-HP that addresses the unique challenges of hidden PAMO-MAB with two key designs: (1) A weighted least squares-based hidden preference learner, with weights set as the inverse squared $\ell_2$-norm of reward observations, to resolve the random mapping issue caused by random preferences, and (2) A \emph{dual-exploration} policy with novel bonus design to balance the trade-off between \emph{local exploration} for identifying better reward arms and \emph{global exploration} for refining preference learning.
Additionally, we show that the unknown preference case can be viewed as a special instance of the hidden setting. A simplified variant, PRUCB-UP, naturally emerges under this setup, reusing the same backbone with reduced uncertainty and simplified estimation, while still achieving near-optimal regret cross all users.

\item 
Extensive experiments consistently validate the effectiveness of our algorithms in estimating preferences and rewards online, as well as in optimizing the overall reward.
\end{itemize}

\section{Problem Formulation}

We consider MO-MAB with $N$ users, $K$ arms and $D$ objectives.
At each round $t \in [T]$, each user $n \in [N]$ is presented with an arm set $\mathcal{A}_t^n \subseteq [K]$, which may differ across users and time. The learner chooses an arm $a_t^n$ for user $n$ and observes a stochastic $D$-dimensional \emph{reward vector} $\boldsymbol{r}_{a_t^n,t} \in \mathcal{R} \subseteq \mathbb{R}^D$, which we refer to as \emph{reward}.
For the reward, we make the following standard assumption:

\begin{assumption}
[Bounded stochastic reward]
\label{assmp: all_1}
For $i \in [K], t \in [T], d \in [D]$, each reward entry $\boldsymbol{r}_{i,t}(d)$ is independently drawn from a \textbf{fixed} but \textbf{unknown} distribution with mean $\boldsymbol{\mu}_{i}(d)$ and variance $\sigma_{r,i,d}^2$, satisfying
$\boldsymbol{r}_{i,t} (d) \in [0,1]$, and $\sigma_{r,i,d}^2 \in [\sigma^2_{r \downarrow}, \sigma^2_{r \uparrow}] $, where $\sigma^2_{r \downarrow}, \sigma^2_{r \uparrow} \in \mathbb{R}^{+}$.
\end{assumption}

\noindent\textbf{User preferences.}
At each round $t$, we consider each user $n$ to be associated with a stochastic $D$-dimensional \emph{preference vector} $\boldsymbol{c}_t^n \in \mathcal{C} \subseteq \mathbb{R}^D$, indicating the user preferences across the $D$ objectives. 
We refer to this vector as \emph{preference} for short. Specifically, we make the following assumptions:

\begin{assumption}[Bounded stochastic preference]
\label{assmp: all_2}
For $t \in [T], d \in [D], n \in [N]$, each preference entry $\boldsymbol{c}_t^n(d)$ is independently drawn from a fixed distribution \textbf{(either known or unknown)} with mean $\boldsymbol{\overline{c}}^n(d)$ and variance $\sigma_{c^{n},d}^2$, satisfying
$\boldsymbol{c}^{n}(d) \geq 0$, $\Vert \boldsymbol{c}_t^{n} \Vert_1 \leq \delta$, $\sigma_{c^{n},d}^2 \in [0, \sigma^2_{c}]$.
\end{assumption}

\begin{assumption}
[Independence]
\label{assmp: all_3}
For any $t \!\in\! [T]$, $n \!\in\! [N]$, $i \!\in\! [K]$, $d_1,d_2 \!\in\! [D]$, $\boldsymbol{r}_{i,t}(d_1)$, $\boldsymbol{c}^{n}_t(d_2)$ are independent.
\end{assumption}

Assumption \ref{assmp: all_3} is common in real applications since $\boldsymbol{c}_t$ and $\boldsymbol{r}_t$ are inherently determined by independent factors: user characteristics and arm properties. For example, an individual user's preferences do not influence a restaurant's location, environment, pricing level, etc., and vice versa.

\noindent\textbf{Preference-aware reward.}
We define an \emph{overall-reward} as the \emph{inner product} of arm's reward and user's preference, which models the user reward rating under their preferences.
In each round $t$, for each user $n$, the overall-reward score $g_{a_t^{n},t}$ for the chosen arm $a_t^n$ is defined as: 
\begin{equation}
\label{eq:g_at}
\textstyle
g_{a_t^{n},t}
= {\boldsymbol{c}_t^{n}}^{\top} \boldsymbol{r}_{a_t^{n}, t}.
\end{equation}
To evaluate the learner’s performance, we define regret as the cumulative gap in overall-reward between selecting the optimal arm for each user at each round and the actual learner’s policy across the entire user set:
\begin{equation}
\label{eq: regret_def}
\textstyle
R(T) 
=
\sum^{T}_{t=1} \sum^{N}_{n=1} \boldsymbol{\overline{c}^n}^{\top} (\boldsymbol{\mu}_{a^{n*}_t} - \boldsymbol{\mu}_{a_t^n} ),
\end{equation}
$a^{n*}_t \!=\! \argmaxA_{i \in \mathcal{A}_t^n} \boldsymbol{\overline{c}^n}^{\top} \boldsymbol{\mu}_{i}$ is the optimal arm for user $n$ at round $t$.
The goal is to minimize the regret $R(T)$.
We term this problem as \emph{Preference-Aware MO-MAB} (PAMO-MAB).

\section{A Lower Bound}
\label{sec:lower_bd}
In the following, we develop a lower bound (Proposition \ref{prop: lower_bd}) on the defined regret for PAMO-MAB.
Such a lower bound will quantify how difficult it is to control regret without preference-adaptive policies under PAMO-MAB. 
Firstly, we present a definition characterizing a class of MO-MAB algorithms that are "preference-free".

\begin{definition}[Preference-Free Algorithm]
\label{def: pref_free_alg}
Let
$\boldsymbol{\mathrm{c}}^{\top} = \{\boldsymbol{c}_1, \boldsymbol{c}_2, ..., \boldsymbol{c}_{t}\} \in \mathbb{R}^{D \times t}$ be the preference sequence up to $t$ episode with mean $\boldsymbol{\overline{c}}$. 
Let $\pi_t^{\mathcal{A}}$ be the policy of algorithm $\mathcal{A}$ at time $t$ for selecting arm $a_t$. 
Then $\mathcal{A}$ is defined as preference-free if its policy 
$\pi_t^{\mathcal{A}}$ is independent of $\boldsymbol{\mathrm{c}}^{\top}$ and $\boldsymbol{\overline{c}}$, i.e., 
$\mathbb{P}_{\pi_t^{\mathcal{A}}} (a_t = i \mid \boldsymbol{\mathrm{c}}^{\top}, \boldsymbol{\overline{c}}) = \mathbb{P}_{\pi_t^{\mathcal{A}}} (a_t = i)$
for all arms $i \in [K]$ and all episodes $t \in [T]$.
\end{definition} 

To our knowledge, most existing algorithms in theoretical MO-MAB studies \citep{drugan2013designing, busa2017multi, xu2023pareto, huyuk2021multi, cheng2024hierarchize} fall within the class of preference-free algorithms, which employ a global policy for arm selection, while neglecting users' preferences.

\begin{proposition}
\label{prop: lower_bd}
Assume an MO-MAB environment contains multiple objective-conflicting arms, i.e., $\vert \mathcal{O}^{*} \vert \geq 2$, where $\mathcal{O}^{*}$ is the Pareto Optimal front. 
Then, for any preference-free algorithm, there exists a subset of users with distinct preferences such that the regret $R(T) = \Omega(T)$. 
\end{proposition} 

Proposition \ref{prop: lower_bd} shows that for PAMO-MAB problem with $|\mathcal{O}^*| \geq 2$, sub-linear regret is unachievable for preference-free algorithms.
This is because, for any arm $i \in \mathcal{O}^*$ that is optimal in one user preference subset $\mathcal{C}^{+}$, there exists another user preference subset $\mathcal{C}^{-}$ where arm $i$ becomes suboptimal,
while preference-free algorithms cannot adapt their policies to varying preference across the entire space $\mathcal{C}$. 
Please see Appendix \ref{sce: app_pf_lower_bd} for the  detailed proof. 
We therefore ask the following question:
\textbf{\emph{Can we design preference-adaptive algorithms that achieve sub-linear regret for PAMO-MAB?}}
The answer is {\bf yes.}
In the following, we analyze PAMO-MAB under two scenarios: hidden preference and preference feedback provided, demonstrating that with preference adaptation, sub-linear regret can indeed be achieved.

\section{General Case with Hidden Preference}
\label{sec: hidden}

\begin{figure}[t]
\centerline{\includegraphics[width=0.95\columnwidth]{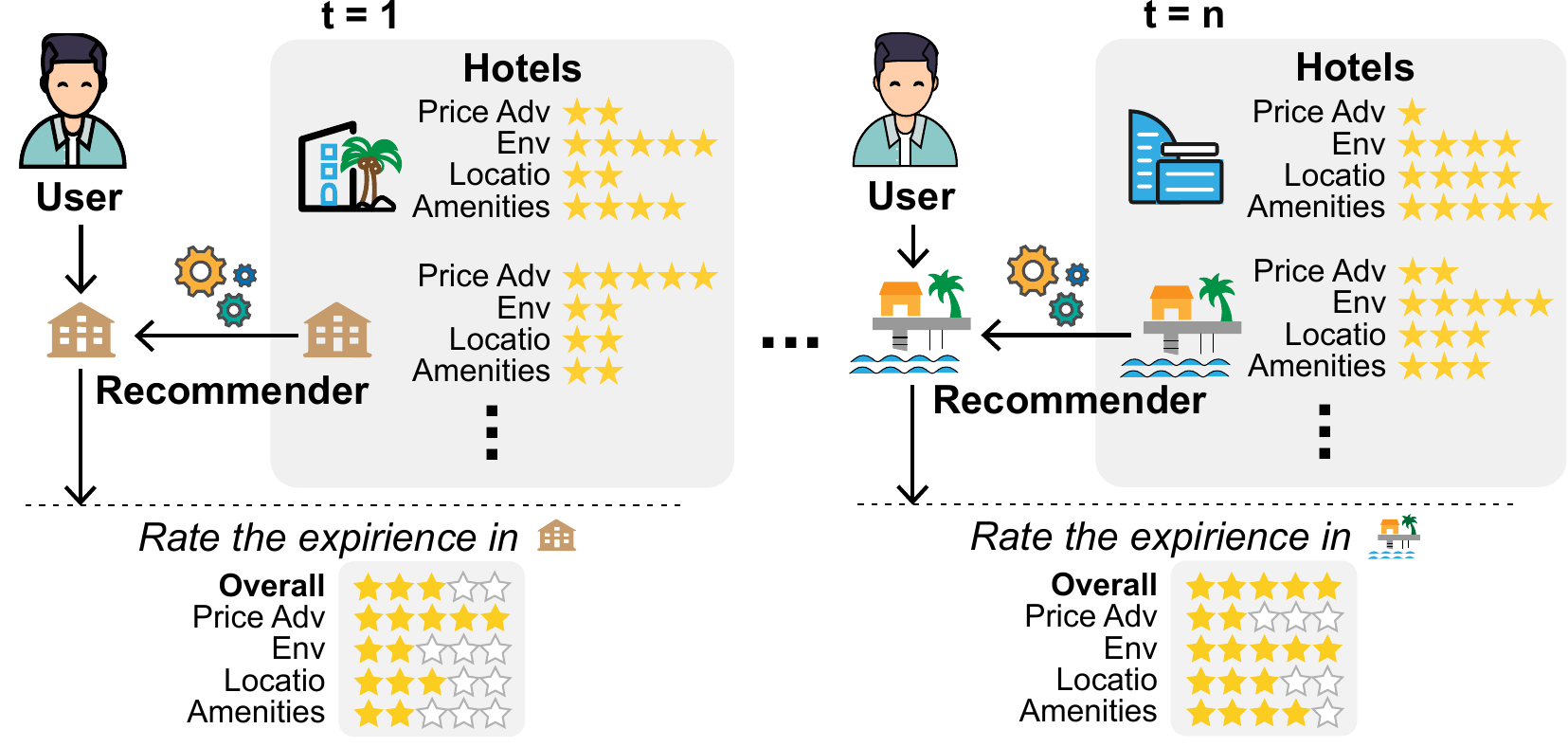}}
\caption{A scenario of user’s preference feedback is not explicitly provided (hidden preference).}
\label{fig: s3}
\end{figure}

We first consider a more practical but more challenge scenario where only feedback on the reward and overall reward is observable, while preference feedback is hidden. For instance, in hotel surveys, customers often provide ratings on specific objectives (e.g., price, location, environment, amenities) along with an overall rating (as depicted in Fig. \ref{fig: s3}). In such cases, user preferences can be inferred from the latent relationship between the overall rating and the individual objective ratings.
Formally, at each round $t$, the learner selects an arm $a_t \!\in\! \mathcal{A}_{t}^n$ for each user $n$, and observes the reward vector $\boldsymbol{r}_{a_t^n} \!\in\! \mathbb{R}^{D}$, and the overall-reward score $g_{a_t^n,t}$.

Within this framework, we adhere to the original Assumption~\ref{assmp: all_1} regarding rewards. It is worth noting that, in many real-world applications like hotel rating systems, the overall rating often shares the same scale as individual objective ratings.  
Therefore, we assume in this problem that the bound on the overall reward is identical to that of the individual rewards. This introduces one additional assumption and one revised assumption, as outlined below:
\begin{assumption}
\label{assmp: hpm_2}
For $t \in [T]$, $n \in [N]$, $a_t^n \in [K]$, the overall-reward score satisfies 
$g_{a_t^n,t} \in [0,1]$.
\end{assumption}

\begin{assumption}
\label{assmp: hpm_3}
For $t \in [T]$, $n \in [N]$, the stochastic preference is bounded and satisfies $\Vert \boldsymbol{c}_t^n \Vert_1 \leq 1$. Without loss of generality, we assume $\boldsymbol{c}_t^n(d)$ is $R$-sub-Gaussian\footnote{By Hoeffding’s lemma, for any $X \in [a,b]$ almost surely, $X$ is a $R$-sub-Gaussian random variable with $R$ at most $(b-a)/2$.}, $\forall d \in [D]$.
\end{assumption}

\begin{figure}[t]
\centerline{\includegraphics[width=0.95\columnwidth]{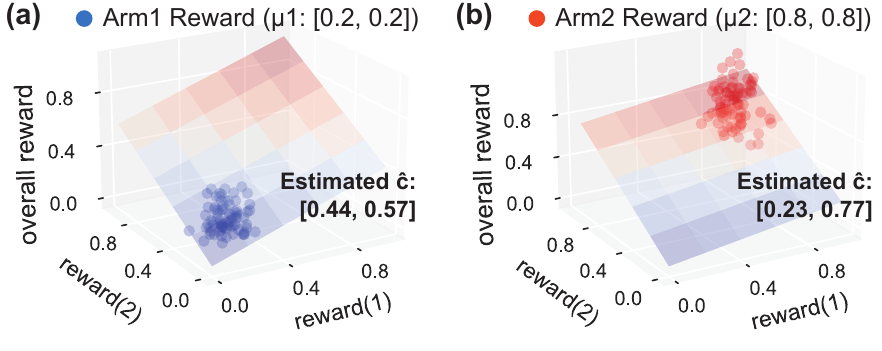}}
\caption{A 2-dimensional hidden preference PAMO-MAB toy example with mean preference $\overline{\boldsymbol{c}} = [0.5, 0.5]$, illustrating preference estimate $\hat{\boldsymbol{c}}$ via linear regression using reward data from (a) Arm-1 (dominated mean reward: $[0.2,0.2]$) and (b) Arm-2 (Pareto-optimal mean reward: $[0.8,0.8]$).
}
\label{fig: lr_1}
\end{figure}

As discussed in Section \ref{sec:lower_bd}, policy adapting to user preference is crucial. To enable effective learning and decision-making in this setting, we propose a unified framework that centers around two key components:

\begin{itemize}
    \item 
\emph{Preference Estimation}: Inferring the user preference vector from the observed bandits feedback.

    \item 
\emph{Preference-Aware Optimization}: Selecting arms based on preference estimate to align decisions with user intent.
\end{itemize}

These two components serve as the algorithmic backbone of our approach and remain consistent across both hidden and revealed preference settings. However, in the hidden preference case, each component faces unique technical challenges.

\subsection{Unique Challenges}
\label{sec:uniq_chllenge}

\subsubsection{Random Mapping from $\boldsymbol{r}_t$ to $g_t$.}
In the hidden preference case, for each user $n \in [N]$, the observed overall rewards are generated through a \emph{random mapping} of rewards.
Specifically,
$g_{a_t^n,t} = (\boldsymbol{\overline{c}^n} + \boldsymbol{\zeta}_t^n)^{\top} \boldsymbol{r}_{a_t^n,t} = \boldsymbol{\overline{c}^n}^{\top} \boldsymbol{r}_{a_t^n,t} + {\boldsymbol{\zeta}_t^n}^{\top} \boldsymbol{r}_{a_t^n,t}$, where $\boldsymbol{\zeta}_t^n = \boldsymbol{c}_t^n - \boldsymbol{\overline{c}}^n \in \mathbb{R}^D$ is an independent random noise vector. 
This formulation implies that the overall residual noise term $\zeta_{g,t}^n = {\boldsymbol{\zeta}_t^n}^{\top} \boldsymbol{r}_{a_t^n,t}$ is no longer independent of the input.
Consequently, standard regression models become infeasible for preference estimation, as they rely on the assumption that the residual noise in the output is independent of the input.

Additionally, the magnitude of overall residual noise is a monotonically non-decreasing function w.r.t each reward objective, i.e., 
$\Vert \zeta_{g,t}(\boldsymbol{r}_{i}) \Vert \leq \Vert \zeta_{g,t}(\boldsymbol{r}_{j}) \Vert$ iff $\boldsymbol{r}_{i} \preceq \boldsymbol{r}_{j}$.
This property implies that arms with smaller reward vectors, which are often suboptimal in terms of utility, may paradoxically offer more reliable information for preference estimation due to their lower sensitivity to noise. This issue would also have important implications for the optimization strategy, as we will discuss in the next challenge.
Thus, a tailored latent preference estimator is essential to mitigate the expanding error w.r.t the reward and ensure effective preference learning.

\subsubsection{Local Exploration \emph{vs} Global Exploration.}
Unlike traditional bandit algorithms that focus on a single goal (e.g., identifying the arm with the highest reward), the uncertainty in both preference and reward, combined with the need to infer latent preference, introduces a novel trade-off challenge: balancing \emph{global exploration} for better preference estimation and \emph{local exploration} of arm rewards:
\begin{itemize}[leftmargin=*]
\item \emph{Global exploration for preferences:}
Selecting arms that reduce uncertainty in poorly explored direction of the feature space, refining the model for preference learning.
\item \emph{Local exploration for rewards:}
Selecting arms to reduce uncertainty for specific individual arm reward estimate,  while balancing exploiting empirically high reward arms.
\end{itemize}

Note that these two learning objectives may conflict, as arms with high rewards might lack sufficient information for latent preference learning and may even worsen degrade estimation performance (as we discussed in the first challenge). 
This can also be verified by Fig. \ref{fig: lr_1}, where 80 samples of $[\boldsymbol{r}_{t}, g_{t}]$ are collected by repeatedly pulling an arm, and preference $\boldsymbol{\hat{c}}$ is estimated using linear regression. Here, $\boldsymbol{c}_t$ at each step follows a Gaussian distribution with a mean of $[0.5, 0.5]$. The results demonstrate that samples from suboptimal Arm-1 (Fig. \ref{fig: lr_1}a) significantly outperform those from Pareto-optimal Arm-2 (Fig. \ref{fig: lr_1}b) in preference estimation.
This necessitates an exploration policy that effectively addresses both global and local learning objectives.

\begin{algorithm}[t]
\caption{PRUCB with Hidden Preference (PRUCB-HP)}
\label{alg:PRUCB_HP}
\begin{algorithmic}
\STATE \textbf{Parameters:}
$\alpha$, $\lambda$, $\beta_t$, $\omega$.
\STATE \textbf{Initialization:}
$\boldsymbol{\hat{r}}_{i,1} \!\leftarrow\! [0]^D,
 N_{i, 1} \!\leftarrow\! 0, \forall i \!\in\! [K]$;\\
For each user $n \in [N]$:
$\boldsymbol{\hat{c}}_1^{n} \!\leftarrow\! [1/D]^D$, 
$\boldsymbol{V}_0^{n} \!\leftarrow\! \lambda \boldsymbol{I}$.
\FOR{$t=1, \cdots, T$}
\FOR{ user $n \in [N]$}
    \STATE Compute reward bonus term:\\ $B_{i,t}^{n,r} = \Vert \hat{\boldsymbol{c}}^n_t \Vert_1 \sqrt{\frac{\log{t/\alpha}}{\max\{N_{i,t}, 1\}}}, \forall i \in \mathcal{A}_t^n$.
    \STATE Compute pseudo information gain term:\\ $B_{i,t}^{n,c} \!=\! \beta_t \left \Vert \hat{\boldsymbol{r}}_{i,t} \!+\! \sqrt{\frac{\log{t/\alpha}}{\max\{N_{i,t}, 1\}}} \boldsymbol{e} \right \Vert_{{\boldsymbol{V}_{t-1}^n}^{-1}}, \forall i \in \mathcal{A}_t^n$.
    \STATE \textbf{Pull arm} $a_t \!\leftarrow\! \argmaxA_{i \in \mathcal{A}_{t}^n} ({\hat{\boldsymbol{c}}_t^{n}})^{\top} \hat{\boldsymbol{r}}_{i,t} \!+\! B_{i,t}^{n,r} \!+\! B_{i,t}^{n,c}$.
    \STATE \textbf{Observe} reward $\boldsymbol{r}_{a_t^n,t}$ and overall-reward $g_{a_t^n,t}$.
    \STATE \textbf{Updating:}
    \STATE \indent $w_t^n = \frac{\omega}{\Vert \boldsymbol{r}_{a_t^n,t} \Vert_2^2}$, $\boldsymbol{V}^n_{t} = \boldsymbol{V}^n_{t-1} + w_t^n \boldsymbol{r}_{a_t^n,t}\boldsymbol{r}_{a_t^n,t}^{\top}$.
    \STATE \indent $\boldsymbol{\hat{c}}^n_{t+1} = ({\boldsymbol{V}_{t}^n})^{-1} \sum_{\ell=1}^{\top} w_{\ell}^n g_{a_{\ell}^n, \ell} \boldsymbol{r}_{a_{\ell}^n,\ell}$.
\ENDFOR
\STATE \textbf{Updating:}
\STATE \indent $N_{i, t+1} = N_{i, t} + \sum_{n \in [N]} \mathds{1}_{ \{a_{t}^n = i \}}$.
\STATE \indent $\hat{\boldsymbol{r}}_{i,t+1} \!=\! 
\frac{ \hat{\boldsymbol{r}}_{i,t} N_{i, t} + \sum_{n=1}^{N} \boldsymbol{r}_{a_{t}^n, t} \cdot \mathds{1}_{ \{a_{t}^n = i \}} }{N_{i, t+1}}, \forall i \in [K]$.
\ENDFOR
\end{algorithmic}
\end{algorithm}

\subsection{Our Algorithm}
To this end, we propose a novel PRUCB-HP method (Algorithm \ref{alg:PRUCB_HP}) involving two key designs for both preference estimation and preference-aware optimization as follows.

\subsubsection{Key design I: WLS-Preference Estimator.}
For each user $n \in [N]$, as we have seen before, the randomness of preference $\boldsymbol{c}_t^n$ leads to the overall residual noise $\zeta_{g,t}^n$ be a function w.r.t, input reward $\boldsymbol{r}_t^n$. Moreover, larger input rewards $\boldsymbol{r}_t^n$ result in greater corruption from the residual noise.
To resolve this, we employ a weighted least-squares (WLS) estimator for preference learning. Specifically, our algorithm assigns a weight $w_t^n$ to each observed sample and estimates the unknown preference using weighted ridge regression:
\[
\textstyle
\hat{\boldsymbol{c}}_t^n \xleftarrow{} \argminA_{\boldsymbol{c} \in \mathbb{R}^{D}} \lambda \Vert \boldsymbol{c} \Vert_2^2 
+
\sum_{\ell=1}^{t-1} w_{\ell}^n ( {\boldsymbol{c}}^{\top} \boldsymbol{r}_{a_{\ell}^n,\ell} - g_{a_{\ell}^n,\ell})^2,
\]
where $\lambda$ is the regularization parameter. Above optimization problem has a closed-form solution as:
\begin{equation}
\label{eq:hpm_c_t}
\textstyle
\boldsymbol{\hat{c}}_{t}^n = ({\boldsymbol{V}_{t-1}}^n)^{-1} \sum_{\ell=1}^{t-1} w_{\ell}^n g_{a_{\ell}^n, \ell} \boldsymbol{r}_{a_{\ell}^n,\ell},
\end{equation}
where the Gram matrix $\boldsymbol{V}_{t-1}^n = \lambda \boldsymbol{I} + \sum_{\ell=1}^{t-1} w_{\ell}^n \boldsymbol{r}_{a_{\ell}^n,\ell} \boldsymbol{r}_{a_{\ell}^n,\ell}^{\top}$.

Inspired by \citet{zhou2021nearly} using the inverse of the noise variance as weight 
for tight variance-dependent regret guarantee, 
we define the weight as the inverse of squared $\ell_2$-norm of the reward: $w_t^n = \omega / \Vert \boldsymbol{r}_{a_t^n,t} \Vert_2^2$, where $\omega > 0$ is a threshold parameter guaranteeing $w_t^n \geq 1$. 
Intuitively, it ensures samples with high rewards will be assigned smaller weights to reduce the influence of potentially large residual noises, while samples with low rewards receive larger weights to ensure their contribution to the estimation.

To see how our choice of weight can tackle the \emph{random mapping} issue, we first define 
$\boldsymbol{r}_{a_t^n,t}^{\prime} = \sqrt{w_t^n} \boldsymbol{r}_{a_t^n,t}$, 
$g_{a_t^n,t}^{\prime} = \sqrt{w_t^n} g_{a_t^n,t}$, then the original formula Eq. \ref{eq:g_at} can be rewrite as
\[
\textstyle
\sqrt{w_t^n} \cdot g_{a_t^n,t} = \sqrt{w_t^n} \cdot {\boldsymbol{c}_t^n}^{\top} \boldsymbol{r}_{a_t^n, t} = \sqrt{w_t^n} \cdot (\boldsymbol{\overline{c}}^n + \boldsymbol{\zeta_t}^n)^{\top} \boldsymbol{r}_{a_t^n, t}
\]
\begin{equation}
\label{eq:g_t_2}
\textstyle
\implies
g_{a_t^n,t}^{\prime} = {\boldsymbol{\overline{c}}^n}^{\top} \boldsymbol{r}_{a_t^n, t}^{\prime} + \sqrt{w_t^n} \cdot {\boldsymbol{\zeta}_{t}^n}^{\top} \boldsymbol{r}_{a_{t}^n, t}.
\end{equation}
For term $\sqrt{w_t^n} \boldsymbol{\zeta}_{t}^{\top} \boldsymbol{r}_{a_{t}^n, t}$, we have the following lemma:
\begin{lemma}
\label{lemma: R_normed}
For any $n\in[N]$, the random variable $\sqrt{w_t^n} {\boldsymbol{\zeta}_{t}^n}^{\top} \boldsymbol{r}_{a_{t}^n, t}$ is sub-Gaussian with constant $R^{\prime} = \sqrt{\omega} R$.
\end{lemma}
The proof is available in Appendix \ref{sec: app_pf_lemma_R_normed}. By above lemma, we observe that with the designed weight, the original random mapping regression problem is transferred into a new formula as (\ref{eq:g_t_2}). Specifically, the output $g_{a_t^n,t}^{\prime}$ is mapped from $\boldsymbol{r}_{a_t^n,t}^{\prime}$ via a \emph{fixed} vector $\boldsymbol{\overline{c}}^n$ with a normed $R^{\prime}$-sub-Gaussian residual noise, where $R^{\prime} = \sqrt{\omega} R$, independent of the input.

\subsubsection{Key design II: Dual-Exploration Policy.}
As discussed earlier, there is a new global-local exploration dilemma in our setting. 
On the one hand, the algorithm must focus on local exploration by selecting optimistically profitable arms to discover better ones. Simultaneously, it must globally explore diverse arms to gather information about the relationship between $\boldsymbol{r}_t$ and $g_t$ for modeling the hidden preference.

To resolve this, we design an \emph{optimistic dual-exploration policy} by incorporating a \emph{preference-driven bonus} and \emph{reward-driven bonus} under preference-aware optimization framework for trade-off.
The optimistic policy is defined as
\begin{equation}
\label{eq:a_t_hidden}
\textstyle
a_t \leftarrow \argmaxA_{i \in \mathcal{A}_{t}^n} ({\hat{\boldsymbol{c}}_t^{n}})^{\top}  \hat{\boldsymbol{r}}_{i,t} + B_{i,t}^{n,r} + B_{i,t}^{n,c},
\end{equation}
where $B_{i,t}^{n,r}$ and $B_{i,t}^{n,c}$ are the dual-exploration bonus terms for user $n$. We detail the design of these bonus terms below and will later theoretically demonstrate in Section \ref{sec:theoretical_result_hidden} how they establish a tight UCB for the expected overall reward, ensuring the effectiveness of the optimistic policy in (\ref{eq:a_t_hidden}).

\noindent\underline{\emph{Reward Bonus}} $B_{i,t}^{n,r}$.
The reward bonus term explicitly encourages local exploration of arms with potentially high rewards, for the principle of optimism in face of uncertainty. Specifically, the bonus $B_{i,t}^{n,r}$ is formulated as a reward uncertainty-aware regularization term:
\begin{equation}
B_{i,t}^{n,r} = \rho_{i,t}^{\alpha} \Vert \hat{\boldsymbol{c}}_t^n \Vert_1.
\end{equation}
$\rho_{i,t}^{\alpha} = \sqrt{ \log(t/\alpha) / \max \{1, N_{i, t } \} }$ represents the standard Hoeffding bonus that quantifies the uncertainty in the reward estimates for each arm, ensuring that arms with higher uncertainty or lower exploration counts will be prioritized.

\noindent\underline{\emph{Preference Bonus}} $B_{i,t}^{n,c}$.
The preference bonus term aims to encourage the exploration of arms that reduce uncertainty in preference estimation.
In previous bandit studies \cite{abbasi2011improved, zhao2020simple, he2022nearly} involving linear coefficient ($\theta^*$) learning, it has been shown that $\beta \Vert \boldsymbol{x}_i \Vert_{\boldsymbol{V}^{-1}}$ provides a tight confidence bonus for the payoff of arm $i$, where $\beta$ is the confidence set radius for coefficient estimation, and $\boldsymbol{x}_i$ is the observable arm feature. In information theory, $\Vert \boldsymbol{x}_i \Vert_{\boldsymbol{V}^{-1}}$ also reflects entropy reduction in the model posterior, and is used to measure the estimator uncertainty improvement contributed by the chosen action $\boldsymbol{x_i}$ \cite{li2010contextual}.

However, such design is not feasible in our setting, as the exact reward $\boldsymbol{r}_{a_t^n,t}$ is revealed only after pulling the arm $a_t^n$, making the actual information gain $\Vert \boldsymbol{r}_{a_t^n,t} \Vert_{{(\boldsymbol{V}^{n}_{t-1})}^{-1}}$ from arm $a_t^n$ unpredictable beforehand.
To resolve this problem, we introduce a \emph{pseudo information gain} term, defined as $\Vert \boldsymbol{\hat{r}}_{i,t} + \rho_{i,t}^{\alpha} \boldsymbol{e} \Vert_{{(\boldsymbol{V}^{n}_{t-1})}^{-1}}$, where $\rho_{i,t}^{\alpha}$ is the standard Hoeffding bonus. And then the preference bonus is set as 
\begin{equation}
B_{i,t}^{n,r} = \beta_t \cdot \Vert \boldsymbol{\hat{r}}_{i,t} + \rho_{i,t}^{\alpha} \boldsymbol{e} \Vert_{{(\boldsymbol{V}^{n}_{t-1})}^{-1}},
\end{equation}
where $\beta_t$ is the confidence radius of the preference estimate we will give in Lemma \ref{lemma:g_estimator_upper_conf_bd}.
Intuitively, this pseudo information gain term captures the potential improvement of preference estimator could achieve by \emph{optimistically} selecting arm $i$ based on its reward estimation.
In this way, it explicitly encourages global exploration of arms that reduce uncertainty in preference estimation while performing local exploration.

\subsection{Theoretical Results}
\label{sec:theoretical_result_hidden}

In this section, we provide theoretical guarantees for the PRUCB-HP algorithm.
We first characterize the estimation error of $\boldsymbol{\hat{c}}_t$ w.r.t $\overline{\boldsymbol{c}}$ by WLS preference estimator below.

\begin{lemma}
\label{lemma:c_estimator_conf_bd}
Under Assumption \ref{assmp: hpm_3}, for any user $n \in [N]$, $0 < \alpha <1$, $\omega>0$, with a probability at least $1-\vartheta$, the preference estimator $\hat{c}_t^n$ in Algorithm \ref{alg:PRUCB_HP} verifies for all $t \in [1,T]$:
\[
\Vert \hat{\boldsymbol{c}}_t^n - \boldsymbol{\overline{c}}^n \Vert_{\boldsymbol{V}_{t-1}^n} 
\leq
R \sqrt{\omega D \log\big((1 + \omega t/\lambda)/\vartheta\big)} + \sqrt{\lambda}.
\]
\end{lemma}
Please see Appendix \ref{sec:app_proof_lemma_c_estimator_conf_bd} for the proof. This estimate confidence bound essentially implies the effectiveness of WLS preference estimator with the designed weight to handle the random-mapping issue in our problem. With this in hand, we can then derive an upper confidence bound for the expected overall reward for each user.

\begin{lemma}
\label{lemma:g_estimator_upper_conf_bd}
Set $\beta_t = \sqrt{\omega D \log\big((1 + \omega t/\lambda)/\vartheta\big)} + \sqrt{\lambda}$, then for any $n \in [N]$, $i \in [K]$ and $t > 0$, with probability at least equal to $1 - \vartheta - D\alpha^2/t^2$, we have
\begin{equation}
\label{eq:ucb_hidden}
{\boldsymbol{\overline{c}}^n}^{\top} \boldsymbol{\mu}_{i} 
\leq
{\boldsymbol{\hat{c}}_{t}^n}^{\top} \boldsymbol{\hat{r}}_{i,t} 
+
B_{i,t}^{n,r}
+
B_{i,t}^{n,c},
\end{equation}
with 
$B_{i,t}^{n,r} \!=\! \rho_{i,t}^{\alpha} \Vert \hat{\boldsymbol{c}}_t^n \Vert_1$
and 
$B_{i,t}^{n,c} \!=\! \beta_t \left \Vert \hat{\boldsymbol{r}}_{i,t} \!+\! \rho_{i,t}^{\alpha} \boldsymbol{e} \right \Vert_{{(\boldsymbol{V}_{t-1}^{n})}^{-1}}$
as the bonus terms for dual-exploration.
\end{lemma}

Please see Appendix \ref{sec:app_pf_lemma_g_estimator_upper_conf_bd} for the proof. Lemma \ref{lemma:g_estimator_upper_conf_bd} essentially suggests an upper confidence bound of the expected reward for each user, as adopted in our dual-exploration policy (Eq. \ref{eq:a_t_hidden}).
Notably, two bonus terms strike a balance between local and global explorations while guaranteeing optimization under the principle of optimism in face of uncertainty.

\begin{theorem}
\label{theorem:up_bd_hiden}
For PAMO-MAB with hidden preference, for any $\lambda > 0$, by setting
$\alpha = \sqrt{\frac{12 \vartheta}{KD(D+3) \pi^2}}$,
$\omega \geq D$,
$\beta_t = \sqrt{\omega D \log\big((1 + \omega T/\lambda)/\alpha\big)} + \sqrt{\lambda}$, 
with probability greater than $1- 2\vartheta$, Algorithm \ref{alg:PRUCB_HP} has
\begin{small}
\[
\begin{aligned}
\textstyle
R(T) 
=
& O\Big(
DRN \sqrt{ \omega T \log^{2} \Big( \big(1 + \omega T/\lambda\big)/\vartheta \Big)}
 \\
& \qquad + 
\frac{DRN}{\sqrt{\lambda}} \sqrt{ \omega DK T \log^2 \left(\big( 1 + \omega T/\lambda \big)/\vartheta) \right)}
 \\
& \qquad + 
N\sqrt{ KT \log \left( T/\vartheta \right)}
+
MN
\Big),
\end{aligned}
\]
\end{small}
with 
\begin{small}
$M = \left\lfloor \min \big \{ t^{\prime} \mid t  \sigma^2_{r \downarrow} + \lambda \geq 2D \omega \sqrt{Kt\log \frac{t}{\alpha} }, \forall t \geq t^{\prime} \big \} \right \rfloor$.\footnote{Since $\sigma^2_{\boldsymbol{r} \downarrow} \!\in\! \mathbb{R}^{+}$, 
$\lim_{t \rightarrow \infty} 2D \omega \sqrt{Kt \log \frac{t}{\alpha} }/\big(\sigma^2_{r \downarrow}t\big) = \lim_{t \rightarrow \infty} C_1 \sqrt{ (\log t - C_2)/t} = 0$, as $\sqrt{\log t}$ grows much slowly compared to $\sqrt{t}$. Hence $M$ exists for sufficiently large $t^{\prime}$.}
\end{small}
Here the first term represents regret from preference estimation error, the third from reward estimation error, and the second from the combined error of both.
\end{theorem} 

Please see Appendix \ref{sec: app_pf_thm_up_bd_hiden} for the proof. The key difficulty of the proof is to upper-bound the accumulative preference bonus $\sum_{t}B_{i,t}^{c}$. Specifically, we need to quantify the weighted $\ell_2$-norm of the empirical estimation $\boldsymbol{\hat{r}}_{a_t,t}$ with weighting matrix $\boldsymbol{V}_{t}$ constructed by the true reward $\boldsymbol{r}_{a_t,t}$ instead. 
This inconsistency renders the classical induction method \cite{abbasi2011improved} for deriving $\log (\frac{\det \boldsymbol{V}_{T}}{\det \boldsymbol{V}_{0}})$ infeasible for upper-bounding $\sum_{t} \Vert \boldsymbol{\hat{r}}_{a_t,t} \Vert_{\boldsymbol{V}_{t-1}^{-1}}^2$. 
To resolve this, we first transfer $\Vert \boldsymbol{\hat{r}}_{a_t,t} \Vert_{\boldsymbol{V}_{t-1}^{-1}}$ to $\Vert \boldsymbol{\mu}_{a_t} \Vert_{\boldsymbol{V}_{t-1}^{-1}}$. Then we show that for sufficiently large $t$, $a \Vert \boldsymbol{\mu}_{a_t} \Vert_{\mathbb{E}[\boldsymbol{V}_{t-1}]^{-1}}$ serves as an upper bound for $\Vert \boldsymbol{\mu}_{a_t} \Vert_{\boldsymbol{V}_{t-1}^{-1}}$ with constant $a > 1$ (see Lemma \ref{lemma: hidden_sum_reg_c_expectation_bd}).
This allows us to use $a \Vert \boldsymbol{\mu}_{a_t} \Vert_{\mathbb{E}[\boldsymbol{V}_{t-1}]^{-1}}$ as an upper-bound, where a new recursion relationship between $\mathbb{E}[\boldsymbol{V}_{t-1}]$ and $\boldsymbol{\mu}_{a_t}$ can be guaranteed, enabling us to bound $\log (\frac{\det \mathbb{E}[\boldsymbol{V}_{T}]}{\det \mathbb{E}[\boldsymbol{V}_{0}]})$ via induction, which can further be bounded by slightly modifying existing techniques in linear bandits.

\begin{remark}
Theorem \ref{theorem:up_bd_hiden} shows that, even without explicit preference feedback, PRUCB-HP achieves sub-linear regret through carefully designed mechanisms for preference adaptation.
In particular, for $t \geq M$, where $M$ is a constant independent of $T$, the regret asymptotically scales as $\tilde{O} ( D \sqrt{T} )$.
\emph{To the best of our knowledge, this is the first result characterizing the performance of PAMO-MAB with hidden preference.}
\end{remark}

\section{A Special Case with Provided Preference}
\label{sec: knonw}

We next consider another practically important case  where the user's preferences are explicitly provided to the agent either before or after decision making. 
This setup is prevalent in numerous real-world applications. 
Many online systems now allow users to express their preferences through interactive techniques, such as conversations and prompt design, either before or after taking action. For example, as shown in Fig. \ref{fig: s1}, a user can explicitly share her movie preferences with an LLM-based chat system either prior to receiving a recommendation (\textsf{\ding{182}} in the initial input) or after receiving one (\textsf{\ding{183}} in the follow-up conversation).

\begin{figure}[t]
\centerline{\includegraphics[width=\columnwidth]{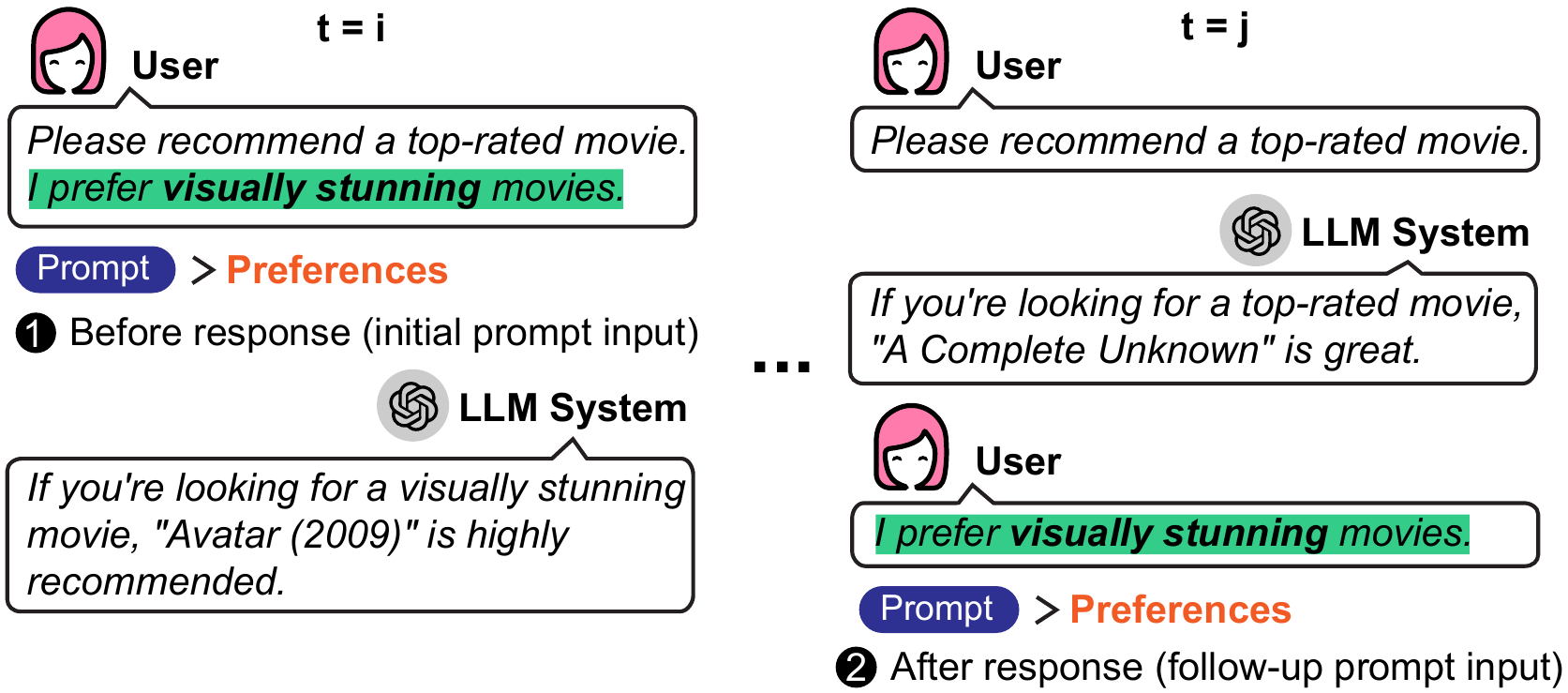}}
\caption{
An example where user explicitly provides her preference to LLM system before (\ding{182}) or after (\ding{183}) the response of movie recommendation (decision making).}
\label{fig: s1}
\end{figure}

We observe that when the preference vector $\boldsymbol{c}t$ is known before decision-making, the problem reduces to a standard single-objective MAB. Instead, we focus on the more challenging case where the preference is unknown during action selection but revealed afterward. The known-before-action case can be treated as a special instance of this setting and is deferred to Appendix~\ref{sec:app_up_bd_stat_known}. The goal remains to maximize the cumulative overall-reward over all users.
Formally, at each round $t$, the learner selects an arm $a_t$, then observes both the reward $\boldsymbol{r}\_{a_t,t}$ and user's preference $\boldsymbol{c}_t$, with both adhering to the original Assumptions \ref{assmp: all_1} and \ref{assmp: all_2}.
This setting avoids the challenges of preference inference and stochastic reward mapping, and can be viewed as a special case of the hidden preference setting with zero noise in preference feedback.

To address this case, we propose a simplified algorithm: PRUCB-UP (Algorithm \ref{alg:PRUCB_UP}) which builds on the same Preference-Aware framework used in the hidden case, with two direct adaptations:


\subsubsection{Preference estimation.}
Due to the explicitly provided preference feedback $\boldsymbol{c}_{t}^n$, we can directly leverage the empirical average of historical feedback as the preference estimate for each user $n$. For $t \geq 1$, preference estimate is updated as
\begin{equation}
\label{eq:PRUCB_SPM_c_t}
\textstyle
\hat{\boldsymbol{c}}^n_{t+1} = \big( (t-1) \boldsymbol{\hat{c}}^n_{t} + \boldsymbol{c}_{t}^n \big)/t.
\end{equation}
Similarly, the reward estimate $\hat{\boldsymbol{r}}_{i,t}$ is defined as empirical estimation. For $t \in [1,T]$, it is updated as:
\begin{equation}
\label{eq:prucb_r_t}
\textstyle
N_{i, t+1} = N_{i, t} + \sum_{n \in [N]} \mathds{1}_{ \{a_{t}^n = i \}},
\end{equation}
\[
\textstyle
\hat{\boldsymbol{r}}_{i,t+1} = (\hat{\boldsymbol{r}}_{i,t} N_{i, t} + \sum_{n \in [N]} \boldsymbol{r}_{a_{t}^n, t} \cdot \mathds{1}_{ \{a_{t}^n = i \}} )/N_{i, t+1},
\]
with $N_{i, 1} \leftarrow 0, \boldsymbol{\hat{r}}_{i, 1} \leftarrow [0]^D, \forall i \in [K]$. 

\subsubsection{Preference-aware optimization.}

\begin{algorithm}[t]
\caption{PRUCB Unknown Preference (PRUCB-UP)}
\label{alg:PRUCB_UP}
\begin{algorithmic}
\STATE \textbf{Initialize:} $\alpha$,
$N_{i, 1} \!\leftarrow\! 0$, $\boldsymbol{\hat{r}}_{i,1} \!\leftarrow\! [0]^{D}, \forall i \!\in\! [K]$; \\
$\quad \quad \quad \quad \boldsymbol{\hat{c}}_{1}^n \!\leftarrow\! [0]^D, \forall n \in [N]$.
\FOR{$t=1,\cdots,T$ }
\FOR{user $n \in [N]$ }
    \STATE \textbf{Draw arm} $a_t^n$ by (\ref{eq:prucb_spm_at}), 
    \STATE \textbf{Observe} reward $\boldsymbol{r}_{a_t^n, t}$ and user preference $\boldsymbol{c}_{t}^n$.
    \STATE \textbf{Update} preference estimate $\hat{\boldsymbol{c}}^n_{t+1}$ by (\ref{eq:PRUCB_SPM_c_t}).
\ENDFOR
\STATE \textbf{Update} $N_{i, t+1}$, $\hat{\boldsymbol{r}}_{i,t+1}, \forall i \in [K]$ by (\ref{eq:prucb_r_t}).
\ENDFOR
\end{algorithmic}
\end{algorithm}

We adopt the same UCB framework as in hidden preference case for preference-aware optimization, but with key simplifications. Since the preference $\boldsymbol{c}_t$ is observed after action selection and is independent of the chosen arm, its estimation does not involve sequential decision-making—hence no confidence term is needed for $\hat{\boldsymbol{c}}_t$. In contrast, reward estimates $\hat{\boldsymbol{r}}_t$ remain action-dependent, requiring reward bonuses to ensure sufficient exploration.

To this end, the arm selection policy is designed as:
\begin{equation}
\label{eq:prucb_spm_at}
\textstyle
a_t^n = \argmaxA_{i \in \mathcal{A}^n_t} ({\hat{\boldsymbol{c}}_t^{n}})^{\top} (\hat{\boldsymbol{r}}_{i,t} + \rho_{i,t}^{\alpha} \boldsymbol{e} ).
\end{equation}
where $\rho_{i,t}^{\alpha} = \sqrt{ \log(t/\alpha) / \max \{1, N_{i, t } \} }$ is standard Hoeffding bonus.
We characterize the regret of PRUCB-UP below.

\begin{theorem}
\label{theorem:up_bd_stat}
Let 
\begin{small}
$\mathcal{T}_i^{n} \!=\! \{ t \! \in \! [T] \mid i \! \in \! \mathcal{A}_{t}^n, i \! \neq \! a_{t}^{n*}\}$,
$\eta_{i}^{n \uparrow} = \max_{t \in \mathcal{T}_i^{n}} {\overline{\boldsymbol{c}}^n}^{\top} (\boldsymbol{\mu}_{a_t^{n*}} - \boldsymbol{\mu}_{i})$,
$\eta_{i}^{n \downarrow} = \min_{t \in \mathcal{T}_i^{n}} {\overline{\boldsymbol{c}}^n}^{\top} (\boldsymbol{\mu}_{a_t^{n*}} - \boldsymbol{\mu}_{i})$,
$\Vert \Delta_i^{n \uparrow} \Vert_2 = \max_{\{t,j\} \in \mathcal{T}_i^n \times \mathcal{A}^n_t } \Vert \boldsymbol{\mu}_{i} - \boldsymbol{\mu}_{j} \Vert_2$, 
\end{small}
Algorithm \ref{alg:PRUCB_UP} has the regret R(T) of
\begin{small}
\[
\begin{aligned}
 O
\Bigg(
\sum_{n \in [N]}
\sum_{i \in [K]}
\Big(
\underbrace{
\frac{\delta^2 \log T}{\eta_{i}^{n \downarrow}} 
+
D \pi^2 \alpha^2 \eta_{i}^{n \uparrow}
}_{R^{r}_T}
 +
\underbrace{
\frac{ D^2 \Vert \Delta_{i}^{n \uparrow} \Vert_2^2 \delta^2 }{ \eta_{i}^{n \downarrow} }
}_{R^{c}_T}
\Big)
\Bigg),
\end{aligned}
\]
\end{small}
where $R^{r}_T$ and $R^{c}_T$ refer to the regrets caused by reward estimate error and preference estimate error respectively.
\end{theorem}

\begin{remark}
Theorem \ref{theorem:up_bd_stat} shows that with preference feedback, PRUCB-UP achieves a regret of $O(KN \delta \log T + KN \delta D^2)$, demonstrating near-optimal performance. Notably, the regret caused by additional preference estimation error is bounded by a constant related to objective dimension $D$ and preference $\ell_1$-norm bound $\delta$. 
This implies that the impact of preference estimation error on the regret is small.
Additional, we show the preference known case can be solved by PRUCB-UP as a special case, achieving regret of $O(K \delta \log T)$, please see Appendix \ref{sec:app_up_bd_stat_known} for details. 
\end{remark}

To prove Theorem \ref{theorem:up_bd_stat}, the main difficulty lies in decoupling and capturing the effect of the joint error from both reward estimation and preference estimation on the final regret. 
To address this, we introduce a tunable parameter $\epsilon$ to quantify the error of preference estimate $\boldsymbol{\hat{c}}_t$, and decompose suboptimal actions into two disjoint sets: 
(1) suboptimal pulls under sufficiently precise preference estimate; (2) suboptimal pulls under imprecise preference estimate, which we show to be more analytically tractable.
Please see Appendix
\ref{sec:app_up_bd_stat} for the full proof.

\section{Numerical Analysis}
We evaluate the performance of PRUCB-UP and PRUCB-HP in unknown and hidden preference environments, respectively. The PAMO-MAB instance includes $N$ users, $K$ arms and $D$ objectives, with preferences and rewards following Gaussian distributions with randomly initialized means.
(Detailed settings refer to Appendix \ref{sec:app_exp_prucb_static}). 
We introduce a user-switching protocol to simulate practical scenarios. The environment features multiple users, each exposed to a block of arms (5 in our setup) per round. Only arms within the current block can be selected for that user. In the next round, the arm block rotates to a another user. The objective is to maximize cumulative overall ratings across all users. A more detailed illustration is provided in Appendix \ref{sec:app_exp_prucb_static}.

\begin{figure}[t]
    \subfigure[Hidden preference case]{
        \includegraphics[width=0.47\columnwidth]{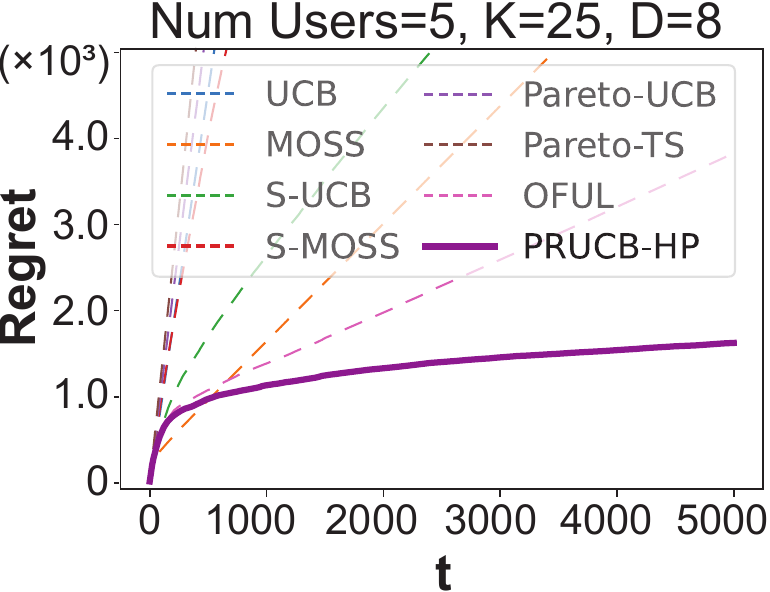}
    }
    \subfigure[Unknown preference case]{
        \includegraphics[width=0.47\columnwidth]{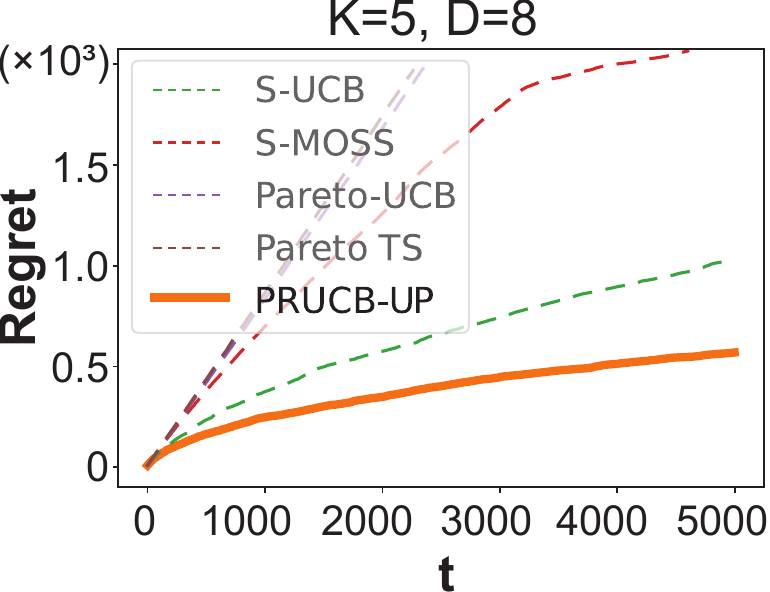}
    }
    \caption{Regret comparison of our proposed PRUCB with other benchmarks under different preference environments, where our methods outperforms other methods significantly.}
\label{fig: exp1_sample}
\end{figure}

We compare our results with other baselines including S-UCB, Pareto-UCB \cite{drugan2013designing}, S-MOSS, Pareto-TS~\cite{yahyaa2015thompson}), UCB~\cite{auer2002finite}, MOSS~\cite{audibert2009minimax} and OFUL \cite{abbasi2011improved}.
The regret is averaged across 10 trials with round $T = 5000$.
Figure \ref{fig: exp1_sample} shows that our proposed algorithms significantly outperform other competitors under both environments.
It is worth noting that for all the preference-free competitors exhibit linear regret, aligning with Proposition \ref{prop: lower_bd}, demonstrating that approaches agnostic to user preferences cannot align their outputs with user preferences, even if they achieve Pareto optimality.
For more comprehensive experimental analyses, please refer to Appendix \ref{sec:app_exp}.

\section{Conclusion}
In this paper, we make the first effort to theoretically explore the explicit user preferences-aware MO-MAB. Motivated by real-world applications, we provide a comprehensive analysis of this problem under unknown preference and hidden preference environments, with tailored algorithms achieving provably near-optimal regrets.


\bibliography{reference}

\begin{thebibliography}{28}
\providecommand{\natexlab}[1]{#1}

\bibitem[{Abbasi-Yadkori, P{\'a}l, and Szepesv{\'a}ri(2011)}]{abbasi2011improved}
Abbasi-Yadkori, Y.; P{\'a}l, D.; and Szepesv{\'a}ri, C. 2011.
\newblock Improved algorithms for linear stochastic bandits.
\newblock \emph{Advances in Neural Information Processing Systems}, 24.

\bibitem[{Audibert and Bubeck(2009)}]{audibert2009minimax}
Audibert, J.-Y.; and Bubeck, S. 2009.
\newblock Minimax policies for adversarial and stochastic bandits.
\newblock In \emph{Conference on Learning Theory}, 217--226.

\bibitem[{Audibert, Munos, and Szepesv{\'a}ri(2007)}]{audibert2007tuning}
Audibert, J.-Y.; Munos, R.; and Szepesv{\'a}ri, C. 2007.
\newblock Tuning bandit algorithms in stochastic environments.
\newblock In \emph{International Conference on Algorithmic Learning Theory}, 150--165. Springer.

\bibitem[{Auer, Cesa-Bianchi, and Fischer(2002)}]{auer2002finite}
Auer, P.; Cesa-Bianchi, N.; and Fischer, P. 2002.
\newblock Finite-time analysis of the multiarmed bandit problem.
\newblock \emph{Machine learning}, 47: 235--256.

\bibitem[{Balef and Maghsudi(2023)}]{balef2023piecewise}
Balef, A.~R.; and Maghsudi, S. 2023.
\newblock Piecewise-stationary multi-objective multi-armed bandit with application to joint communications and sensing.
\newblock \emph{IEEE Wireless Communications Letters}, 12(5): 809--813.

\bibitem[{Busa-Fekete et~al.(2017)Busa-Fekete, Sz{\"o}r{\'e}nyi, Weng, and Mannor}]{busa2017multi}
Busa-Fekete, R.; Sz{\"o}r{\'e}nyi, B.; Weng, P.; and Mannor, S. 2017.
\newblock Multi-objective bandits: Optimizing the generalized gini index.
\newblock In \emph{International Conference on Machine Learning}, 625--634. PMLR.

\bibitem[{Cheng et~al.(2024)Cheng, Xue, Yi, and Zhang}]{cheng2024hierarchize}
Cheng, J.; Xue, B.; Yi, J.; and Zhang, Q. 2024.
\newblock Hierarchize Pareto Dominance in Multi-Objective Stochastic Linear Bandits.
\newblock In \emph{Proceedings of the AAAI Conference on Artificial Intelligence}, volume~38, 11489--11497.

\bibitem[{Drugan(2018)}]{drugan2018covariance}
Drugan, M.~M. 2018.
\newblock Covariance matrix adaptation for multiobjective multiarmed bandits.
\newblock \emph{IEEE Transactions on Neural Networks and Learning Systems}, 30(8): 2493--2502.

\bibitem[{Drugan and Nowe(2013)}]{drugan2013designing}
Drugan, M.~M.; and Nowe, A. 2013.
\newblock Designing multi-objective multi-armed bandits algorithms: A study.
\newblock In \emph{The International Joint Conference on Neural Networks}, 1--8. IEEE.

\bibitem[{Ehrgott(2005)}]{ehrgott2005multicriteria}
Ehrgott, M. 2005.
\newblock \emph{Multicriteria optimization}, volume 491.
\newblock Springer Science \& Business Media.

\bibitem[{Fulton(2000)}]{fulton2000eigenvalues}
Fulton, W. 2000.
\newblock Eigenvalues, invariant factors, highest weights, and Schubert calculus.
\newblock \emph{Bulletin of the American Mathematical Society}, 37(3): 209--249.

\bibitem[{He et~al.(2022)He, Zhou, Zhang, and Gu}]{he2022nearly}
He, J.; Zhou, D.; Zhang, T.; and Gu, Q. 2022.
\newblock Nearly optimal algorithms for linear contextual bandits with adversarial corruptions.
\newblock \emph{Advances in neural information processing systems}, 35: 34614--34625.

\bibitem[{H{\"u}y{\"u}k and Tekin(2021)}]{huyuk2021multi}
H{\"u}y{\"u}k, A.; and Tekin, C. 2021.
\newblock Multi-objective multi-armed bandit with lexicographically ordered and satisficing objectives.
\newblock \emph{Machine Learning}, 110(6): 1233--1266.

\bibitem[{Jun et~al.(2018)Jun, Li, Ma, and Zhu}]{jun2018adversarial}
Jun, K.-S.; Li, L.; Ma, Y.; and Zhu, J. 2018.
\newblock Adversarial attacks on stochastic bandits.
\newblock \emph{Advances in neural information processing systems}, 31.

\bibitem[{Lamprier, Gisselbrecht, and Gallinari(2018)}]{lamprier2018profile}
Lamprier, S.; Gisselbrecht, T.; and Gallinari, P. 2018.
\newblock Profile-based bandit with unknown profiles.
\newblock \emph{Journal of Machine Learning Research}, 19(53): 1--40.

\bibitem[{Li et~al.(2010)Li, Chu, Langford, and Schapire}]{li2010contextual}
Li, L.; Chu, W.; Langford, J.; and Schapire, R.~E. 2010.
\newblock A contextual-bandit approach to personalized news article recommendation.
\newblock In \emph{Proceedings of the 19th international conference on World wide web}, 661--670.

\bibitem[{Lu et~al.(2019)Lu, Wang, Hu, and Zhang}]{lu2019multi}
Lu, S.; Wang, G.; Hu, Y.; and Zhang, L. 2019.
\newblock Multi-objective generalized linear bandits.
\newblock In \emph{Proceedings of the International Joint Conference on Artificial Intelligence}, 3080--3086.

\bibitem[{Mehrotra, Xue, and Lalmas(2020)}]{mehrotra2020bandit}
Mehrotra, R.; Xue, N.; and Lalmas, M. 2020.
\newblock Bandit based optimization of multiple objectives on a music streaming platform.
\newblock In \emph{Proceedings of the ACM SIGKDD International Conference on Knowledge Discovery and Data Mining}, 3224--3233.

\bibitem[{Reymond et~al.(2024)Reymond, Bargiacchi, Roijers, and Now{\'e}}]{reymond2024interactively}
Reymond, M.; Bargiacchi, E.; Roijers, D.~M.; and Now{\'e}, A. 2024.
\newblock Interactively Learning the User's Utility for Best-Arm Identification in Multi-Objective Multi-Armed Bandits.
\newblock In \emph{Proceedings of the 23rd International Conference on Autonomous Agents and Multiagent Systems}, 1611--1620.

\bibitem[{Turgay, Oner, and Tekin(2018)}]{turgay2018multi}
Turgay, E.; Oner, D.; and Tekin, C. 2018.
\newblock Multi-objective contextual bandit problem with similarity information.
\newblock In \emph{International Conference on Artificial Intelligence and Statistics}, 1673--1681. PMLR.

\bibitem[{Vershynin(2018)}]{vershynin2018high}
Vershynin, R. 2018.
\newblock \emph{High-dimensional probability: An introduction with applications in data science}, volume~47.
\newblock Cambridge university press.

\bibitem[{Wanigasekara et~al.(2019)Wanigasekara, Liang, Goh, Liu, Williams, and Rosenblum}]{wanigasekara2019learning}
Wanigasekara, N.; Liang, Y.; Goh, S.~T.; Liu, Y.; Williams, J.~J.; and Rosenblum, D.~S. 2019.
\newblock Learning Multi-Objective Rewards and User Utility Function in Contextual Bandits for Personalized Ranking.
\newblock In \emph{Proceedings of the International Joint Conference on Artificial Intelligence}, volume~19, 3835--3841.

\bibitem[{Xie et~al.(2021)Xie, Liu, Zhang, Wang, Xia, and Lin}]{xie2021personalized}
Xie, R.; Liu, Y.; Zhang, S.; Wang, R.; Xia, F.; and Lin, L. 2021.
\newblock Personalized approximate pareto-efficient recommendation.
\newblock In \emph{Proceedings of the Web Conference 2021}, 3839--3849.

\bibitem[{Xu and Klabjan(2023)}]{xu2023pareto}
Xu, M.; and Klabjan, D. 2023.
\newblock Pareto regret analyses in multi-objective multi-armed bandit.
\newblock In \emph{International Conference on Machine Learning}, 38499--38517. PMLR.

\bibitem[{Yahyaa, Drugan, and Manderick(2014)}]{yahyaa2014knowledge}
Yahyaa, S.~Q.; Drugan, M.~M.; and Manderick, B. 2014.
\newblock Knowledge Gradient for Multi-objective Multi-armed Bandit Algorithms.
\newblock In \emph{International Conference on Agents and Artificial Intelligence}, 74--83.

\bibitem[{Yahyaa and Manderick(2015)}]{yahyaa2015thompson}
Yahyaa, S.~Q.; and Manderick, B. 2015.
\newblock Thompson Sampling for Multi-Objective Multi-Armed Bandits Problem.
\newblock In \emph{ESANN}.

\bibitem[{Zhao et~al.(2020)Zhao, Zhang, Jiang, and Zhou}]{zhao2020simple}
Zhao, P.; Zhang, L.; Jiang, Y.; and Zhou, Z.-H. 2020.
\newblock A simple approach for non-stationary linear bandits.
\newblock In \emph{International Conference on Artificial Intelligence and Statistics}.

\bibitem[{Zhou, Gu, and Szepesvari(2021)}]{zhou2021nearly}
Zhou, D.; Gu, Q.; and Szepesvari, C. 2021.
\newblock Nearly minimax optimal reinforcement learning for linear mixture markov decision processes.
\newblock In \emph{Conference on Learning Theory}, 4532--4576. PMLR.

\end{thebibliography}
\bibliographystyle{aaai2026}


\appendix

\onecolumn
\section{Related Work}
\label{sec:related_work}

\subsubsection{Multi-Objective Multi-Armed Bandits.} 
MO-MAB extends scalar rewards in the standard MAB problem to multi-dimensional vectors. 
The Pareto-UCB \citep{drugan2013designing} introduced the MO-MAB framework and Pareto regret as a metric, achieving $O(\log T)$ Pareto regret using the UCB technique.
Other techniques, including Knowledge Gradient \citep{yahyaa2014knowledge} and Thompson Sampling \citep{yahyaa2015thompson}, have subsequently been adapted for MO-MAB.
Additionally, researchers have extended the contextual setup to MO-MAB \citep{turgay2018multi, lu2019multi}. 
These studies aim to approximate the entire Pareto front 
$\mathcal{O}^*$, and employ a \emph{random arm selection policy} on the estimated Pareto front for Pareto optimality.
Considering computing the full Pareto front is expensive, another line of work propose to converts the multi-dimensional reward into a scalar value through a scalarization function, targeting a specific Pareto optimal solution while not the entire Pareto front. Our work also falls within the realm of this framework.
The scalarization function can either be randomly initialized (chosen) \citep{drugan2013designing, xu2023pareto}, or optimized based on a fixed metric, such as the Generalized Gini Index score \citep{busa2017multi, mehrotra2020bandit}.
Nonetheless, existing studies primarily achieve Pareto optimality through a \emph{global policy} for arm selection across all users. 
As discussed in Section \ref{sec: intro}, merely achieving Pareto optimality with a global policy may not yield favorable outcomes, as users have diverse preferences on different objectives. 

\subsubsection{Preference-based MO-MAB optimization.}
Recent studies have explored MO-MAB optimization using lexicographic order \citep{ehrgott2005multicriteria} to reflect user preferences. In lexicographic order, objectives are prioritized hierarchically, where the first objective takes absolute precedence over the second, and so on. \citet{huyuk2021multi} first introduced lexicographic order to MO-MAB, and \citet{cheng2024hierarchize} extended it to mixed Pareto-lexicographic environments.
However, lexicographic order may not adequately capture a user's overall satisfaction in real-world applications, where preferences often involve trade-offs rather than strict prioritization. For example, a user may prefer a \$10 meal with good taste over a \$9.5 meal with poor taste, even though cost is a priority.
Our work proposes a more general framework that incorporates a weighted order based on the user's explicit preference space. Notably, the lexicographic order becomes a special case of our proposed PAMO-MAB framework.

\section{Experiments}
\label{sec:app_exp}

In this section, we conduct numerical experiments to evaluate the effectiveness of our proposed algorithms under different user preference environments.

\subsection{Experiments in Hidden Preferences Environment}
\label{sec:app_exp_hidden}
In this section, we evaluate the performance of PRUCB-HP in modeling user preference \(\boldsymbol{c}_t\) and optimizing the overall reward when explicit user preference is not visible, but overall reward $g_{a_t,t}$ and reward $\boldsymbol{r}_{a_t, t}$ are revealed after each episode. 

\subsubsection{Experimental protocol.}
Given that PRUCB-HP models both the expected arms reward and user preference, we designed a new \emph{user-switching protocol} for evaluation. Figure \ref{fig:experiment3_protocol} illustrates this protocol with 3 users and 9 arms. Specifically, at each episode, one user is exposed to a block of arms (3 in our illustration). Only the arms within this block can be selected for this user. After one arm has been pulled, the system observes the reward $\boldsymbol{r}_{a_t, t}$ and user’s overall ratings $g_{a_t,t}$ corresponding to the pulled arm $a_t$. In the next episode, the arm block rotates to another user. The goal is to maximize the cumulative overall ratings from all users.

This protocol simulates real-world applications, such as recommender systems, where empirical multi-objective rewards (ratings) of arms (recommendation candidates) are obtained from a diverse set of users rather than a single fixed user. Additionally, users are not always exposed to a fixed set of arms (recommendation candidates). This user-switching protocol allows us to evaluate the algorithm’s ability to model arm reward and user preference, thus enabling the customized optimization of users’ overall ratings.
In Figure \ref{fig:experiment3_protocol}(b), we present an intuitive example of the protocol in the context of real-world hotel recommendations. Specifically, the blocks represent different cities (e.g., NYC, LA, CHI), and the hotel candidates within these cities correspond to the arms within the blocks. At each time step, a customer travels to a city, stays in a hotel recommended by the system, and leaves feedback (both objective and overall ratings) after her or his stay. In the next episode, the customer travels to a different city and encounters a new set of hotel options. The hotel recommender system needs to learn the multi-objective rewards of all hotel candidates from various customers and model each customer’s preference based on their multi-objective and overall feedback. This enables the system to customize optimal hotel recommendations tailored to individual user preference.

\begin{figure*}[t]
    \centering    
    \includegraphics[width=1\textwidth]{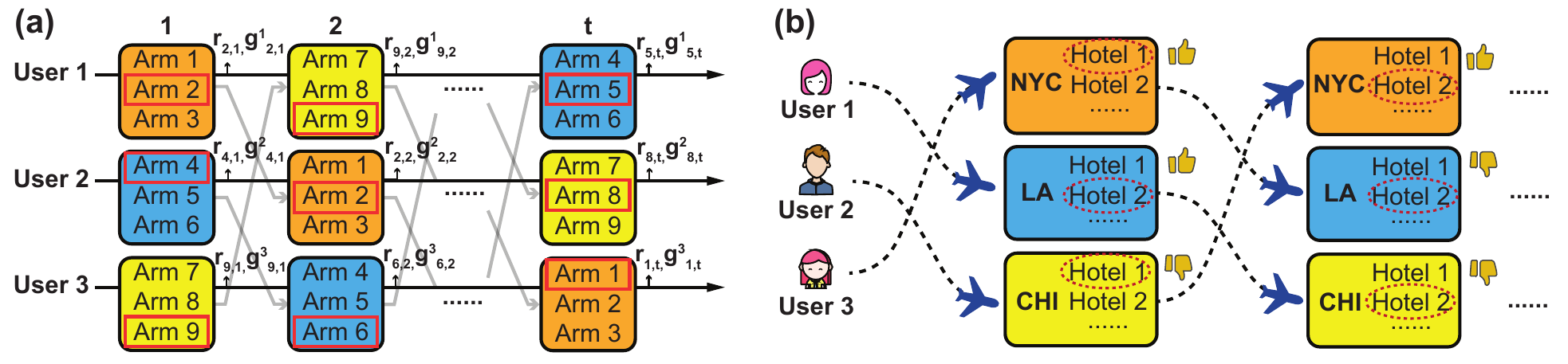}
    \caption{(a) Users switching protocol for experimental evaluation of hidden preference and multi-objective reward modelings. (b) One real-world example of the experimental protocol.
}
\vspace{-10pt}
    \label{fig:experiment3_protocol}
\end{figure*}

\begin{figure*}[t]
    \centering    
    \includegraphics[width=1\textwidth]{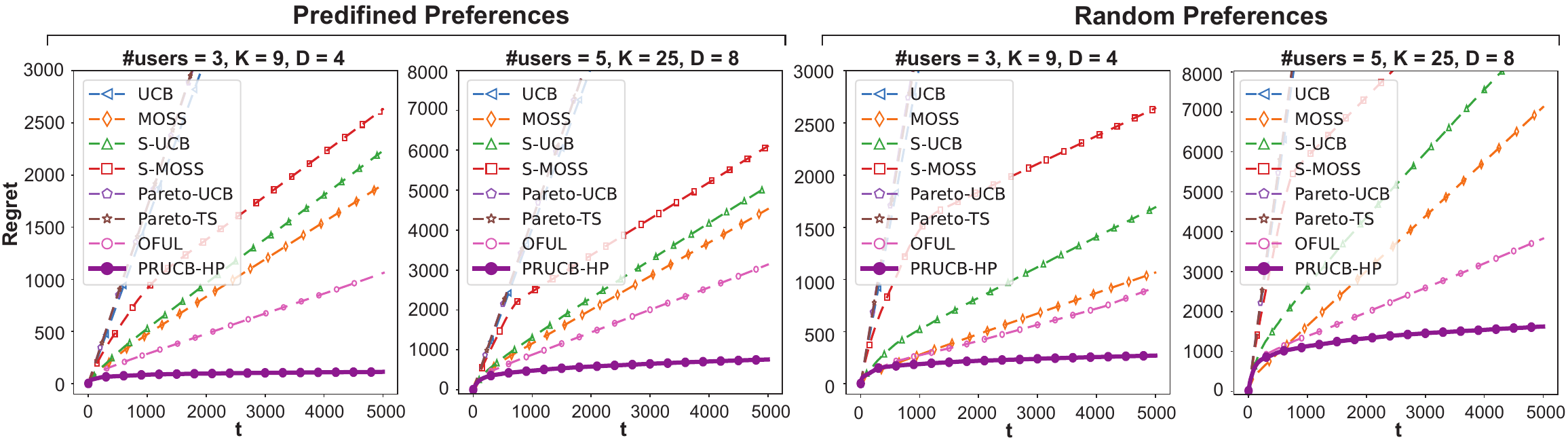}
    \caption{Regrets of different algorithms under hidden preference environment.
}
    \label{fig:experiment_3}
\end{figure*}

\begin{figure*}[t]
    \centering    
    \includegraphics[width=1\textwidth]{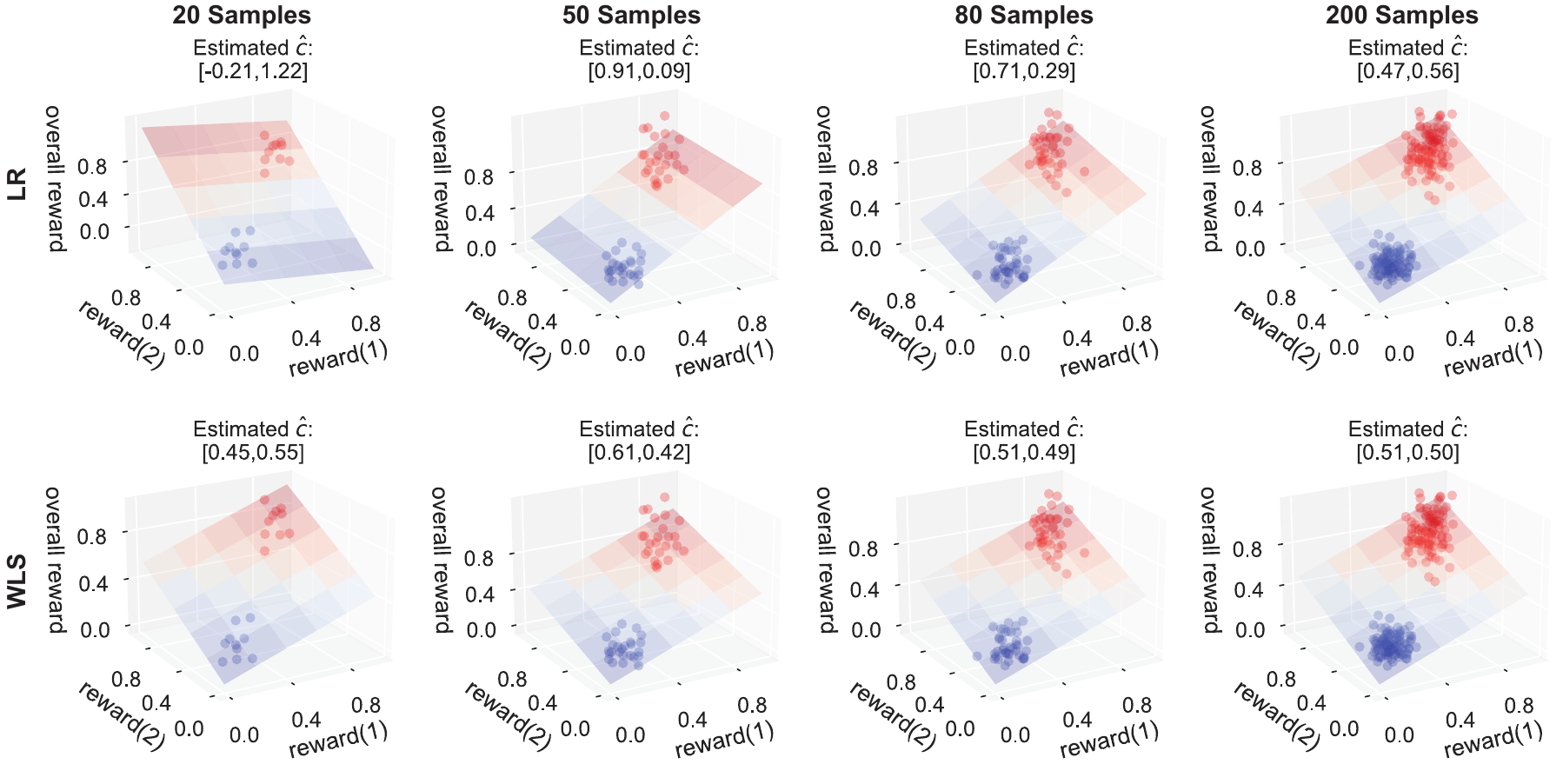}
    \caption{A simple 2D hidden preference PAMO-MAB for preference estimation comparison between standard LR estimator and our tailored WLS-estimator under different samples.}
\vspace{-10pt}
    \label{fig:lr_exp_2}
\end{figure*}

\subsubsection{Baselines.}
We compare our algorithm PRUCB-HP in terms of regret with the following multi-objective bandits algorithms.
\begin{itemize}[leftmargin=*]
\item 
S-UCB~\citep{drugan2013designing}:
the scalarized UCB algorithm, which scalarizes the multi-dimensional reward by assigning weights to each objective and then employs the single objective UCB algorithm \citet{auer2002finite}. Throughout the experiments, we assign each objective with equal weight. 
\item 
S-MOSS:
the scalarized UCB algorithm, which follows the similar way with S-UCB by scalarizing the multi-dimensional reward into a single one, but uses MOSS \citep{audibert2009minimax} policy for arm selection.
\item 
Pareto-UCB~\citep{drugan2013designing}:
the Pareto-based algorithm, which compares different arms by the upper confidence bounds of their expected multi-dimensional reward by Pareto order and pulls an arm uniformly from the approximate Pareto front.
\item 
Pareto-TS~\citep{yahyaa2015thompson}:
the Pareto-based algorithm, which makes use of the Thompson sampling technique to estimate the expected reward for every arm and selects an arm uniformly at random from the estimated Pareto front.
\item 
OFUL \cite{abbasi2011improved}: an adaptation variant of OFUL \cite{abbasi2011improved}, a widely used linear bandit benchmark, in MO-MAB problem. This variant estimates user preferences using ridge regression based on reward and overall reward feedback, replaces the input feature with the empirical reward estimate, and applies an $\epsilon = 0.05$ rate for exploration.
\item UCB~\citep{auer2002finite}, MOSS~\citep{audibert2009minimax}: classic algorithms for single-objective Multi-Armed Bandits (MAB) that rely on feedback from overall scalar rewards.
\end{itemize}

\subsubsection{Experimental settings.}
For evaluation, we use a synthetic dataset. 
Specifically, we consider the MO-MAB with $N$ users, $K$ arms, each arm $i \in [K]$ associated with a $D$-dimensional reward, where the reward of each objective $d$ follows a Gaussian distribution with a randomized mean $\boldsymbol{\mu}_i(d) \in [0,1]$ and variance of 0.01.
For each user $n \in [N]$ preference, we consider two settings including predefined preference and randomized preference. For predefined preference-aware structure, we define the mean preference $\boldsymbol{\overline{c}}^n$ as 
$ \boldsymbol{\overline{c}}_n(d) = 
2.0 \text{ if } d=j; 0.5 \text{ otherwise}
$, 
where $j \in [D]$ is randomly selected. The practical implication of this structure is that it represents a common scenario in which the user exhibits a markedly higher preference for one particular objective while showing little interest in others.
For randomized preference, the values of mean preference $\boldsymbol{\overline{c}}^n$ are randomly defined within $[0,5]$. For both setups, the instantaneous preference is generated under Gaussian distributions with corresponding means and variance of 0.5. 

For user-switching protocol, we set $N$ blocks in total, with each block containing 3 fixed arms. At each episode, each user will be randomly assigned one block without replacement. The learner can only select the arm within assigned block for each user. 

\subsubsection{Implementation.}
Similarly, we set $\alpha = 1$ in PRUCB-HP. For regularization coefficient, we set $\lambda = 1$. For confidence radius, we set $\beta_t = 0.1\sqrt{D \log(t)}$. For weight scalar, we set $\omega = D$.
We perform 10 trials up to round $T = 5000$ for each set of evaluation.

\subsubsection{Regret Results.}
We report average performance of the algorithms in Fig.~\ref{fig:experiment_3}. 
As shown, our proposed PRUCB-HP achieves superior results in terms of regret under all experimental settings compared to other competitors. This empirical evidence suggests that modeling user preference and leveraging this information for arm selection significantly enhances the performance of customized bandits optimization.

\subsubsection{Analysis on WLS Estimator}
In this section, we investigate the effectiveness of the proposed WLS-preference estimator with the carefully designed weight $w_t$. 
Specifically, we consider the 2D PAMO-MAB toy instance shown in Section \ref{sec:uniq_chllenge}. This instance contains two arms: Arm-1 with dominated mean reward $[0.2,0.2]^{\top}$, and Arm-2 with Pareto-optimal mean reward $[0.8,0.8]^{\top}$. The preference $\boldsymbol{c}_t$ at each step follows a Gaussian distribution with a mean of $[0.5, 0.5]$ and variance of $0.05$ for each objective dimension. 

We compare the preference estimation performance with the standard linear regression model under different sample numbers, i.e., 20, 50, 80, 200 (half for each arm).
The results are shown in Figure \ref{fig:lr_exp_2}. Notably, our proposed WLS-preference estimator consistently achieves better preference estimation performance (lower error to the ground-truth $[0.5, 0.5]$) than standard linear regression, verifying the effectiveness of our method for preference learning in hidden preference PAMO-MAB problem.

\begin{figure*}[ht]
    \centering    
    \includegraphics[width=1\textwidth]{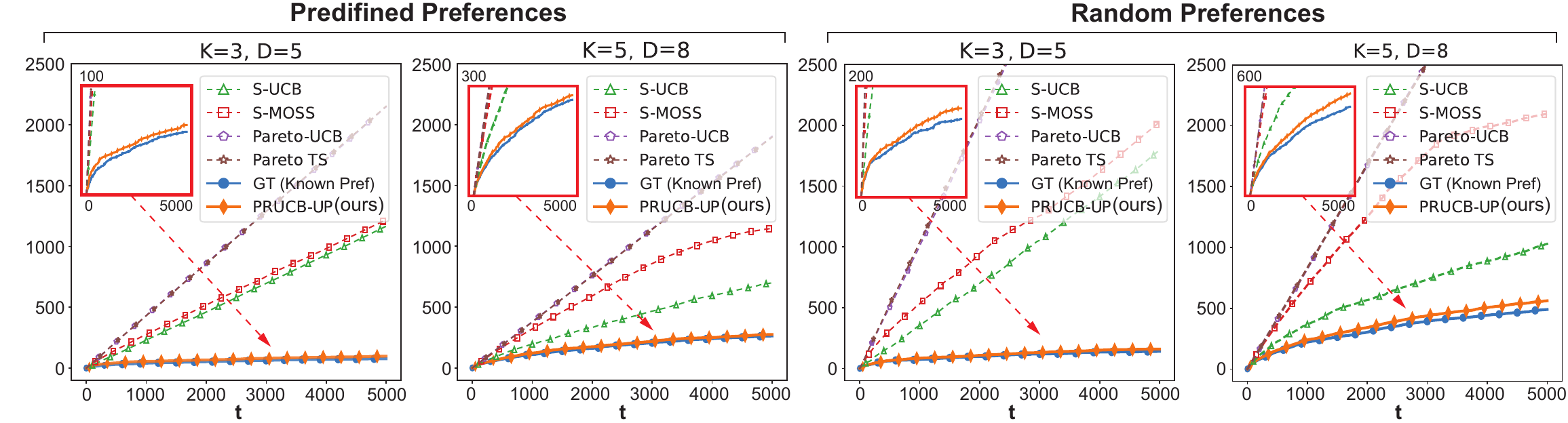}
    \caption{Regrets of different algorithms under unknown preference environment. 
}
\vspace{-5pt}
\label{fig: experiments_1}
\end{figure*}

\subsection{Experiments in Unknown Preference Environment}
\label{sec:app_exp_prucb_static}

In this section, we verify the capability of PRUCB-UP to model users' preferences $\boldsymbol{c}_t^n$ and optimize the overall reward cross all users in a stationary preference environment. 
We compare our algorithm PRUCB-UP in terms of regret with the following multi-objective bandits algorithms including S-UCB~\citep{drugan2013designing}, S-MOSS, Pareto-UCB~\citep{drugan2013designing} and Pareto-TS~\citep{yahyaa2015thompson}). Additionally, we compare the performance with GT: a PRUCB-UP variant with known preference as a ground truth, where we replace the preference estimation in Algorithm \ref{alg:PRUCB_UP} with the true preference estimation.

\textbf{Experimental Settings and Implementations.}
The generations of user preference $\boldsymbol{c}_{t}^n$ and objective rewards $\boldsymbol{r}_t$ follow the same settings as hidden environment.
For the implementations of the algorithms, following previous studies~\citep{auer2002finite, audibert2007tuning}, we set $\alpha=1$.
The time horizon is set to $T = 5000$ rounds, the number of user $N$ is set as 2, and we repeat 10 trials for each set of evaluation due to the randomness from both environment and algorithms.

\textbf{Regret Results.}
We report the averaged regret performance of the algorithms in Fig.~\ref{fig: experiments_1}.
It is evident that our algorithm significantly outperform other competitors in all experiments. This is expected since the competing algorithms are designed for Pareto-optimality identification and do not utilize the preference structure of users considered in this paper, which our algorithm explicitly exploits.
Additionally, from the zoom-in window, we observe that PRUCB-UP exhibits only a very slight performance degradation compared to GT, which knows the preference expectation in advance. This indicates that the proposed PRUCB-UP can effectively model user preference via empirical estimation in stationary preference environments.

\section{Proof of Proposition \ref{prop: lower_bd}}
\label{sce: app_pf_lower_bd}

\begin{lemma}[Variant of Lemma 7 in \cite{jun2018adversarial}]
\label{lemma: R_N_relation}
Assume that a bandit algorithm enjoys a sub-linear regret bound, then 
$\mathbb{E}[ N_{i,T} ] = o(T), \forall i \neq a^*$.
\end{lemma}

\begin{proof}
The sub-linear regret bound implies that for a sufficiently large $T$ there exists a constant $C > 0$ such that 
$\sum_{i=1}^{K} \mathbb{E}[ N_{i,T} ] \boldsymbol{\overline{c}}_t^{\top} (\mu_{a^*} - \mu_{i}) < CT $. 
Hence we have 
$\mathbb{E}[ N_{i,T} ] \boldsymbol{\overline{c}}_t^{\top} (\mu_{a^*} - \mu_{i}) \leq CT, \forall i \neq a^*$, implying 
$\mathbb{E}[ N_{i,T} ] < \frac{CT}{\boldsymbol{\overline{c}}_t^{\top} (\mu_{a^*} - \mu_{i})}$.
\end{proof}

\begin{proof}[Proof of Proposition \ref{prop: lower_bd}]  
We proof the Proposition \ref{prop: lower_bd} using a simple toy instance with 2 arms and 2 objectives. For the derivation based on the general instance, please refer to Appendix \ref{sce: app_pf_lower_bd_general}.

Let us consider an instance with two arms, where arm-1 has fixed reward $\boldsymbol{r}_{1} = [1,0]^{\top}$, arm-2 has fixed reward $\boldsymbol{r}_{2} = [0,1]^{\top}$. There two preference vectors $\boldsymbol{c}_1 = [1,0]^{\top}$ and $\boldsymbol{c}_2 = [0,1]^{\top}$.
Apparently, under preference $\boldsymbol{c}_1$, arm-1 is the optimal arm, while for preference $\boldsymbol{c}_2$, arm-2 is the optimal.

Assume there exists a preferences-free algorithm $\mathcal{A}$ (i.e., Pareto-UCB \citep{drugan2013designing}) achieving sub-linear regret under preference $\boldsymbol{c}_1$.
By Lemma \ref{lemma: R_N_relation}, we have 
\[
\mathbb{E}_{\boldsymbol{c}_1}[N_{2,T}] = \sum_{t \in [T]} \mathbb{P}_{\pi^{\mathcal{A}}_t \mid \boldsymbol{c}_1} (a_t = 2 \mid \boldsymbol{c}_1) = o(T),
\]
where $N_{2,T}$ denotes the number of pulls of arm-2 (suboptimal).

Since the policy $\pi^{\mathcal{A}}_t$ of $\mathcal{A}$ is  independent on the sequences of instantaneous preferences and preferences means, thus for algorithm $\mathcal{A}$ under preference $\boldsymbol{c}_2$, we have 
\[
\mathbb{E}_{\boldsymbol{c}_2}[N_{2,T}] = \sum_{t \in [T]} \mathbb{P}_{\pi^{\mathcal{A}}_t \mid \boldsymbol{c}_2} (a_t = 2 \mid \boldsymbol{c}_2) 
= \sum_{t \in [T]} \mathbb{P}_{\pi^{\mathcal{A}}_t \mid \boldsymbol{c}_1 } (a_t = 2 \mid \boldsymbol{c}_1)
= \mathbb{E}_{\boldsymbol{c}_1}[N_{2,T}] = o(T),
\]
where the second equality holds by the the definition of preferences-free algorithm in Definition \ref{def: pref_free_alg}.

However, recall that under preference $\boldsymbol{c}_2$, arm-2 is the optimal arm, which implies that the regret of $\mathcal{A}$ under preference $\boldsymbol{c}_2$ would be at least $\Omega(T)$, i.e.,
\[
R_{\boldsymbol{c}_2}(T) 
= 
1 \cdot \mathbb{E}_{\boldsymbol{c}_2}[N_{1,T}]
=
1 \cdot (T - \mathbb{E}_{\boldsymbol{c}_2}[N_{2,T}])
>
T - o(T)
=
\Omega(T).
\]
\end{proof}

\subsection{General Version}
\label{sce: app_pf_lower_bd_general}

\begin{definition}[Pareto order, \cite{lu2019multi}]
Let $\boldsymbol{u}, \boldsymbol{v} \in \mathbb{R}^{D}$ be two vectors.
\begin{itemize}[leftmargin=*]
    \item $\boldsymbol{u}$ dominates $\boldsymbol{v}$, denoted as $\boldsymbol{u} \succ \boldsymbol{v}$, if and only if $\forall d \in [D], \boldsymbol{u}(d) > \boldsymbol{v}(d)$.
    \item $\boldsymbol{v}$ is not dominated by $\boldsymbol{u}$, denoted as by $\boldsymbol{u} \not\succ \boldsymbol{v}$, if and only if $\boldsymbol{u} = \boldsymbol{v}$ or $\exists d \in [D], \boldsymbol{v}(d) > \boldsymbol{u}(d)$.
    \item $\boldsymbol{u}$ and $\boldsymbol{v}$ are incomparable, denoted as $\boldsymbol{u}||\boldsymbol{v}$, if and only if either vector is not dominated by the other, i.e., $\boldsymbol{u} \not\succ \boldsymbol{v}$ and $\boldsymbol{v} \not\succ \boldsymbol{u}$.
\end{itemize}
\end{definition}

\begin{proof}[Proof of Proposition \ref{prop: lower_bd} (General Version)]  
We first construct an arbitrary $K$-armed $D$-objective MO-MAB environment with conflicting reward objectives. 
Let each objective reward of each arm follow a distribution, i.e., $\boldsymbol{r}_{i,t}(d) \sim {\rm Dist}_{i,d}, \forall i \in [K], \forall d \in [D]$, with mean of $\mu_i(d)$.
Define $\mathcal{P} := \{ [{\rm Dist}_{1,d}]^D, [{\rm Dist}_{2,d}]^D, ..., [{\rm Dist}_{K,d}]^D  \} $ be the set of $K$-armed $D$-dimensional reward distributions.

We start with a simple case where the MO-MAB environment has two conflicting objective arms. 
Specifically, assume that $\exists u,v \in [K]$, s.t., 
\[
\boldsymbol{\mu}_{u} \neq \boldsymbol{\mu}_{v}; 
\quad \boldsymbol{\mu}_{u} || \boldsymbol{\mu}_{v}
\]
and 
\[
\boldsymbol{\mu}_{u} \succ \boldsymbol{\mu}_{i}, \boldsymbol{\mu}_{v} \succ \boldsymbol{\mu}_{i}, \forall i \in [k] \setminus \{u, v\}.
\]

Due to $\boldsymbol{\mu}_{u} \neq \boldsymbol{\mu}_{v}$, by taking the orthogonal complement of $\boldsymbol{\mu}_{u} - \boldsymbol{\mu}_{v}$, we can construct a subset $\mathcal{C}_{\varsigma^{+}} := \{ \boldsymbol{c} \in \mathbb{R}^D | \boldsymbol{c}^{\top} ( \boldsymbol{\mu}_{u} - \boldsymbol{\mu}_{v} ) = 0 \}$.
Next we consider two different constant preferences vector sets as the user's preferences, to construct two sets of preferences-aware MO-MAB scenarios.

\textbf{Scenarios $\mathcal{S}_{\varsigma^{+}}$}.
For any $\varsigma^{+} > 0$, we can construct a subset $\mathcal{C}_{\varsigma^{+}} := \{ \boldsymbol{c} \in \mathbb{R}^D | \boldsymbol{c}^{\top} ( \boldsymbol{\mu}_{u} - \boldsymbol{\mu}_{v} ) = \varsigma^{+} \}$. Specifically, the general form of $\boldsymbol{c}_{\varsigma^{+}} \in \mathcal{C}_{\varsigma^{+}}$ can be written as $\boldsymbol{c}_{\varsigma^{+}} = \frac{\varsigma^{+}}{ \Vert \boldsymbol{\mu}_{u} - \boldsymbol{\mu}_{v} \Vert_2^2 } (\boldsymbol{\mu}_{u} - \boldsymbol{\mu}_{v}) + \boldsymbol{c}_0$, where $\boldsymbol{c}_0$ is any vector such that $\boldsymbol{c}_0 \in \mathcal{C}_0$.
Then for the preferences-aware MO-MAB scenarios $\mathcal{S}_{\varsigma^{+}} := \{\mathcal{P} \times \mathcal{C}_{\varsigma^{+}} \}$ under the sets of arm reward distributions $\mathcal{P}$ and user preferences $\mathcal{C}_{\varsigma^{+}}$, it is obvious that arm $u$ is the optimal arm since 
$\boldsymbol{\mu}_{u} \succ \boldsymbol{\mu}_{i}, \forall i \in [K] \setminus \{u,v\}$ and $\boldsymbol{c}_{\varsigma^{+}}^{\top}  \boldsymbol{\mu}_{u} > \boldsymbol{c}_{\varsigma^{+}}^{\top} \boldsymbol{\mu}_{v}, \forall \boldsymbol{c}_{\varsigma^{+}} \in \mathcal{C}_{\varsigma^{+}}$.

\textbf{Scenarios $\mathcal{S}_{\varepsilon^{-}}$}.
Similarly, for any $\varepsilon^{-} < 0$, we can construct a subset $\mathcal{C}_{\varepsilon^{-}} := \{ \boldsymbol{c} \in \mathbb{R}^D | \boldsymbol{c}^{\top} ( \boldsymbol{\mu}_{u} - \boldsymbol{\mu}_{v} ) = \varepsilon^{-} \}$, with the general form of $\boldsymbol{c}_{\varepsilon^{-}} = \frac{\varepsilon^{-}}{ \Vert \boldsymbol{\mu}_{u} - \boldsymbol{\mu}_{v} \Vert_2^2 } (\boldsymbol{\mu}_{u} - \boldsymbol{\mu}_{v}) + \boldsymbol{c}_0$, where $\boldsymbol{c}_0$ is any vector such that $\boldsymbol{c}_0 \in \mathcal{C}_0$.
For scenarios $\mathcal{S}_{\varepsilon^{-}} := \{\mathcal{P} \times \mathcal{C}_{\varepsilon^{-}} \}$ with same arm rewards distributions $\mathcal{P}$ but modified user preferences $\mathcal{C}_{\varepsilon^{-}}$ sets, we have the arm $v$ to be the optimal.

We use $\mathbb{P}_{\varsigma^{+}}$ to denote the probability with respect to the scenarios $\mathcal{S}_{\varsigma^{+}}$, and use $\mathbb{P}_{\varepsilon^{-}}$ to denote the probability conditioned on $\mathcal{S}_{\varepsilon^{-}}$. Analogous expectations $\mathbb{E}_{\varsigma^{+}}[\cdot]$ and $\mathbb{E}_{\varepsilon^{-}}[\cdot]$ will also be used.
Let $\boldsymbol{\mathrm{a}}^{t-1} = \{ A_1,..., A_{t-1} \}$ and $\boldsymbol{\mathrm{r}}^{t-1} = \{ \boldsymbol{x}_{1}, ..., \boldsymbol{x}_{t-1} \}$ be the actual sequence of arms pulled and the sequence of received rewards up to episode $t-1$, and $\boldsymbol{\mathrm{H}}^{t-1} = \{ \langle A_1, \boldsymbol{x}_{1} \rangle, ..., \langle A_{t-1}, \boldsymbol{x}_{t-1} \rangle \}$ be the corresponding historical rewards sequence. 
For consistency, we define $\boldsymbol{\mathrm{a}}^{0}$, $\boldsymbol{\mathrm{r}}^{0}$ and $\boldsymbol{\mathrm{H}}^{0}$ as the empty sets.
Assume there exists a preferences-free algorithm $\mathcal{A}$ (i.e., Pareto-UCB \citep{drugan2013designing}) that is possibly dependent on historical rewards sequence $\boldsymbol{\mathrm{H}}^{t-1}$ at episode $t$ (classical assumption in MAB), 
achieving sub-linear regret in scenarios $\mathcal{S}_{\varsigma^{+}}$. Let $N_{i,T}$ be the number of pulls of arm $i$ by $\mathcal{A}$ up to $T$ episode. By Lemma \ref{lemma: R_N_relation}, we have 
\begin{equation}
\label{eq: E_N_+}
\mathbb{E}_{\varsigma^{+}}[N_{*,T}] = \mathbb{E}_{\varsigma^{+}}[N_{u,T}] = T - o(T).
\end{equation}

Since the policy $\pi^{\mathcal{A}}_t$ of $\mathcal{A}$ is possibly dependent on $\boldsymbol{\mathrm{H}}^{t-1}$ but independent on the sequences of instantaneous preferences $\boldsymbol{\mathrm{c}}^{\top}$ and preferences means $\boldsymbol{\overline{c}}$, for $t \in (0,T\}$, $i \in [K]$ we have 

\begin{equation}
\label{eq: E_0-E_epsilon}
\begin{aligned}
& \mathbb{E}_{\varsigma^{+}}[ \mathds{1}_{a_t = i} ] - \mathbb{E}_{\varepsilon^{-}}[ \mathds{1}_{a_t = i} ] \\
& \quad = 
\sum_{\substack{\boldsymbol{\mathrm{a}}^{t-1} \in [K]^{t-1}}} \int_{\substack{\boldsymbol{\mathrm{r}}^{t-1} \in [0,1]^{D \times (t-1)}}}
\mathbb{P}_{\pi_t^{\mathcal{A}}} (a_t=i | \boldsymbol{\mathrm{H}}^{t-1}, [\boldsymbol{c}_0]^{\top}, \boldsymbol{c}_0)
\cdot
\mathbb{P}_{\varsigma^{+}}(\boldsymbol{\mathrm{H}}^{t-1} ) 
d\boldsymbol{\mathrm{r}}^{t-1} \\
& \quad \quad \quad \quad -
\sum_{\substack{\boldsymbol{\mathrm{a}}^{t-1} \in [K]^{t-1}}} \int_{\substack{\boldsymbol{\mathrm{r}}^{t-1} \in [0,1]^{D \times (t-1)}}}
\mathbb{P}_{\pi_t^{\mathcal{A}}} (a_t=i | \boldsymbol{\mathrm{H}}^{t-1}, [\boldsymbol{c}_{\varepsilon^{-}}]^{\top}, \boldsymbol{c}_{\varepsilon^{-}})
\cdot
\mathbb{P}_{\varepsilon^{-}}(\boldsymbol{\mathrm{H}}^{t-1} ) 
d\boldsymbol{\mathrm{r}}^{t-1} \\
& \quad \underset{(a)}{=}
\sum_{\substack{\boldsymbol{\mathrm{a}}^{t-1} \in [K]^{t-1}}} \int_{\substack{\boldsymbol{\mathrm{r}}^{t-1} \in [0,1]^{D \times (t-1)}}}
\mathbb{P}_{\pi_t^{\mathcal{A}}} (a_t=i | \boldsymbol{\mathrm{H}}^{t-1})
\cdot
\bigg(
\mathbb{P}_{\varsigma^{+}}(\boldsymbol{\mathrm{H}}^{t-1} )
-
\mathbb{P}_{\varepsilon^{-}}(\boldsymbol{\mathrm{H}}^{t-1} )
\bigg )
d\boldsymbol{\mathrm{r}}^{t-1}, \\
\end{aligned}
\end{equation}

with 
\begin{equation}
\label{eq: P_0=P_epsilon}
\begin{aligned}
& \mathbb{P}_{\varsigma^{+}}(\boldsymbol{\mathrm{H}}^{t-1} ) 
= \prod_{\tau=1}^{t-1}
\left(
\mathbb{P}_{\varsigma^{+}}(\boldsymbol{\mathrm{H}}^{\tau-1})
\cdot
\mathbb{P}_{\pi_{\tau}^{\mathcal{A}}}(a_{\tau} = A_{\tau}| \boldsymbol{\mathrm{H}}^{\tau-1} )
\cdot
\mathbb{P}_{\varsigma^{+}}(r_{a_{\tau}} = \boldsymbol{x}_{\tau} | a_{\tau}=A_{\tau})
\right) , \\
& \mathbb{P}_{\varepsilon^{-}}(\boldsymbol{\mathrm{H}}^{t-1} ) 
= \prod_{\tau=1}^{t-1}
\left(
\mathbb{P}_{\varepsilon^{-}}(\boldsymbol{\mathrm{H}}^{\tau-1})
\cdot
\mathbb{P}_{\pi_{\tau}^{\mathcal{A}}}(a_{\tau} = A_{\tau}| \boldsymbol{\mathrm{H}}^{\tau-1} )
\cdot
\mathbb{P}_{\varepsilon^{-}}(r_{a_{\tau}} = \boldsymbol{x}_{\tau} | a_{\tau}=A_{\tau})
\right).
\end{aligned}
\end{equation}

where $\boldsymbol{c}_0, \boldsymbol{c}_{\varepsilon^{-}}$ can be any constant vectors such that $\boldsymbol{c}_0 \in \mathcal{C}_0$ and $\boldsymbol{c}_0 \in \mathcal{C}_{\varepsilon^{-}}$. (a) holds since the policy $\pi^{\mathcal{A}}_t$ is independent of $\boldsymbol{\mathrm{c}}^{\top}$ and $\boldsymbol{\overline{c}}$. Hence 
$
\mathbb{P}_{\pi_t^{\mathcal{A}}} (a_t=i | \boldsymbol{\mathrm{H}}^{t-1})
=
\mathbb{P}_{\pi_t^{\mathcal{A}}} (a_t=i | \boldsymbol{\mathrm{H}}^{t-1}, [\boldsymbol{c}_0]^{\top}, \boldsymbol{c}_0)
= \mathbb{P}_{\pi_t^{\mathcal{A}}} (a_t=i | \boldsymbol{\mathrm{H}}^{t-1}, [\boldsymbol{c}_{\varepsilon^{-}}]^{\top}, \boldsymbol{c}_{\varepsilon^{-}})$
(recall the definition of preferences-free algorithm in Definition \ref{def: pref_free_alg}).

Additionally, please note that both scenarios $\mathcal{S}_{\varsigma^{+}}$ and $\mathcal{S}_{\varepsilon^{-}}$ share the same arm reward distributions $\mathcal{P}$, which implies that for any $t \in (0,T]$ and $A \in [K]$, we have

\[
\mathbb{P}_{\varsigma^{+}}(r_{a_{t}} = \boldsymbol{x}_{t} | a_{t}=A)
=
\mathbb{P}_{\varepsilon^{-}}(r_{a_{t}} = \boldsymbol{x}_{t} | a_{t}=A).
\]

Combining result above with Eq. \ref{eq: P_0=P_epsilon} and using the fact that $\boldsymbol{\mathrm{H}}^{0}:= \emptyset$ for both  $\mathcal{S}_{\varsigma^{+}}$ and $\mathcal{S}_{\varepsilon^{-}}$, it can be easily verified by induction that 
$\mathbb{P}_{\varsigma^{+}}(\boldsymbol{\mathrm{H}}^{t-1} ) 
= 
\mathbb{P}_{\varepsilon^{-}}(\boldsymbol{\mathrm{H}}^{t-1} )$.
Plugging this back to Eq \ref{eq: E_0-E_epsilon} yields

\begin{equation}
\label{eq: E_0=E_epsilon}
\begin{aligned}
& \mathbb{E}_{\varsigma^{+}}[ \mathds{1}_{a_t = i} ] - \mathbb{E}_{\varepsilon^{-}}[ \mathds{1}_{a_t = i} ] \\
& \quad =
\sum_{\substack{\boldsymbol{\mathrm{a}}^{t-1} \in [K]^{t-1}}} \int_{\substack{\boldsymbol{\mathrm{r}}^{t-1} \in [0,1]^{D \times (t-1)}}}
\mathbb{P}_{\pi_t^{\mathcal{A}}} (a_t=i | \boldsymbol{\mathrm{H}}^{t-1})
\cdot
\bigg(
\cancelto{0}
{
\mathbb{P}_{\varsigma^{+}}(\boldsymbol{\mathrm{H}}^{t-1} )
-
\mathbb{P}_{\varepsilon^{-}}(\boldsymbol{\mathrm{H}}^{t-1} )
}
\bigg )
d\boldsymbol{\mathrm{r}}^{t-1}
= 0.
\end{aligned}
\end{equation}

By summing over $T$ we can derive that 

\[
\mathbb{E}_{\varsigma^{+}}[ N_{i,T} ]
=
\sum_{t=1}^{\top} \mathbb{E}_{\varsigma^{+}}[ \mathds{1}_{a_t = i} ]
=
\sum_{t=1}^{\top} \mathbb{E}_{\varepsilon^{-}}[ \mathds{1}_{a_t = i} ]
=
\mathbb{E}_{\varepsilon^{-}}[ N_{i,T} ].
\]

Combining above result with Eq. \ref{eq: E_N_+} gives that= 
\[
\mathbb{E}_{\varsigma^{+}}[N_{u,T}] = \mathbb{E}_{\varepsilon^{-}}[N_{u,T}] = T - o(T) = \Omega(T).
\]

However, recall that in scenarios $\mathcal{S}_{\varepsilon^{-}}$, $u$ is a suboptimal arm, which implies that the regret of $\mathcal{A}$ in $\mathcal{S}_{\varepsilon^{-}}$ would be at least $\Omega(T)$, i.e.,
\[
\begin{aligned}
R(T) 
& = 
\sum_{i \neq v} \boldsymbol{c}_{\varepsilon^{-}}^{\top} (\mu_{v} - \mu_{i}) \mathbb{E}_{\varepsilon^{-}}[ N_{i,T}] \\
& >
|\varepsilon^{-}| \mathbb{E}_{\varepsilon^{-}}[ N_{u,T}]
=
\Omega(T).
\end{aligned}
\]

The analysis above indicates that for the case with two objective-conflicting arms $u, v$, for any preferences-free algorithm $\mathcal{A}$, if there exists a $ \varsigma^{+}>0$ such that $\mathcal{A}$ can achieve sub-linear regret in scenarios $\mathcal{S}_{\varsigma^{+}}$, then it will suffer the regret of the order $\Omega(T)$ in scenarios $\mathcal{S}_{\varepsilon^{-}}$ for all $\varepsilon^{-}<0$, and vice verse (i.e., sub-linear regret in $\varepsilon^{-}>0$ while $\Omega(T)$ regret in $\mathcal{S}_{\varsigma^{+}}$).

Next we extend the solution to the MO-MAB environment containing more than two objective-conflicting arms.
Specifically, for each conflicting arm $i$, we can simply select another conflicting arm $j$ to construct a pair, and apply the solution we derived in two-conflicting arms case. 
By traversing all conflicting arms, we have that for any preferences-free algorithm $\mathcal{A}$ achieving sub-linear regret in a scenarios set $\mathcal{S}_{0}$ with a subset of conflicting arms $\{a^*\}$ as the optimal, there must exists another scenarios set $\mathcal{S}^{\prime}_0$ for each arm $i \in \{a^*\}$ such that the arm $i$ is considered as suboptimal and lead to the regret of order $\Omega(T)$.
This concludes the proof of Proposition \ref{prop: lower_bd}.
\end{proof}

\begin{remark}
\label{remark:lower_bd}
As a side-product of the analysis above, we have:
If one MO-MAB environment contains multiple objective-conflicting arms, i.e., $\vert \mathcal{O}^{*} \vert \geq 2$, where $\mathcal{O}^{*}$ is the Pareto Optimal front. 
Then for any Pareto-Optimal arm $i \in \mathcal{O}^{*}$, there exists preferences subsets such that the arm $i$ is suboptimal.
\end{remark}

\section{Analyses for Section \ref{sec: hidden} (Hidden Preference) }

Our main result of Theorem \ref{theorem:up_bd_hiden} in Section \ref{sec: hidden} indicates that the proposed PUCB-HPM under hidden preference environment achieves sublinear expected regret $R(T) \leq \tilde{\mathcal{O}}(D \sqrt{T})$. 
To prove this, we need two key components. 
The first is to show that the value of $\hat{\boldsymbol{r}}_{i,t}$, the matrix of $\boldsymbol{V}_t$, and the value of $\boldsymbol{\hat{c}}_{t}$ are good estimators of $\boldsymbol{\mu}_i$, $\mathbb{E}[\boldsymbol{V}_t]$ and $\overline{\boldsymbol{c}}$ respectively.
The second is to show that as long as the aforementioned high-probability event holds, we have some control on the growth of the regret.
We show the analyses regarding these two components in the following sections.
In the following proofs, we focus on the user-specific regret $R^{n}(T)$ and then sum the result over whole user set $N$. For notational simplicity, we fix a user $n$ and omit the user index $n$ as superscript.

\subsection{Proof of Lemma \ref{lemma: R_normed}}
The following lemma directly come from the definition of the sub-Gaussian variables.

\begin{lemma}[\citet{lamprier2018profile}]
\label{lemma: sub_gaussian_addity}
Let $X_1$ and $X_2$ be two sub-Gaussian variables with respective constant $R_1$ and $R_2$, let $\alpha_1$ and $\alpha_2$ be two real scalars. Then the variable $\alpha_1 X_1 + \alpha_2 X_2$ is sub-Gaussian too, with constant $\sqrt{\alpha_1^2 X_1^2 + \alpha_2^2 X_2^2}$.
\end{lemma}

\begin{proof}[Proof of Lemma \ref{lemma: R_normed}]
\label{sec: app_pf_lemma_R_normed}

By Assumption \ref{assmp: hpm_3} and Lemma \ref{lemma: sub_gaussian_addity}, we have that that for any $d \in [D]$ and any $t$, $\boldsymbol{\zeta}_{t}(d) = \boldsymbol{c}_t - \boldsymbol{\overline{c}}$ is also $R$ sub-Gaussian:
$\mathbb{E}[e^{x\boldsymbol{\zeta}_{t}(d)}] \leq e^{\frac{x^2 R^2}{2}}$.
For $\boldsymbol{\zeta}_{t}^{\top} \boldsymbol{r}_{a_{t}, t}$, the independence implies:
\[
\mathbb{E}[e^{\lambda \boldsymbol{\zeta}_{t}^{\top} \boldsymbol{r}_{a_{t}, t} }]
=
\mathbb{E}[e^{\lambda \sum_{d \in [D]} \boldsymbol{\zeta}_{t}(d)) \boldsymbol{r}_{a_{t}, t}(d) }]
= 
\prod_{d \in [D]} \mathbb{E}[e^{\lambda \boldsymbol{\zeta}_{t}(d)) \boldsymbol{r}_{a_{t}, t}(d) }]
\leq
e^{\frac{\lambda^2 R^2 \sum_{d \in [D]} \boldsymbol{r}_{a_{t}, t}(d)^2 }{2}}
=
e^{\frac{\lambda^2 R^2 \Vert \boldsymbol{r}_{a_{t}, t} \Vert_2^2 }{2}},
\]
where the inequality holds by sub-Gaussian definition and Lemma \ref{lemma: sub_gaussian_addity} that if 
$a$ is $\sigma$-subgaussian, then $ba$ is $\vert b \vert R$-subgaussian. 
The result above implies that the overall residual noise $\boldsymbol{\zeta}_{t}^{\top} \boldsymbol{r}_{a_{t}, t}$ is $\Vert \boldsymbol{r}_{a_{t}, t} \Vert_2 R$-subgaussian, which is collinear with the input reward sample $\boldsymbol{r}_{a_{t}, t}$. 

By applying the the weight $w_{t}$ and using Lemma \ref{lemma: sub_gaussian_addity}, we can derive that the normed overall residual term $\sqrt{w_t} \boldsymbol{\zeta}_{t}^{\top} \boldsymbol{r}_{a_{t}, t} = \frac{\sqrt{\omega}}{\Vert \boldsymbol{r}_{a_{t}, t} \Vert_2} \boldsymbol{\zeta}_{t}^{\top} \boldsymbol{r}_{a_{t}, t}$ is $\sqrt{\omega} R$-subgaussian, which eliminates the heteroscedasticity of the residual error induced by selected arm.
\end{proof}

\subsection{Proof of Lemma \ref{lemma:c_estimator_conf_bd} (Confidence Ellipsoid for $\boldsymbol{\overline{c}}$)}
\label{sec:app_proof_lemma_c_estimator_conf_bd}

First we state two useful lemmas that will be utilized in our confidence analysis of preference estimator:

\begin{lemma}[Self-Normalized Bound for Vector-Valued Martingales, Variant of Theorem 1 in~\citep{abbasi2011improved}]
\label{lemma: self_norm_bound}
Let $\{\mathcal{F}_t\}_{t=0}^{\infty}$ be a filtration, and let $\{\zeta_t^{\prime} = \frac{\sqrt{\omega} \boldsymbol{\zeta}_{t}^{\top} \boldsymbol{r}_{a_{t}, t}}{\Vert \boldsymbol{r}_{a_{t}, t} \Vert_2} \}_{t=1}^{\infty}$ be a real-valued stochastic process such that $\zeta_t$ is $\mathcal{F}_t$-measurable, $\mathbb{E}[\boldsymbol{\zeta_t} \mid \mathcal{F}_{t-1}]=\boldsymbol{0}$ and $\boldsymbol{\zeta_t}$ is conditionally $R$-sub-Gaussian for some $R \geq 0$.
Let $\{\boldsymbol{r}_t^{\prime} = \frac{\sqrt{\omega} \boldsymbol{r}_{a_{t}, t}}{\Vert \boldsymbol{r}_{a_{t}, t} \Vert_2} \}_{t=1}^{\infty}$ be an $\mathbb{R}^d$-valued stochastic process such that $\boldsymbol{r}_t$ is $\mathcal{F}_{t}$-measurable. 
Assume that $\boldsymbol{V} \in \mathbb{R}^{d \times d}$ is a positive definite matrix, and define $\boldsymbol{\overline{V}}_{t} = \boldsymbol{V} + \sum_{\ell=1}^{\top} \boldsymbol{r}_{a_{\ell}, \ell}^{\prime} {\boldsymbol{r}_{a_{\ell}, \ell}^{\prime}}^{\top}$.
Then for any $\alpha \geq 0$, with probability at least $1 - \alpha$, for all $t \geq 1$, we have
\[
\left \Vert \sum_{\ell=1}^{\top} \zeta_{\ell}^{\prime} \boldsymbol{r}_{a_{\ell},\ell} \right \Vert_{\boldsymbol{\overline{V}}_{t}^{-1}}^2
\leq
2 \omega R^2 \log \left( \frac{ \det \left( \boldsymbol{\overline{V}}_{t} \right)^{\frac{1}{2}} \det \left( \boldsymbol{V} \right)^{-\frac{1}{2}} }{\alpha} \right).
\]
\end{lemma}

\begin{proof}
By definition and Tower property (Lemma \ref{lemma: Tower}), for any $t \in [T]$, we have 
\[
\mathbb{E}[\zeta_t' \boldsymbol{r_{a_t,t}^{\prime}}|\mathcal{F}_{t-1}] = \mathbb{E} \left[ \frac{\omega}{\Vert \boldsymbol{r}_{a_t,t} \Vert_2^2} \boldsymbol{\zeta}_t^{\top} \boldsymbol{r}_{a_t,t} \boldsymbol{r}_{a_t,t} \mid \mathcal{F}_{t-1} \right] = \mathbb{E} \left[ \mathbb{E} \left[ \frac{\omega}{\Vert \boldsymbol{r}_{a_t,t} \Vert_2^2} \boldsymbol{\zeta}_t^{\top} \boldsymbol{r}_{a_t,t} \boldsymbol{r}_{a_t,t} \mid \boldsymbol{r}_{a_t,t}, \mathcal{F}_{t-1} \right] \mid \mathcal{F}_{t-1} \right].
\]

By Assumption 3 (independence between reward and preference), we have $(\boldsymbol{r}_{a_t,t} \perp \boldsymbol{\zeta}_t) \mid \mathcal{F}_{t-1}$, and since $\boldsymbol{r}_{a_t,t}$ is known inside the inner expectation, we have 
\[
\mathbb{E}[\boldsymbol{\zeta}_t \mid r_{a_t,t},\mathcal{F}_{t-1}]
= \mathbb{E}[\boldsymbol{\zeta}_t \mid \mathcal{F}_{t-1}]
= \boldsymbol{0}.
\]

Hence 
\[
\mathbb{E} \left[ \frac{\omega}{\Vert \boldsymbol{r}_{a_t,t} \Vert_2^2} 
\boldsymbol{\zeta}_t^{\top} \boldsymbol{r}_{a_t,t} \boldsymbol{r}_{a_t,t} \mid \boldsymbol{r}_{a_t,t},\mathcal{F}_{t-1} \right]
=\frac{\omega}{\Vert \boldsymbol{r}_{a_t,t} \Vert_2^2} \boldsymbol{r}_{a_t,t} \boldsymbol{r}_{a_t,t}^{\top} \cdot \mathbb{E}[ \boldsymbol{\zeta}_t \mid \boldsymbol{r}_{a_t,t},\mathcal{F}_{t-1} ]
= \boldsymbol{0}
\]
which implies $\mathbb{E}[\boldsymbol{\zeta}_t^{\prime} \boldsymbol{r}_{a_t,t}^{\prime} \mid \mathcal{F}_{t-1}] = \boldsymbol{0}$.
Thus, the process $S_t = \sum_{\ell=1}^{\top} \boldsymbol{\zeta}_{\ell}^{\prime} \boldsymbol{r}_{a_\ell,\ell}'$ remains a \textbf{martingale} with respect to $\{\mathcal{F}_t\}$, even though $\boldsymbol{r}_{a_\ell,\ell}$ is not $\mathcal{F}_{\ell-1}$-measurable.

\emph{Please note this differs from \citep{abbasi2011improved}, where the feature vector $X_{t}$ must be $\mathcal{F}_{t-1}$measureable so that $\mathbb{E}[\eta_t \boldsymbol{X}_{t} \mid \mathcal{F}_{t-1}]=\boldsymbol{X}_{t} \mathbb{E}[ \eta_t \mid \mathcal{F}_{t-1}]=\boldsymbol{0}$, with $\eta_t$ as conditionally zero-mean sub-gaussian random variable.}

With the conditions that $S_t$ is martingale and $\boldsymbol{\zeta}_t$ is conditionally sub-Gaussian, the process $M_{t}^{u}=\exp(\boldsymbol{u}^{\top} S_t - \frac{1}{2}\Vert \boldsymbol{u} \Vert_{V_t}^2)$ is a nonnegative supermartingale w.r.t $\{F_t\}$ for arbitrary $\boldsymbol{u}$. Hence, Lemma 8 and Lemma 9 in~\citep{abbasi2011improved} can be directly applied to derive the self-normalized bound stated as Lemma \ref{lemma: self_norm_bound} in our paper. The proof is omitted here, as it follows identically from the proof presented in Lemma 9 of \citet{abbasi2011improved}.
\end{proof}

\begin{lemma}[Determinant-Trace Inequality~\citep{abbasi2011improved}, Lemma 10]
\label{lemma: det_trace_ineq}
Suppose $\boldsymbol{X}_{1}, ... , \boldsymbol{X}_{t} \in \mathbb{R}^d$ and $\Vert \boldsymbol{X}_{\ell} \Vert_2 \leq L, \forall \ell \in [1,t]$. 
Let $\boldsymbol{\overline{V}}_{t} = \lambda \boldsymbol{I} + \sum_{\ell=1}^{\top} \boldsymbol{X}_{\ell} \boldsymbol{X}_{\ell}^{\top}$ for some $\lambda > 0$, then
\[
\det \left( \boldsymbol{\overline{V}}_t \right) \leq \left(\lambda + \frac{t L^2}{d} \right)^d.
\]
\end{lemma}


\begin{proof}[Proof of Lemma \ref{lemma:c_estimator_conf_bd}]

According to the definition of estimated vector $\hat{\boldsymbol{c}}_{t}$ in Algorithm \ref{alg:PRUCB_HP}, we have 
\[
\hat{\boldsymbol{c}}_{t} 
=
\boldsymbol{V}_{t-1}^{-1} \sum_{\ell = 1}^{t-1} w_{\ell} g_{a_{\ell}, \ell} \boldsymbol{r}_{a_{\ell}, \ell} 
= 
\boldsymbol{V}_{t-1}^{-1} \sum_{\ell = 1}^{t-1} w_{\ell} (\boldsymbol{\overline{c}}^{\top} \boldsymbol{r}_{a_{\ell}, \ell} + \boldsymbol{\zeta}_{\ell}^{\top} \boldsymbol{r}_{a_{\ell}, \ell}) \boldsymbol{r}_{a_{\ell}, \ell},
\]
where the second equality followed by the definition of overall reward $g_{a_t,t} = (\boldsymbol{\overline{c}} + \boldsymbol{\zeta}_t)^{\top} \boldsymbol{r}_{a_t,t} = \boldsymbol{\overline{c}}^{\top} \boldsymbol{r}_{a_t,t} + \boldsymbol{\zeta}_t^{\top} \boldsymbol{r}_{a_t,t}$, $\boldsymbol{\zeta}_t^{\top} \in \mathbb{R}^{D}$ is an independent noise term over $\overline{\boldsymbol{c}}$ to denote the randomness of $\boldsymbol{c}_t$.

This equation further implies that the difference between estimated vector $\hat{\boldsymbol{c}}_{t}$ and the unknown vector $\overline{\boldsymbol{c}}$ can be decomposed as:
\begin{equation}
\begin{aligned}
\label{eq:hidden_conf_c}
\Vert \hat{\boldsymbol{c}}_{t} -\overline{\boldsymbol{c}} \Vert_{\boldsymbol{V}_{t-1}}
& = 
\Big \Vert 
\boldsymbol{V}_{t-1}^{-1} \sum_{\ell = 1}^{t-1} w_{\ell} (\boldsymbol{\overline{c}}^{\top} \boldsymbol{r}_{a_{\ell}, \ell} + \boldsymbol{\zeta}_{\ell}^{\top} \boldsymbol{r}_{a_{\ell}, \ell}) \boldsymbol{r}_{a_{\ell}, \ell} - \overline{\boldsymbol{c}}
\Big \Vert_{\boldsymbol{V}_{t-1}} \\
& = 
\Big \Vert 
\boldsymbol{V}_{t-1}^{-1} \sum_{\ell = 1}^{t-1} w_{\ell} (\boldsymbol{\overline{c}}^{\top} \boldsymbol{r}_{a_{\ell}, \ell} + \boldsymbol{\zeta}_{\ell}^{\top} \boldsymbol{r}_{a_{\ell}, \ell}) \boldsymbol{r}_{a_{\ell}, \ell} - 
\boldsymbol{V}_{t-1}^{-1} \big( \sum_{\ell = 1}^{t-1} w_{\ell} \boldsymbol{r}_{a_{\ell},\ell} \boldsymbol{r}_{a_{\ell},\ell}^{\top} + \lambda \boldsymbol{I} \big)
\overline{\boldsymbol{c}}
\Big \Vert_{\boldsymbol{V}_{t-1}} \\
& = 
\Big \Vert 
\boldsymbol{V}_{t-1}^{-1} \sum_{\ell = 1}^{t-1} w_{\ell} \boldsymbol{\zeta}_{\ell}^{\top} \boldsymbol{r}_{a_{\ell}, \ell} \boldsymbol{r}_{a_{\ell}, \ell} - 
\lambda \boldsymbol{V}_{t-1}^{-1} \overline{\boldsymbol{c}}
\Big \Vert_{\boldsymbol{V}_{t-1}} \\
& \underset{(a)}{\leq}
\underbrace{
\Big \Vert 
\boldsymbol{V}_{t-1}^{-1} \sum_{\ell = 1}^{t-1} w_{\ell} \boldsymbol{\zeta}_{\ell}^{\top} \boldsymbol{r}_{a_{\ell}, \ell} \boldsymbol{r}_{a_{\ell}, \ell}
\Big \Vert_{\boldsymbol{V}_{t-1}} 
}_{\text{Stochastic error: } I_1}
+
\underbrace{
\Big \Vert 
\lambda \boldsymbol{V}_{t-1}^{-1} \overline{\boldsymbol{c}}
\Big \Vert_{\boldsymbol{V}_{t-1}} 
}_{\text{Regularization error: } I_2},
\end{aligned}
\end{equation}

where (a) holds followed by the triangle inequality that $\Vert \boldsymbol{X} + \boldsymbol{Y} \Vert_{\boldsymbol{A}} \leq \Vert \boldsymbol{X} \Vert_{\boldsymbol{A}} + \Vert \boldsymbol{Y} \Vert_{\boldsymbol{A}}$.

\textbf{Bounding term $I_1$.}
For the stochastic error term $I_1$, we first construct two auxiliary terms: 
\[
\zeta_{\ell}^{\prime} = \frac{\sqrt{\omega} \boldsymbol{\zeta}_{\ell}^{\top} \boldsymbol{r}_{a_{\ell}, \ell}}{\Vert \boldsymbol{r}_{a_{\ell}, \ell} \Vert_2}, 
\text{ and }
\boldsymbol{r}_{a_{\ell}, \ell}^{\prime} = \frac{\sqrt{\omega} \boldsymbol{r}_{a_{\ell}, \ell}}{\Vert \boldsymbol{r}_{a_{\ell}, \ell} \Vert_2}.
\]


For $\zeta_{\ell}^{\prime}$, by Lemma \ref{lemma: R_normed}, we have $\zeta_{\ell}^{\prime}$ is $\sqrt{\omega} R$-subgaussian.

For $\boldsymbol{r}_{a_{\ell}, \ell}^{\prime}$, we have 
\[\Vert \boldsymbol{r}_{a_{\ell}, \ell}^{\prime} \Vert _2 = \sqrt{\omega} \frac{\Vert \boldsymbol{r}_{a_{\ell}, \ell} \Vert_2}{\Vert \boldsymbol{r}_{a_{\ell}, \ell} \Vert_2}=\sqrt{\omega}.
\]

With the notations of $\zeta_{\ell}^{\prime}$ and $\boldsymbol{r}_{a_{\ell}, \ell}^{\prime}$, we have 

\begin{equation}
\begin{aligned}
\label{eq: hidden_conf_c_I_1_error}
I_1 
& = 
\Big \Vert 
\boldsymbol{V}_{t-1}^{-1} \sum_{\ell = 1}^{t-1} w_{\ell} \boldsymbol{\zeta}_{\ell}^{\top} \boldsymbol{r}_{a_{\ell}, \ell} \boldsymbol{r}_{a_{\ell}, \ell}
\Big \Vert_{\boldsymbol{V}_{t-1}} \\
& = 
\Big \Vert 
\sum_{\ell = 1}^{t-1} \frac{\omega}{\Vert \boldsymbol{r}_{a_{\ell}, \ell} \Vert_2^2} \boldsymbol{\zeta}_{\ell}^{\top} \boldsymbol{r}_{a_{\ell}, \ell} \boldsymbol{r}_{a_{\ell}, \ell}
\Big \Vert_{\boldsymbol{V}_{t-1}^{-1}} \\
& = 
\Big \Vert 
\sum_{\ell = 1}^{t-1} \zeta_{\ell}^{\prime} \boldsymbol{r}_{a_{\ell}, \ell}^{\prime}
\Big \Vert_{\boldsymbol{V}_{t-1}^{-1}} \\
& \underset{(a)}{\leq}
\sqrt{2 \omega R^2 \log \left(\frac{\det(\boldsymbol{V}_{t-1})^{1/2} \det(\boldsymbol{V}_{0})^{-1/2}}{\alpha} \right)} \\
& \qquad \qquad \text{(with probability at least }1-\alpha) \\
& \underset{(b)}{\leq}
R \sqrt{ \omega D \log \left(\frac{1 + \omega T/\lambda}{\alpha} \right)},
\end{aligned}
\end{equation}

where (a) holds by Lemma \ref{lemma: self_norm_bound}, and (b) holds by Lemma \ref{lemma: det_trace_ineq} with $\boldsymbol{V}_{t-1} = \lambda \boldsymbol{I} + \sum_{\ell   =1}^{t-1} \boldsymbol{r}_{a_{\ell}, \ell}^{\prime} {\boldsymbol{r}_{a_{\ell}, \ell}^{\prime}}^{\top}$ and $\Vert \boldsymbol{r}_{a_{\ell}, \ell}^{\prime} \Vert _2 = \sqrt{\omega}$.

\textbf{Bounding term $I_2$.}
Note $\overline{\boldsymbol{c}}^{\top} \boldsymbol{r}_t \in [0,1]$ and $\boldsymbol{r}_t \in [0,1]^D$, we have $\sum_{d \in [D]} \overline{\boldsymbol{c}}(d) \leq 1$ with $\overline{\boldsymbol{c}}(d) \geq 0, \forall d \in [D]$, which implies
\[
\Vert \overline{\boldsymbol{c}} \Vert_2
\leq 
\Vert \overline{\boldsymbol{c}} \Vert_1
\leq 
1.
\]

Thus for the regularization error term $I_2$, we have
\[
I_2 = \left \Vert 
\lambda \boldsymbol{V}_{t-1}^{-1} \overline{\boldsymbol{c}}
\right \Vert_{\boldsymbol{V}_{t-1}} = 
\lambda \left \Vert 
\overline{\boldsymbol{c}} \right \Vert_{\boldsymbol{V}_{t-1}^{-1}} 
\underset{(a)}{\leq}
\lambda \left \Vert 
\overline{\boldsymbol{c}} \right \Vert_{\frac{1}{\lambda}\boldsymbol{I}}
=
\sqrt{\lambda} \left \Vert 
\overline{\boldsymbol{c}} \right \Vert_{2}
\leq 
\sqrt{\lambda},
\]
where (a) holds since $\boldsymbol{V}_{t-1}^{-1} \preceq \boldsymbol{V}_{0}^{-1} = \frac{\boldsymbol{I}}{\lambda}$.

Combining term $I_1$ and term $I_2$ with Eq. \ref{eq:hidden_conf_c} completes the proof of Lemma \ref{lemma:c_estimator_conf_bd}.
\end{proof}

\subsection{Proof of Lemma \ref{lemma:g_estimator_upper_conf_bd} (Upper Confidence Bound for Expected Overall Reward)}
\label{sec:app_pf_lemma_g_estimator_upper_conf_bd}

\begin{proof}

For any $i \in [K], d \in [D], t>0$, by Hoeffding’s Inequality (Lemma~\ref{lemma: Hoeffding}), we have:

\begin{equation}
\begin{aligned}
\mathbb{P} \left( | \hat{\boldsymbol{r}}_{i,t}(d) - \boldsymbol{\mu}_{i}(d) | > \sqrt{\frac{ \log(t/\alpha)}{N_{i,t}}} \right)
& \leq
2\exp \left( \frac{-2 N_{i,t}^2 \log(t/\alpha) }{ N_{i,t} \sum_{\ell=1}^{N_{i,t}}(1-0)^2 } \right) \\
& =
2\exp \left( -2 \log(t/\alpha) \right) \\
& =
2 \left( \frac{\alpha}{t} \right)^2,
\end{aligned}
\end{equation}

Thus for any $i \in [K] $ and $t > 0$, with at least probability $1 - 2D \left( \frac{\alpha}{t} \right)^2$, we have 

\[
\boldsymbol{\hat{r}}_{i,t} - \sqrt{\frac{ \log(t/\alpha)}{N_{i,t}}} \boldsymbol{e} \preceq \boldsymbol{\mu}_{i} \preceq \boldsymbol{\hat{r}}_{i,t} + \sqrt{\frac{ \log(t/\alpha)}{N_{i,t}}} \boldsymbol{e}.
\]

And thus we can derive 

\begin{equation}
\begin{aligned}  
\label{eq: c_hidden_confidence_upper_bd2}
\boldsymbol{\hat{c}}_t^{\top} \boldsymbol{\hat{r}}_{i,t} + B_{i,t}^{r} + B_{i,t}^{c} -\boldsymbol{\overline{c}}^{\top} \boldsymbol{\mu}_{i} 
& \geq
\boldsymbol{\hat{c}}_t^{\top} \boldsymbol{\hat{r}}_{i,t} + B_{i,t}^{r} + B_{i,t}^{c} - \boldsymbol{\overline{c}}^{\top} \left( \boldsymbol{\hat{r}}_{i,t} + \sqrt{\frac{ \log(t/\alpha)}{N_{i,t}}} \boldsymbol{e} \right)  \\
& =
\boldsymbol{\hat{c}}_t^{\top} \left( \boldsymbol{\hat{r}}_{i,t} + \sqrt{\frac{ \log(t/\alpha)}{N_{i,t}}} \boldsymbol{e} \right) + B_{i,t}^{c} - \boldsymbol{\overline{c}}^{\top} \left( \boldsymbol{\hat{r}}_{i,t} + \sqrt{\frac{ \log(t/\alpha)}{N_{i,t}}} \boldsymbol{e} \right)  \\
& = 
(\boldsymbol{\hat{c}}_t - \boldsymbol{\overline{c}})^{\top}
\left( \boldsymbol{\hat{r}}_{i,t} + \sqrt{\frac{ \log(t/\alpha)}{N_{i,t}}} \boldsymbol{e} \right) + B_{i,t}^{c} \\
& \underset{(a)}{\geq}
- \Vert \boldsymbol{\hat{c}}_t - \boldsymbol{\overline{c}}_t \Vert_{\boldsymbol{V}_{t-1}} \cdot \left\Vert \boldsymbol{\hat{r}}_{i,t} + \sqrt{\frac{ \log(t/\alpha)}{N_{i,t}}} \boldsymbol{e} \right\Vert_{\boldsymbol{V}_{t-1}^{-1}} 
+
B_{i,t}^{c} \\
& \underset{(b)}{\geq}
- \beta_t \cdot \left\Vert \boldsymbol{\hat{r}}_{i,t} + \sqrt{\frac{ \log(t/\alpha)}{N_{i,t}}} \boldsymbol{e} \right\Vert_{\boldsymbol{V}_{t-1}^{-1}} 
+
B_{i,t}^{c} \\
& \qquad \qquad \qquad (\text{with probability at least } 1-\vartheta) \\
& = 0,
\end{aligned}
\end{equation}
where (a) holds by Cauchy-Schwarz inequality, (b) holds by Lemma \ref{lemma:c_estimator_conf_bd}. 
\end{proof}

\subsection{Uniform Confidence Bound for Estimators}
\label{sec: proof_proposition_uniform_confidence_bound}


\begin{proposition}
\label{proposition: uniform_confidence_bound}
Let $\alpha = \sqrt{12 \vartheta/(KD(D+3) \pi^2)}$, for all $t\in [1,T]$, with probability at least $1 - \vartheta$, we have following events hold simultaneously:
\[
\begin{aligned}
& \text{Event A:} \Bigg\{ |\boldsymbol{\mu}_{i}(d) - \boldsymbol{\hat{r}}_{i,t}(d)| \leq \sqrt{\frac{\log \left( \frac{t}{\alpha} \right)}{N_{i,t}}}, \forall i \in [K], \forall d \in [D] \Bigg\}, \\
& \text{Event B:} \Bigg\{ \mathbb{E} \left[ \sum_{\ell \in \mathcal{T}_{i,t-1}} w_{\ell} \boldsymbol{r}_{i,\ell} \boldsymbol{r}_{i,\ell}^{\top} \right](m,n)
-
\sum_{\ell \in \mathcal{T}_{i,t-1}} \left( w_{\ell} \boldsymbol{r}_{i,\ell} \boldsymbol{r}_{i,\ell}^{\top} \right)(m,n)
\leq 
\omega \sqrt{ N_{i,t} \log \left( \frac{t}{\alpha} \right)}, \\ 
& \qquad \qquad \quad \forall i \in [K], \forall m \in [D], \forall n \in [m,D] \Bigg\}, \\
\end{aligned}
\]
where $\mathcal{T}_{i,t}$ is the set of episodes that arm $i$ is pulled within $t$ steps.
\end{proposition}

\begin{proof}

\textbf{Step-1 (Confidence analysis of Event A):} 

For any $i \in [K], d \in [D], t \in (0,T]$, by Hoeffding’s Inequality (Lemma~\ref{lemma: Hoeffding}), we have the instantaneous failure probability of Event B can be bounded as:

\begin{equation}
\begin{aligned}
\mathbb{P} \left( | \hat{\boldsymbol{r}}_{i,t}(d) - \boldsymbol{\mu}_{i}(d) | > \sqrt{\frac{ \log(t/\alpha)}{N_{i,t}}} \right)
& \leq
2\exp \left( \frac{-2 N_{i,t}^2 \log(t/\alpha) }{ N_{i,t} \sum_{\ell=1}^{N_{i,t}}(1-0)^2 } \right) \\
& =
2\exp \left( -2 \log(t/\alpha) \right) \\
& =
2 \left( \frac{\alpha}{t} \right)^2,
\end{aligned}
\end{equation}

which yields the upper bound of $\mathbb{P} (B^\mathsf{c})$ by union bound as 

\begin{equation}
\begin{aligned}
\label{eq: conf_B_upbd}
\mathbb{P} (B^\mathsf{c}) 
& = 
\mathbb{P} \left( \exists \{i,d,t\}, | \hat{\boldsymbol{r}}_{i,t}(d) - \boldsymbol{\mu}_{i}(d) | > \sqrt{\frac{2 \log(t/\alpha)}{N_{i,t}}} \right) \\
& \leq
2 \sum_{t=1}^{\top}
\sum_{i=1}^{K}
\sum_{d=1}^{D}
\mathbb{P} \left( | \hat{\boldsymbol{r}}_{i,t}(d) - \boldsymbol{\mu}_{i}(d) | > \sqrt{\frac{ \log(t/\alpha)}{N_{i,t}}} \right) \\
& \leq 
2 \sum_{t=1}^{\top} \sum_{i=1}^{K} \sum_{d=1}^{D} \left( \frac{\alpha}{t} \right)^2
\underset{(Eq. \ref{eq: riemann_zeta})}{\leq }
\frac{KD \alpha^2 \pi^2}{3},
\end{aligned}
\end{equation}

\textbf{Step-2 (Confidence analysis of Event B):} 

The proof follows similar lines as above.
Note that for any $i \in [K], t \in (1,T], m \in [1, D], n \in [m, D]$, we have the instantaneous failure probability of Event C can be bounded as

\begin{small}
\[
\begin{aligned}
& \mathbb{P} \left( \mathbb{E} \left[ \sum_{\ell \in \mathcal{T}_{i,t-1}} w_{\ell} \boldsymbol{r}_{i,\ell} \boldsymbol{r}_{i,\ell}^{\top} \right](m,n)
-
\sum_{\ell \in \mathcal{T}_{i,t-1}} \left( w_{\ell} \boldsymbol{r}_{i,\ell} \boldsymbol{r}_{i,\ell}^{\top} \right)(m,n)
>
\omega \sqrt{ N_{i,t} \log \left( \frac{t}{\alpha} \right)} \right) \\
& \quad = 
\mathbb{P} \left( \mathbb{E} \left[ \frac{\omega}{\Vert \boldsymbol{r}_{i} \Vert_2^2} \boldsymbol{r}_{i} \boldsymbol{r}_{i}^{\top} \right](m,n)
-
\frac{1}{N_{i,t}} \sum_{\ell \in \mathcal{T}_{i,t-1}} \left( 
\frac{\omega}{\Vert \boldsymbol{r}_{i,\ell} \Vert_2^2} \boldsymbol{r}_{i,\ell} \boldsymbol{r}_{i,\ell}^{\top} \right)(m,n)
>
\omega \sqrt{\frac{ \log(t/\alpha)}{N_{i,t}}} \right) \\
& \quad \leq 
\exp \left( - \frac{ 2 \omega^2 N_{i,t}^2 \log(t/\alpha) }{ N_{i,t}^2 (\omega - 0)^2 } \right) 
=
(\frac{\alpha}{t})^2.
\qquad \text{\big(by Lemma \ref{lemma: Hoeffding} and $( \frac{\omega}{\Vert \boldsymbol{r}_{i,\ell} \Vert_2^2} \boldsymbol{r}_{i,\ell} \boldsymbol{r}_{i,\ell}^{\top})(m,n) \in [0,\omega]$ \big)}
\end{aligned}
\]
\end{small}

Using union bound, we have $\mathbb{P} (C^{\mathsf{c}})$ as
\begin{small}
\begin{equation}
\begin{aligned}
\label{eq: conf_C_upbd}
\mathbb{P} (C^\mathsf{c}) 
& = 
\mathbb{P} \left( \exists \{i,t,m,n\}, \mathbb{E} \left[ \sum_{\ell \in \mathcal{T}_{i,t-1}} w_{\ell} \boldsymbol{r}_{i,\ell} \boldsymbol{r}_{i,\ell}^{\top} \right](m,n)
-
\sum_{\ell \in \mathcal{T}_{i,t-1}} \left( w_{\ell} \boldsymbol{r}_{i,\ell} \boldsymbol{r}_{i,\ell}^{\top} \right)(m,n)
>
\omega \sqrt{ N_{i,t} \log \left( \frac{t}{\alpha} \right)} \right) \\
& \leq
\sum_{t=1}^{\top}
\sum_{i=1}^{K}
\sum_{m=1}^{D}
\sum_{n=m}^{D}
\mathbb{P} \left( \mathbb{E} \left[ \sum_{\ell \in \mathcal{T}_{i,t-1}} w_{\ell} \boldsymbol{r}_{i,\ell} \boldsymbol{r}_{i,\ell}^{\top} \right](m,n)
-
\sum_{\ell \in \mathcal{T}_{i,t-1}} \left( w_{\ell} \boldsymbol{r}_{i,\ell} \boldsymbol{r}_{i,\ell}^{\top} \right)(m,n)
>
\omega \sqrt{ N_{i,t} \log \left( \frac{t}{\alpha} \right)} \right) \\
& \leq 
\sum_{t=1}^{\top}
\sum_{i=1}^{K}
\sum_{m=1}^{D}
\sum_{n=m}^{D} \left( \frac{\alpha}{t} \right)^2
\underset{(Eq.~\ref{eq: riemann_zeta})}{\leq }
\frac{KD(D-1) \alpha^2 \pi^2}{12}.
\end{aligned}
\end{equation}
\end{small}

\textbf{Step-3 (Union confidence on three Events):} 

Combining Eq.~\ref{eq: conf_B_upbd} and Eq.~\ref{eq: conf_C_upbd}, and setting $\alpha = \sqrt{\frac{12 \vartheta}{KD(D+3) \pi^2}}$, by union bound, we can have the overall failure probability bound of three Events as

\[
\begin{aligned}
\mathbb{P} (A^\mathsf{c} \cup B^\mathsf{c}) 
\leq 
\mathbb{P} (A^\mathsf{c}) + \mathbb{P} (B^\mathsf{c})
= 
\left( \frac{KD(D-1) \pi^2}{12} + \frac{4KD \pi^2}{12} \right) \left(\frac{12 \vartheta}{KD(D+3) \pi^2} \right)
= 
\vartheta.
\end{aligned}
\]

This concludes the proof of Proposition~\ref{proposition: uniform_confidence_bound}.
\end{proof}

\subsection{Proof of Theorem \ref{theorem:up_bd_hiden} (Regret Analysis of Algorithm \ref{alg:PRUCB_HP})}
\label{sec: app_pf_thm_up_bd_hiden}

\begin{proof}

Based on the assumptions in Proposition \ref{proposition: uniform_confidence_bound}, we next show that when Events of A, B, C in Proposition \ref{proposition: uniform_confidence_bound} hold (detailed definitions of Events of A, B, C refer to Appendix \ref{sec: proof_proposition_uniform_confidence_bound}), the sub-linear regret of PUCB-HPM can be achieved. Please see the detailed proof steps below.






\textbf{Step-1 (Regret Analysis and Decomposition)}

Let $M$ be an arbitrary positive integer, we can express $R(T)$ in a truncated form with respect to $M$ as follows:
\begin{equation}
\begin{aligned}
\label{eq: trunc_regret_0}
R(T) = \sum_{t=1}^{\top} \text{regret}_{t} \leq M + \sum_{t=M+1}^{\top} \text{regret}_{t},
\end{aligned}
\end{equation}

where $\text{regret}_{t}$ denotes the instantaneous regret of PRUCB-HPM at step $t \in \left[ T \right]$, and the last inequality holds since the fact that the instantaneous regret is upper-bounded by 1 (by Assumption \ref{assmp: hpm_2}).

Next, we analyze the instantaneous regret over the truncated time horizon $[M+1, T]$.
Note that since event B holds, we have 

\begin{equation}
\begin{aligned}
\label{eq: a^*_confidence_hidden}
\boldsymbol{\mu}_{a^*} \preceq \boldsymbol{\hat{r}}_{a^*,t} + \sqrt{\frac{ \log(t/\alpha)}{N_{i,t}}} \boldsymbol{e}, 
\text{ and }
\boldsymbol{\mu}_{a_t} \succeq  \boldsymbol{\hat{r}}_{a_t, t} - \sqrt{\frac{ \log(t/\alpha)}{N_{i,t}}} \boldsymbol{e}.
\end{aligned}
\end{equation}

By the definition of regret and fact above, we can derive the upper-bound of expected instantaneous regret as follows:

\[
\begin{aligned}
\text{regret}_{t} 
& = \boldsymbol{\overline{c}}^{\top}  \boldsymbol{\mu}_{a^*} - \boldsymbol{\overline{c}}^{\top}  \boldsymbol{\mu}_{a_t} \\
& \underset{(a)}{\leq}
\boldsymbol{\hat{c}}^{\top} \boldsymbol{\hat{r}}_{a^*,t} + B_{a^*,t}^{r} + B_{a^*,t}^{c} - \boldsymbol{\overline{c}}^{\top}  \boldsymbol{\mu}_{a_t} \\
& \underset{(b)}{\leq}
\boldsymbol{\hat{c}}^{\top} \boldsymbol{\hat{r}}_{a^*,t} + B_{a^*,t}^{r} + B_{a^*,t}^{c} - \boldsymbol{\overline{c}}^{\top} \left( \boldsymbol{\hat{r}}_{a_t,t} + \sqrt{\frac{ \log(t/\alpha)}{N_{i,t}}} \boldsymbol{e} \right) + 2\Vert\boldsymbol{\overline{c}}\Vert_1 \sqrt{\frac{ \log(t/\alpha)}{N_{i,t}}} \\
& \underset{(c)}{\leq}
\boldsymbol{\hat{c}}^{\top} \boldsymbol{\hat{r}}_{a_t,t} + B_{a_t,t}^{r} + B_{a_t,t}^{c} - \boldsymbol{\overline{c}}^{\top} \left( \boldsymbol{\hat{r}}_{a_t,t} + \sqrt{\frac{ \log(t/\alpha)}{N_{i,t}}} \boldsymbol{e} \right) + 2\Vert\boldsymbol{\overline{c}}\Vert_1 \sqrt{\frac{ \log(t/\alpha)}{N_{i,t}}}  \\
& = 
(\boldsymbol{\hat{c}} - \boldsymbol{\overline{c}})^{\top} \left( \boldsymbol{\hat{r}}_{a_t,t} + \sqrt{\frac{ \log(t/\alpha)}{N_{i,t}}} \boldsymbol{e} \right) + B_{a_t,t}^{c} + 2\Vert\boldsymbol{\overline{c}}\Vert_1 \sqrt{\frac{ \log(t/\alpha)}{N_{i,t}}} \\
& \underset{(d)}{\leq}
\Vert \boldsymbol{\hat{c}}_t - \boldsymbol{\overline{c}}_t \Vert_{\boldsymbol{V}_{t-1}} \cdot \left \Vert \boldsymbol{\hat{r}}_{i,t} + \sqrt{\frac{ \log(t/\alpha)}{N_{i,t}}} \boldsymbol{e} \right \Vert_{\boldsymbol{V}_{t-1}^{-1}} + B_{a_t,t}^{c} + 2\Vert\boldsymbol{\overline{c}}\Vert_1 \sqrt{\frac{ \log(t/\alpha)}{N_{i,t}}} \\
& \underset{(e)}{\leq}
\underbrace{
\min \left( 2 B_{a_t,t}^{c}, 1 \right)
}_{\text{regret}_{t}^{\tilde{\boldsymbol{c}}}} 
+ 
\underbrace{
2 \Vert\boldsymbol{\overline{c}}\Vert_1 \sqrt{\frac{ \log(t/\alpha)}{N_{i,t}}} 
}_{\text{regret}_{t}^{\tilde{\boldsymbol{r}}}},
\end{aligned}
\]

where (a) followed by Lemma \ref{lemma:g_estimator_upper_conf_bd}, (b) followed by Eq.~\ref{eq: a^*_confidence_hidden}, (c) holds by the definition of optimization policy for arm selection, (d) holds by Cauchy-Schwarz inequality, (e) followed by Lemma \ref{lemma:c_estimator_conf_bd} with the definition of $B_{a_t,t}^{c}$, and the fact that the instantaneous regret is at most 1.
Interestingly, the derived instantaneous regret above can also be interpreted as the sum of two components: 
\begin{itemize}
  \item $\text{regret}_{t}^{\tilde{\boldsymbol{c}}}$: Regret caused by the imprecise estimation of preference $\boldsymbol{\overline{c}}$.
  \item $\text{regret}_{t}^{\tilde{\boldsymbol{r}}}$: Regret caused by the imprecise estimation of expected reward of arms.
\end{itemize}

Plugging above results back to Eq. \ref{eq: trunc_regret_0}, we have 

\begin{equation}
\begin{aligned}
\label{eq: trunc_regret}
R(T) 
& \leq
M
+
\sum_{t=M+1}^{\top} \text{regret}_t \\
& \leq
M 
+
\sum_{t=M+1}^{\top} \Big( \text{regret}_{t}^{\tilde{\boldsymbol{c}}} 
+
\text{regret}_{t}^{\tilde{\boldsymbol{r}}}
\Big)
\\ 
& \leq
M
+
\underbrace{
\sum_{t=M+1}^{\top} 
\min \left( 2 B_{a_t,t}^{c}, 1 \right)
}_{R_{M+1:T}^{\tilde{\boldsymbol{c}}}}+
\underbrace{
\sum_{t=M+1}^{\top} 2 \Vert\boldsymbol{\overline{c}}\Vert_1 \sqrt{\frac{ \log(t/\alpha)}{N_{i,t}}} 
}_{R_{M+1:T}^{\tilde{\boldsymbol{r}}}},
\end{aligned}
\end{equation}

which also yields two components of $R_{M+1:T}^{\tilde{\boldsymbol{c}}}$ and $R_{M+1:T}^{\tilde{\boldsymbol{r}}}$, denoting the accumulated truncated expected errors caused by the imprecise estimations of preference and reward respectively. 
Next we analyze two components of $R_{M+1:T}^{\tilde{\boldsymbol{c}}}$ and $R_{M+1:T}^{\tilde{\boldsymbol{r}}}$ separately.

\textbf{Step-2 (Upper-Bound over $R_{M+1:T}^{\tilde{\boldsymbol{c}}}$)}

Before analyzing $R_{M+1:T}^{\tilde{\boldsymbol{c}}}$, we first show three lemmas that will be utilized for proof:

\begin{lemma}
\label{lemma: hidden_sum_reg_c_expectation_bd}
Let $M = \left\lfloor \min \big \{ t^{\prime} \mid t  \sigma^2_{r \downarrow} + \lambda \geq 2D \omega \sqrt{Kt\log \frac{t}{\alpha} }, \forall t \geq t^{\prime} \big \} \right \rfloor$, follow the assumptions outlined in Proposition \ref{proposition: uniform_confidence_bound}, for any $t \geq M+1$, and any $\boldsymbol{\mu} \in \mathbb{R}^{D}$, we have 
\[
\boldsymbol{\mu}^{\top} \boldsymbol{V}_{t-1}^{-1} \boldsymbol{\mu}
\leq
2 \boldsymbol{\mu}^{\top} \mathbb{E}[\boldsymbol{V}_{t-1}]^{-1} \boldsymbol{\mu}.
\]
\end{lemma}
Please see Appendix \ref{sec: proof_lemma_hidden_sum_reg_c_expectation_bd} for the proof of Lemma \ref{lemma: hidden_sum_reg_c_expectation_bd}.

\begin{lemma}
\label{lemma: hidden_sum_reg_c_1}
Follow the assumptions outlined in Proposition \ref{proposition: uniform_confidence_bound}, then for any $M \geq 0$, we have
\[
\sum_{t=M+1}^{\top}
\min \left(
\Vert \boldsymbol{\mu}_{a_t} \Vert_{\mathbb{E}\left[\boldsymbol{V}_{t-1}\right]^{-1}}, 
\frac{1}{2\sqrt{2}}
\right)
\leq
\sqrt{\frac{D}{2 \log(9/8)} (T-M) \log \Big( 1 + \frac{\omega}{\lambda}(T-M) \Big) }.
\]
\end{lemma}
Please see Appendix \ref{sec: proof_lemma_upbd_term1_hidden} for the proof of Lemma \ref{lemma: hidden_sum_reg_c_1}.

\begin{lemma}
\label{lemma: hidden_sum_reg_c_2}
Let assumptions follow those outlined in Proposition \ref{proposition: uniform_confidence_bound}, then for any $M \geq 0$, we have
\[
\sum_{t=M+1}^{\top} \sqrt{\frac{ \log \left( \frac{t}{\alpha} \right)}{N_{a_t,t}}}
\leq
2 \sqrt{ K(T-M) \log \left( \frac{T}{\alpha} \right)}.
\]
\end{lemma}
Please see Appendix \ref{sec: proof_lemma_hidden_sum_reg_c_2} for the proof of Lemma \ref{lemma: hidden_sum_reg_c_2}.

Since we assume the Event B always holds, for any $t>1$ and $i \in [K]$, we have  
\[
\left \Vert \boldsymbol{\hat{r}}_{i,t} + \sqrt{\frac{ \log(t/\alpha)}{N_{i,t}}} \boldsymbol{e} - \boldsymbol{\mu}_{i} \right \Vert_2
\leq
\left \Vert 2 \sqrt{\frac{ \log(t/\alpha)}{N_{i,t}}} \boldsymbol{e} \right \Vert_2
=
2 \sqrt{D \frac{ \log(t/\alpha)}{N_{i,t}}} 
=
\rho_{i,t}.
\]

Consequently,
\[
\begin{aligned}
\left \Vert \boldsymbol{\hat{r}}_{i,t} + \sqrt{\log(t/\alpha)/N_{i,t}} \boldsymbol{e} \right \Vert_{\boldsymbol{V}_{t-1}^{-1}}
-
\Vert \boldsymbol{\mu}_{i} \Vert_{\boldsymbol{V}_{t-1}^{-1}}
& \leq
\left \Vert \boldsymbol{\hat{r}}_{i,t} + \sqrt{\log(t/\alpha)/N_{i,t}} \boldsymbol{e} 
-
\boldsymbol{\mu}_{i} \right \Vert_{\boldsymbol{V}_{t-1}^{-1}} \\
& \leq
\frac{1}{\sqrt{\lambda}} 
\left \Vert \boldsymbol{\hat{r}}_{i,t} + \sqrt{\log(t/\alpha)/N_{i,t}} \boldsymbol{e} 
-
\boldsymbol{\mu}_{i} \right \Vert_{2} \\
& \leq
\frac{1}{\sqrt{\lambda}} \rho_{i,t}, \\
\end{aligned}
\]
\begin{equation}
\label{eq:expect_r_gap_hidden}
\implies
\left \Vert \boldsymbol{\hat{r}}_{i,t} + \sqrt{\log(t/\alpha)/N_{i,t}} \boldsymbol{e} \right \Vert_{\boldsymbol{V}_{t-1}^{-1}}
\leq 
\Vert \boldsymbol{\mu}_{i} \Vert_{\boldsymbol{V}_{t-1}^{-1}}
+
\frac{1}{\sqrt{\lambda}} \rho_{i,t}.
\end{equation}

Define $M = \left\lfloor \min \big \{ t^{\prime} \mid t  \sigma^2_{r \downarrow} + \lambda \geq 2D \omega \sqrt{Kt\log \frac{t}{\alpha} }, \forall t \geq t^{\prime} \big \} \right \rfloor$. 
Please note that for $\sigma^2_{r \downarrow} >0$, we have $\lim_{t \rightarrow \infty}\frac{2D \omega \sqrt{Kt \log \frac{t}{\alpha} }}{\sigma^2_{r \downarrow}t} = \lim_{t \rightarrow \infty} C_1 \sqrt{ \frac{ \log (t) - C_2 }{t}} = 0$ since as $t$ increase, $\sqrt{\log t}$ grows much slowly compared to $\sqrt{t}$. Hence for sufficiently large $t^{\prime}$, the inequality $t \sigma^2_{r \downarrow} + \lambda \geq 2D\omega  \sqrt{Kt\log \frac{t}{\alpha} }, \forall t \geq t^{\prime}$ holds, which implies that such an $M$ does indeed exist.
By Lemma~\ref{lemma: hidden_sum_reg_c_expectation_bd}, for any $t \in \left[ M+1, T \right]$, we have 

\begin{equation}
\begin{aligned}
\text{regret}_{t}^{\tilde{\boldsymbol{c}}}
& =
\min \left(
2 B_{a_t,t}^{c}, 1 \right)  \\
& =
\min \left( 
2 
\beta_t 
\left \Vert \boldsymbol{\hat{r}}_{a_t,t} + \sqrt{ \log(t/\alpha)/N_{a_t,t}} \boldsymbol{e} \right \Vert_{\boldsymbol{V}_{t-1}^{-1}},
1 \right) \\
& \underset{(a)}{\leq}
\min \left( 
2 \beta_t 
\left(
\Vert \boldsymbol{\mu}_{a_t} \Vert_{\boldsymbol{V}_{t-1}^{-1}}
+
\frac{1}{\sqrt{\lambda}} \rho_{a_t,t}
\right), 
1 \right)\\
& \underset{(b)}{\leq}
\min \left( 
2 \beta_t 
\left(
\sqrt{2} \Vert \boldsymbol{\mu}_{a_t} \Vert_{\mathbb{E}\left[\boldsymbol{V}_{t-1}\right]^{-1}}
+
\frac{1}{\sqrt{\lambda}} \rho_{a_t,t}
\right), 
1 \right)\\
& \underset{(c)}{\leq}
\underbrace{
2\sqrt{2} \beta_T 
\min \left( 
\Vert \boldsymbol{\mu}_{a_t} \Vert_{\mathbb{E}\left[\boldsymbol{V}_{t-1}\right]^{-1}},
\frac{1}{2\sqrt{2}}  \right)
}_{\text{reg}_{t}^{(1)}}
+
\underbrace{
2\beta_T \frac{1}{\sqrt{\lambda}} 
\rho_{a_t,t}
}_{\text{reg}_{t}^{(2)}},
\end{aligned}
\end{equation}

where (a) follows by Eq. \ref{eq:expect_r_gap_hidden}, (b) follows by Lemma \ref{lemma: hidden_sum_reg_c_expectation_bd}, (c) holds due to $\beta_t \geq 1$ and increasing with $t$. By applying Lemma \ref{lemma: hidden_sum_reg_c_1} on $\text{reg}_{t}^{(1)}$ and Lemma \ref{lemma: hidden_sum_reg_c_2} on $\text{reg}_{t}^{(2)}$, we have  

\begin{equation}
\begin{aligned}
\label{eq: hidden_regret_c}
R_{M+1:T}^{\tilde{\boldsymbol{c}}}
& \leq 
\sum_{t=M+1}^{\top} \text{reg}_{t}^{(1)} +
\sum_{t=M+1}^{\top} \text{reg}_{t}^{(2)} \\
& \leq
2\sqrt{2} \beta_T 
\sum_{t=M+1}^{\top}
\min \left( 
\Vert \boldsymbol{\mu}_{a_t} \Vert_{\mathbb{E}\left[\boldsymbol{V}_{t-1}\right]^{-1}},
\frac{1}{2\sqrt{2}}  \right)
+
\frac{2\beta_T}{\sqrt{\lambda}}
\sum_{t=M+1}^{\top} \rho_{a_t,t} \\
& \leq
2\beta_T \sqrt{\frac{D}{\log(9/8)} (T-M) \log\Big( 1 + \frac{\omega}{\lambda} (T-M) \Big) } \\
& \qquad \qquad +
\frac{8\beta_T}{\sqrt{\lambda}}
\sqrt{ DK(T-M) \log \left( \frac{T}{\alpha} \right)}.
\end{aligned}
\end{equation}

\textbf{Step-3 (Upper-Bound over $R_{M+1:T}^{\tilde{\boldsymbol{r}}}$)}

For the truncated regret component $R_{M+1:T}^{\tilde{\boldsymbol{r}}}$ caused by imprecise estimation of reward, we have

\begin{equation}
\label{eq: hidden_regret_r}
\begin{aligned}
R_{M+1:T}^{\tilde{\boldsymbol{r}}}
= 
2 \sum_{t=M+1}^{\top} \Vert\boldsymbol{\overline{c}}\Vert_1 \sqrt{\frac{ \log(t/\alpha)}{N_{i,t}}}
\underset{(a)}{\leq}
2 \sum_{t=M+1}^{\top} \sqrt{\frac{ \log \left( \frac{t}{\alpha} \right)}{N_{a_t,t}}} 
\underset{(b)}{\leq}
4 \sqrt{ K (T-M) \log \left( \frac{T}{\alpha} \right)},
\end{aligned}
\end{equation}

where (a) holds by the fact that $\Vert \boldsymbol{\overline{c}}\Vert_1 \leq 1$, (b) follows by Lemma \ref{lemma: hidden_sum_reg_c_2}.

\textbf{Step-4 (Deriving final regret)}

Finally, we can derive the final regret $R(T)$.
Specifically, combining Eq. \ref{eq: hidden_regret_c} and Eq. \ref{eq: hidden_regret_r} gives
\[
\begin{aligned}
R(T) 
& \leq 
M +
R_{M+1:T}^{\tilde{\boldsymbol{c}}} +
R_{M+1:T}^{\tilde{\boldsymbol{r}}} \\
& \leq
2\beta_T \sqrt{\frac{D}{\log(9/8)} (T-M) \log \Big( 1 + \frac{\omega}{\lambda} (T-M) \Big) } \\
& \qquad +
\frac{8\beta_T}{\sqrt{\lambda}}
\sqrt{ DK(T-M) \log \left( \frac{T}{\alpha} \right)}
+
4 \sqrt{ K (T-M) \log \left( \frac{T}{\alpha} \right)}
+ M \\
& =
O\Bigg(
DR \sqrt{ \omega T \log^{2} \Big( \frac{1 + \omega T/\lambda}{\vartheta} \Big)}
+
\frac{DR}{\sqrt{\lambda}} \sqrt{ \omega DK T \log^2 \left(\frac{1 + \omega T/\lambda}{\vartheta} \right)} \\
& \qquad + 
\sqrt{ KT \log \left( \frac{T}{\vartheta} \right)}
\Bigg).
\end{aligned}
\]

Summing above result over user set $[N$ concludes the proof of Theorem \ref{theorem:up_bd_hiden}.
\end{proof}

\subsection{Proof of Lemma \ref{lemma: hidden_sum_reg_c_expectation_bd}}
\label{sec: proof_lemma_hidden_sum_reg_c_expectation_bd}

Before the proof, we state two lemmas that will be utilized in the derivation as follows.

\begin{lemma}[Eigenvalues of Sums of Hermitian Matrices~\citep{fulton2000eigenvalues}, Eq.(11)]
\label{lemma: eigen_rank}
Let $\boldsymbol{A}$ and $\boldsymbol{B}$ are $n \times n$ Hermitian matrices with eigenvalues $a_1 > a_2 > ... > a_n$ and $b_1 > b_2 > ... > b_n$. 
Let $\boldsymbol{C} = \boldsymbol{A} + \boldsymbol{B}$ and the eigenvalues of $\boldsymbol{C}$ are $c_1 > c_2 > ... > c_n$, then we have 
\[
c_{n-i-j} \geq a_{n-i} + b_{n-j}, \forall i,j \in [0, n-1].
\]
\end{lemma}

\begin{lemma}[Eigenvalue Bounds on Quadratic Forms]
\label{lemma: eigen_bound}
Assuming $A \in \mathbb{R}^{n \times n}$ is symmetric, then for any $\boldsymbol{x} \in \mathbb{R}^n$, the quadratic form is bounded by the product of the minimum and maximum eigenvalues of $A$ and the square of the norm of $\boldsymbol{x}$:
\[
\max \left( {\boldsymbol{\lambda_{A}}} \right) \Vert \boldsymbol{x} \Vert_2^2
\geq
\boldsymbol{x}^{\top} \boldsymbol{A} \boldsymbol{x} 
\geq
\min \left( {\boldsymbol{\lambda_{A}}} \right) \Vert \boldsymbol{x} \Vert_2^2,
\]
where $\boldsymbol{\lambda_{A}}$ is the eigenvalues of $\boldsymbol{A}$.
\end{lemma}

\begin{proof}
The quadratic form $\boldsymbol{x}^{\top} \boldsymbol{A} \boldsymbol{x}$ can be analyzed by decomposing $\boldsymbol{A}$ using its eigenvalues and eigenvectors. Since $\boldsymbol{A}$ is a symmetric matrix, we can write it as:
\[
\begin{aligned}
\boldsymbol{A} = \boldsymbol{Q} \boldsymbol{\Lambda} \boldsymbol{Q}^{\top},
\end{aligned}
\]
where $\boldsymbol{Q}$ is an orthogonal matrix whose columns are the eigenvectors of $\boldsymbol{A}$, and $\boldsymbol{\Lambda}$ is a diagonal matrix with the eigenvalues $\boldsymbol{\lambda_{A}}(i)$ on its diagonal.
By substituting the eigen-decomposition of $\boldsymbol{A}$, we have
\[
\boldsymbol{x}^{\top} \boldsymbol{A} \boldsymbol{x} = \boldsymbol{x}^{\top} \boldsymbol{Q} \boldsymbol{\Lambda} \boldsymbol{Q}^{\top} \boldsymbol{x}.
\]

Let $\boldsymbol{y} = \boldsymbol{Q}^{\top} \boldsymbol{x}$, then we have
\[
\boldsymbol{x}^{\top} \boldsymbol{A} \boldsymbol{x} = \boldsymbol{y}^{\top} \boldsymbol{\Lambda} \boldsymbol{y}
=
\sum_{i=1}^{n} \boldsymbol{\lambda_{A}}(i) \boldsymbol{y}(i)^2
\geq 
\min \left( {\boldsymbol{\lambda_{A}}} \right) \sum_{i=1}^{n} \boldsymbol{y}(i)^2
=
\min \left( {\boldsymbol{\lambda_{A}}} \right) \Vert \boldsymbol{y} \Vert_2^2
\underset{(a)}{=}
\min \left( {\boldsymbol{\lambda_{A}}} \right) \Vert \boldsymbol{x} \Vert_2^2.
\]

where (a) follows since $\Vert \boldsymbol{y} \Vert_2^2 = \Vert \boldsymbol{Q}^{\top} \boldsymbol{x} \Vert_2^2 = \Vert \boldsymbol{x} \Vert_2^2$ as 
$\boldsymbol{Q}$ is orthogonal and preserves the norm.
For $\max \left( {\boldsymbol{\lambda_{A}}} \right) \Vert \boldsymbol{x} \Vert_2^2,
\geq
\boldsymbol{x}^{\top} \boldsymbol{A} \boldsymbol{x}$, the proof follows similarly and is therefore omitted.
\end{proof}

\begin{proof}[Proof of Lemma \ref{lemma: hidden_sum_reg_c_expectation_bd}]

First, let's recall the definitions of $\mathbb{E}[\boldsymbol{V}_t]$ and $\boldsymbol{V}_t$ for $t \in [1,T]$:

\begin{equation}
\begin{aligned}
\label{eq: expected_upsilon_def}
\mathbb{E}[\boldsymbol{V}_t] & = \sum_{\ell=1}^{\top} \mathbb{E}[w_{\ell} \boldsymbol{r}_{a_{\ell}, \ell} \boldsymbol{r}_{a_{\ell}, \ell}^{\top}] + \lambda \boldsymbol{I} 
 =
\sum_{i=1}^{K} \mathbb{E} \left[ \sum_{\ell \in \mathcal{T}_{i,t}} w_{\ell} \boldsymbol{r}_{i,\ell} \boldsymbol{r}_{i,\ell}^{\top} \right] + \lambda \boldsymbol{I}
,
\end{aligned}
\end{equation}
\[
\begin{aligned}
\boldsymbol{V}_t 
=
\sum_{i=1}^{K} \sum_{\ell=1}^{N_{i,t+1}} w_{\ell} \boldsymbol{r}_{i, \ell} \boldsymbol{r}_{i, \ell}^{\top} + \lambda \boldsymbol{I} .
\end{aligned}
\]

where 
$\mathcal{T}_{i,t}$ denotes the set of episodes when arm $i$ was pulled.

Due to the assumption that event C holds, we have $\forall i \in [K]$,

\[
\begin{aligned}
\mathbb{E} \left[ \sum_{\ell \in \mathcal{T}_{i,t}} w_{\ell} \boldsymbol{r}_{i,\ell} \boldsymbol{r}_{i,\ell}^{\top} \right]
-
\omega \sqrt{ N_{i,t+1} \log \left( \frac{t}{\alpha} \right)} \boldsymbol{e e}^{\top}
\preceq
\sum_{\ell \in \mathcal{T}_{i,t}} \left( w_{\ell} \boldsymbol{r}_{i,\ell} \boldsymbol{r}_{i,\ell}^{\top} \right),
\end{aligned}
\]

implying that for any $\boldsymbol{x} \in \mathbb{R}^{D}$, we can get

\begin{equation}
\begin{aligned}
\label{eq: expected_with_beta}
\chi
& = 
\boldsymbol{x}^{\top} 
\boldsymbol{V}_{t}
\boldsymbol{x} \\
& \geq
\boldsymbol{x}^{\top} 
\left( 
\sum_{i=1}^{K} \left(
\sum_{\ell \in \mathcal{T}_{i,t}} \mathbb{E} \left[ w_{\ell} \boldsymbol{r}_{i,\ell} \boldsymbol{r}_{i,\ell}^{\top} \right]
-
\omega \sqrt{ N_{i,t+1} \log \left( \frac{t}{\alpha} \right)} \boldsymbol{e e}^{\top}
\right)
+ \lambda \boldsymbol{I}
\right)
\boldsymbol{x} \\
& = 
\boldsymbol{x}^{\top} 
\left(
\mathbb{E} \left[\boldsymbol{V}_{t} \right]
- 
\sum_{i=1}^{K}
\omega \sqrt{ N_{i,t+1} \log \left( \frac{t}{\alpha} \right)} \boldsymbol{e e}^{\top}
\right)
\boldsymbol{x}
.
\end{aligned}
\end{equation}

Next we make a preliminary analysis over the norm-distances of $\Vert \boldsymbol{x} \Vert^2_{ \sum_{i=1}^{K} \left( \sum_{\ell \in \mathcal{T}_{i,t}} \mathbb{E} \left[ w_{\ell} \boldsymbol{r}_{i,\ell} \boldsymbol{r}_{i,\ell}^{\top} \right]
\right) }$.

Let $w^{\downarrow} = \omega / \max_{i\in[K], t\in[T]}(\Vert \boldsymbol{r}_{i,t} \Vert_2^2)$, $w^{\uparrow} = \omega/\min_{i\in[K], t\in[T]}(\Vert \boldsymbol{r}_{i,t} \Vert_2^2)$, we have 

\[
w^{\downarrow}
\sum_{i=1}^{K} \left( \sum_{\ell \in \mathcal{T}_{i,t}} \mathbb{E} \left[\boldsymbol{r}_{i,\ell} \boldsymbol{r}_{i,\ell}^{\top} \right]
\right) 
\preceq
\sum_{i=1}^{K} \left( \sum_{\ell \in \mathcal{T}_{i,t}} \mathbb{E} \left[ w_{\ell} \boldsymbol{r}_{i,\ell} \boldsymbol{r}_{i,\ell}^{\top} \right]
\right) 
\preceq
w^{\uparrow}
\sum_{i=1}^{K} \left( \sum_{\ell \in \mathcal{T}_{i,t}} \mathbb{E} \left[\boldsymbol{r}_{i,\ell} \boldsymbol{r}_{i,\ell}^{\top} \right]
\right).
\]

By the continuity of norm-distance, there must be a weight $ \tilde{w} \in (w^{\downarrow}, w^{\uparrow})$ such that

\[
\boldsymbol{x}^{\top}
\left( 
\sum_{i=1}^{K} \Big( \sum_{\ell \in \mathcal{T}_{i,t}} \mathbb{E} \left[ w_{\ell} \boldsymbol{r}_{i,\ell} \boldsymbol{r}_{i,\ell}^{\top} \right]
\Big)
\right)
\boldsymbol{x}
=
\boldsymbol{x}^{\top}
\left( 
\tilde{w}
\sum_{i=1}^{K} \Big( \sum_{\ell \in \mathcal{T}_{i,t}} \mathbb{E} \left[ \boldsymbol{r}_{i,\ell} \boldsymbol{r}_{i,\ell}^{\top} \right]
\Big)
\right)
\boldsymbol{x}.
\]

Substituting it in Eq. \ref{eq: expected_with_beta}, we have 

\begin{equation}
\begin{aligned}
\label{eq: expected_with_beta_2}
& \boldsymbol{x}^{\top} 
\left(
\mathbb{E} \left[\boldsymbol{V}_{t} \right]
- 
\sum_{i=1}^{K}
\omega \sqrt{ N_{i,t+1} \log \left( \frac{t}{\alpha} \right)} \boldsymbol{e e}^{\top}
\right)
\boldsymbol{x} \\
& \qquad = 
\boldsymbol{x}^{\top} 
\left( 
\tilde{w}
\sum_{i=1}^{K}
\sum_{\ell \in \mathcal{T}_{i,t}} \mathbb{E} \left[ \boldsymbol{r}_{i,\ell} \boldsymbol{r}_{i,\ell}^{\top} \right]
+ \lambda \boldsymbol{I}
-
\sum_{i=1}^{K} \omega \sqrt{ N_{i,t+1} \log \left( \frac{t}{\alpha} \right)} \boldsymbol{e e}^{\top}
\right)
\boldsymbol{x} \\
& \qquad = 
\boldsymbol{x}^{\top} 
\left( 
\tilde{w}
\sum_{i=1}^{K} N_{i,t+1} \left(
 \boldsymbol{\mu}_{i} \boldsymbol{\mu}_{i}^{\top} + \Sigma_{\boldsymbol{r,i}} \right)
+ \lambda \boldsymbol{I}
-
\sum_{i=1}^{K} \omega \sqrt{ N_{i,t+1} \log \left( \frac{t}{\alpha} \right)} \boldsymbol{e e}^{\top}
\right)
\boldsymbol{x}, \\
\end{aligned}
\end{equation}

where the final line follows by the definition of outer product expectation, and
\begin{small}
$\Sigma_{\boldsymbol{r},i} = 
\begin{bmatrix}
\sigma_{r,i,1}^2 & & 0\\
  & \ddots &\\
0 & & \sigma_{r,i,D}^2
\end{bmatrix}_{d \times d}$ 
\end{small}
denotes the covariance matrix of reward.

Similarly, let $p = \argminA_{i \in [K]} \boldsymbol{x}^{\top} \boldsymbol{\mu}_{i} \boldsymbol{\mu}_{i}^{\top} \boldsymbol{x}$ and $
q = \argmaxA_{j \in [K]} \boldsymbol{x}^{\top} \boldsymbol{\mu}_{j} \boldsymbol{\mu}_{j}^{\top} \boldsymbol{x}$, we can obtain

\[
\boldsymbol{x}^{\top} \left( t \boldsymbol{\mu}_{p} \boldsymbol{\mu}_{p}^{\top}
\right) \boldsymbol{x} 
\leq
\boldsymbol{x}^{\top} \left( \sum_{i=1}^{K} N_{i,t+1} \boldsymbol{\mu}_{i} \boldsymbol{\mu}_{i}^{\top}
\right) \boldsymbol{x} 
\leq 
\boldsymbol{x}^{\top} \left( t \boldsymbol{\mu}_{q} \boldsymbol{\mu}_{q}^{\top}
\right) \boldsymbol{x}. 
\]

By the continuity of norm-distance, result above implies that $\exists a \in [0,1]$, such that
\begin{equation}
\begin{aligned}
\label{eq: eq_mu}
\boldsymbol{x}^{\top} \left( \sum_{i=1}^{K} N_{i,t+1} \boldsymbol{\mu}_{i} \boldsymbol{\mu}_{i}^{\top}
\right) \boldsymbol{x}
=
\boldsymbol{x}^{\top} \left( t \boldsymbol{\tilde{\mu}} \boldsymbol{\tilde{\mu}}^{\top}
\right) \boldsymbol{x}, \\
\end{aligned}
\end{equation}
where $\boldsymbol{\tilde{\mu}} = a \boldsymbol{\mu}_{p} + (1-a) \boldsymbol{\mu}_{q}$.

Similarly, for $\Vert\boldsymbol{x} \Vert^2_{ \sum_{i=1}^{K} \left( N_{i,t+1} \Sigma_{\boldsymbol{r,i}} \right) }$, since the covariance matrices $\Sigma_{\boldsymbol{r,i}}, \forall i \in [K]$ are diagonal, by Lemma \ref{lemma: eigen_bound}, we have 
\begin{small}
\[
\xi_{\min} \left( \sum_{i=1}^{K} N_{i,t+1} \Sigma_{\boldsymbol{r,i}} \right)
\Vert \boldsymbol{x} \Vert_2^2
\leq
\boldsymbol{x}^{\top}
\left( \sum_{i=1}^{K} N_{i,t+1} \Sigma_{\boldsymbol{r,i}} \right)
\boldsymbol{x}
\leq
\xi_{\max} \left( \sum_{i=1}^{K} N_{i,t+1} \Sigma_{\boldsymbol{r,i}} \right)
\Vert \boldsymbol{x} \Vert_2^2,
\]
\end{small}

where 
$\xi_{\min} ( \sum_{i=1}^{K} N_{i,t} \Sigma_{\boldsymbol{r,i}})$
denotes the minimum eigenvalue of matrix 
$\sum_{i=1}^{K} N_{i,t+1} \Sigma_{\boldsymbol{r,i}} $, while 
$\xi_{\max} ( \sum_{i=1}^{K} N_{i,t+1} \Sigma_{\boldsymbol{r,i}} )$
denotes the corresponding maximum one.
We will also use $\xi (\cdot)$ to denote the eigenvalue calculator for a matrix in the following part.
By the continuity of nor-distance, result above implies that there exist a constant 
$\tilde{\xi}_t \in \big[\xi_{\min} ( \sum_{i=1}^{K} N_{i,t+1} \Sigma_{\boldsymbol{r,i}} ), \xi_{\max} ( \sum_{i=1}^{K} N_{i,t+1} \Sigma_{\boldsymbol{r,i}} ) \big]$, such that
\[
\begin{aligned}
\xi_{\min} \left( \sum_{i=1}^{K} N_{i,t+1} \Sigma_{\boldsymbol{r,i}} \right) 
\Vert\boldsymbol{x} \Vert_2^2 
\leq
\tilde{\xi}_t \Vert \boldsymbol{x} \Vert_2^2
=
\boldsymbol{x} ^{\top}
\left( \sum_{i=1}^{K} N_{i,t+1} \Sigma_{\boldsymbol{r,i}} \right)
\boldsymbol{x}
\leq
\xi_{\max} \left( \sum_{i=1}^{K} N_{i,t+1} \Sigma_{\boldsymbol{r,i}} \right)
\Vert \boldsymbol{x} \Vert_2^2,
\end{aligned}
\]

Note that $\sum_{i=1}^{K} N_{i,t+1} \Sigma_{\boldsymbol{r,i}}$ is diagonal, we have $\xi_{\min} ( \sum_{i=1}^{K} N_{i,t+1} \Sigma_{\boldsymbol{r,i}} ) = \min_{d \in [D]} \sum_{i=1}^{K} N_{i,t+1} \sigma_{r,i,d} \geq t \sigma^2_{r \downarrow}$, and similarly, $\xi_{\max} ( \sum_{i=1}^{K} N_{i,t+1} \Sigma_{\boldsymbol{r,i}} ) \leq t \sigma^2_{r \uparrow}$.
Define $\tilde{\sigma}^2_{r,t} = \frac{\tilde{\xi}_t}{t}$,
we have $\tilde{\sigma}^2_{r,t} \in [\sigma^2_{r \downarrow},\sigma^2_{r \uparrow}]$ satisfying 
\begin{equation}
\label{eq: eq_sigma}
t \tilde{\sigma}^2_{r,t} \Vert \boldsymbol{x} \Vert_2^2
=
\boldsymbol{x}^{\top}
\left( \sum_{i=1}^{K} N_{i,t+1} \Sigma_{\boldsymbol{r,i}} \right)
\boldsymbol{x}.
\end{equation}

By plugging above result back to Eq. \ref{eq: expected_with_beta_2}, we have
\begin{equation}
\begin{aligned}
\label{eq: expected_with_beta2}
\chi
& \geq
\boldsymbol{x}^{\top} 
\left( 
\underbrace{
\tilde{w}
\sum_{i=1}^{K} N_{i,t+1} \left(
 \boldsymbol{\mu}_{i} \boldsymbol{\mu}_{i}^{\top} + \Sigma_{\boldsymbol{r,i}} \right)
+ \lambda \boldsymbol{I}
}_{\mathbb{E}[\boldsymbol{V}_t] }
-
\sum_{i=1}^{K} \omega \sqrt{ N_{i,t+1} \log \left( \frac{t}{\alpha} \right)} \boldsymbol{e e}^{\top}
\right)
\boldsymbol{x} \\
& \underset{(a)}{=}
\boldsymbol{x}^{\top} 
\left( 
\tilde{w} t \boldsymbol{\tilde{\mu}} \boldsymbol{\tilde{\mu}}^{\top} + \left( \tilde{w} t \tilde{\sigma}^2_{r,t} + \lambda \right) \boldsymbol{I}
-
\sum_{i=1}^{K} \omega \sqrt{ N_{i,t+1} \log \left( \frac{t}{\alpha} \right)} \boldsymbol{e e}^{\top}
\right)
\boldsymbol{x} \\
& \underset{(b)}{\geq}
\boldsymbol{x}^{\top} 
\bigg (
\underbrace{
\tilde{w} t \boldsymbol{\tilde{\mu}} \boldsymbol{\tilde{\mu}}^{\top} 
}_{\boldsymbol{A}_t}
+ 
\underbrace{
\left( \tilde{w} t \tilde{\sigma}^2_{r,t} + \lambda \right) \boldsymbol{I}
}_{\boldsymbol{B}_t}
-
\underbrace{
\omega \sqrt{ K t \log \left( \frac{t}{\alpha} \right)} \boldsymbol{e e}^{\top}
}_{\boldsymbol{C}_t}
\bigg )
\boldsymbol{x}.
\end{aligned}
\end{equation}

where (a) holds by Eq. \ref{eq: eq_mu} and Eq. \ref{eq: eq_sigma}, (b) holds since the squared root term is maximized when $N_{i,t+1} = t/K, \forall i \in [K]$.
Note that $\boldsymbol{B}_t$ is diagonal matrix, and $-\boldsymbol{C}_t$ is rank-1 matrix yields one eigenvalue of $- \omega \sqrt{ K t \log \left( \frac{t}{\alpha} \right)} \Vert \boldsymbol{e} \Vert_2^2 = -D \omega \sqrt{ K t \log \left( \frac{t}{\alpha} \right)}$ and $D-1$ eigenvalues of 0, we have 

\[
\xi_{\min} (\boldsymbol{B}_t - \boldsymbol{C}_t) = \tilde{w} t \tilde{\sigma}^2_{r,t} + \lambda -D \omega \sqrt{ K t \log \left( \frac{t}{\alpha} \right)}.
\]

Since $t \geq M$, we can trivially derive $\tilde{w} t\tilde{\sigma}^2_{r,t} + \lambda \geq t \sigma^2_{r \downarrow} + \lambda \geq 2D \omega \sqrt{Kt\log \frac{t}{\alpha}}$ due to $\tilde{w}\geq 1$ and $\tilde{\sigma}^2 \geq \sigma^2_{r \downarrow}$, implying that the minimum eigenvalue $\xi_{\min} (\boldsymbol{B}_t - \boldsymbol{C}_t) \geq 0$ and
the matrix $\boldsymbol{B}_t - \boldsymbol{C}_t $ is a positive semi-definite matrix, and thus $\boldsymbol{A}_t + \boldsymbol{B}_t - \boldsymbol{C}_t$ is positive-definite. Also note that $\boldsymbol{A}_t + \boldsymbol{B}_t - \boldsymbol{C}_t$ is symmetric, by Lemma~\ref{lemma: eigen_bound}, we can derive that 

\begin{equation}
\begin{aligned}
\label{eq: ball_width_bd}
& \xi_{\min} \left( \boldsymbol{A}_t + \boldsymbol{B}_t - \boldsymbol{C}_t \right) \Vert \boldsymbol{x} \Vert_2^2
\leq
\left( \boldsymbol{x} \right)^{\top} \left( \boldsymbol{A}_t + \boldsymbol{B}_t - \boldsymbol{C}_t \right) \left( \boldsymbol{x} \right)
\leq
\chi \\
& \underset{(a)}{\implies}
\Vert \boldsymbol{x} \Vert_2^2 
\leq
\frac{\chi}{\xi_{\min} \left( \boldsymbol{A}_t + \boldsymbol{B}_t - \boldsymbol{C}_t \right)},
\end{aligned}
\end{equation}

where $\xi_{\min}\left( \boldsymbol{A}_t + \boldsymbol{B}_t - \boldsymbol{C}_t \right)$ is the minimum eigenvalue of $\boldsymbol{A}_t + \boldsymbol{B}_t - \boldsymbol{C}_t$, and the implication (a) holds since $\xi_{\min}\left( \boldsymbol{A}_t + \boldsymbol{B}_t - \boldsymbol{C}_t \right) > 0$ due to the positive-definite of $\boldsymbol{A}_t + \boldsymbol{B}_t - \boldsymbol{C}_t$.

Note that $\boldsymbol{A}_t$ and $-\boldsymbol{C}_t$ are rank-1 matrices and $\boldsymbol{B}_t$ is diagonal matrix, we can trivially derive that:

\begin{itemize}
\item 
$\boldsymbol{A}_t + \boldsymbol{B}_t$ has one eigenvalue of $\tilde{w} t(\Vert \boldsymbol{\tilde{\mu}} \Vert_2^2 + \tilde{\sigma}^2_{r,t}) + \lambda$ and $D-1$ eigenvalues of $ \tilde{w} t \tilde{\sigma}^2_{r,t} + \lambda$. 
\item $-\boldsymbol{C}_t$ has one eigenvalue of $- \omega \sqrt{ Kt\log \left( \frac{t}{\alpha} \right)} \Vert \boldsymbol{e} \Vert_2^2 = -D \omega \sqrt{ K t \log \left( \frac{t}{\alpha} \right)}$ and $D-1$ eigenvalues of 0.
\end{itemize}

Since $\boldsymbol{A}_t+\boldsymbol{B}_t$ and $-\boldsymbol{C}_t$ are both symmetric, by applying Lemma~\ref{lemma: eigen_rank}, we have 

\[
\begin{aligned}
\xi_{\min} \left( \boldsymbol{A}_t + \boldsymbol{B}_t - \boldsymbol{C}_t \right) 
& \geq
\xi_{\min} \left( \boldsymbol{A}_t + \boldsymbol{B}_t \right) 
+
\xi_{\min} \left( -\boldsymbol{C}_t \right) \\
& = 
\tilde{w} t \tilde{\sigma}^2_{r,t} + \lambda - D \omega \sqrt{ K t \log \left( \frac{t}{\alpha} \right)}
\end{aligned}
\]

Plugging above result back into Eq.~\ref{eq: ball_width_bd}, we have

\begin{equation}
\begin{aligned}
\label{eq: ball_width_bd2}
\Vert \boldsymbol{x} \Vert_2^2 
\leq
\frac{\chi}{\xi_{\min} \left( \boldsymbol{A}_t + \boldsymbol{B}_t - \boldsymbol{C}_t \right)}
\leq
\frac{\chi}{\tilde{w} t \tilde{\sigma}^2_{r,t} + \lambda - D \omega \sqrt{ K t \log \left( \frac{t}{\alpha} \right)}}.
\end{aligned}
\end{equation}

Again, since $t \geq M$ holds, the denominator of the final term is strictly positive.
Combining above result with Eq.~\ref{eq: expected_with_beta2} and rearranging the terms, for $t \geq M$, we can obtain

\begin{equation}
\begin{aligned}
\boldsymbol{x}^{\top} 
\mathbb{E} [\boldsymbol{V}_t]
\boldsymbol{x}
& =
\boldsymbol{x} ^{\top} 
\big ( \boldsymbol{A}_t + \boldsymbol{B}_t \big )
\boldsymbol{x} \\
& \leq 
\chi + \boldsymbol{x}^{\top} 
\boldsymbol{C}_t
\boldsymbol{x} \\
& \underset{(a)}{\leq}
\chi + \xi_{\max} \left(\boldsymbol{C}_t \right) \Vert \boldsymbol{x} \Vert_2^2 \\
& \underset{(b)}{\leq}
\chi + \frac{\chi D \omega \sqrt{ Kt \log \left( \frac{t}{\alpha} \right)}}{ \tilde{w} t \tilde{\sigma}^2_{r,t} + \lambda - D \omega \sqrt{ Kt \log \left( \frac{t}{\alpha} \right)}} \\
& = 
\chi + \frac{\chi}{\frac{ \tilde{w} t \tilde{\sigma}^2_{r,t} + \lambda}{ D \omega \sqrt{ Kt \log \left( \frac{t}{\alpha} \right)}} - 1} \\
& \underset{(c)}{\leq}
2 \chi 
= 
\boldsymbol{x}^{\top} \big( 2\boldsymbol{V}_t \big) \boldsymbol{x},
\end{aligned}
\end{equation}

where (a) follows from Lemma~\ref{lemma: eigen_bound}, (b) holds since Eq.~\ref{eq: ball_width_bd} and $\xi_{\max} \left(\boldsymbol{C}_t \right) = -\xi_{\min} \left(-\boldsymbol{C}_t \right) = D \omega \sqrt{ K t \log \left( \frac{t}{\alpha} \right)}$, (c) holds since $\tilde{w} t \tilde{\sigma}^2_{r,t} + \lambda \geq 2 D \omega \sqrt{ K t \log \left( \frac{t}{\alpha} \right)} $ for $t \geq M$.

Since for $t \geq M$, the above result $\boldsymbol{x}^{\top} 
\mathbb{E} [\boldsymbol{V}_t]
\boldsymbol{x} \leq \boldsymbol{x}^{\top} \big( 2\boldsymbol{V}_t \big) \boldsymbol{x}$ holds for all $\boldsymbol{x} \in \mathbb{R}^{D}$, we have $2\boldsymbol{V}_t - \mathbb{E} [\boldsymbol{V}_t]$ is positive definite, and thus we can trivially derive that for for $t \geq M$,
\[
\boldsymbol{x}^{\top} 
\mathbb{E} [\boldsymbol{V}_t]^{-1}
\boldsymbol{x} 
\geq
\boldsymbol{x}^{\top} \big( 2\boldsymbol{V}_t \big)^{-1} \boldsymbol{x} 
= 
\frac{1}{2} \boldsymbol{x}^{\top} \boldsymbol{V}_t^{-1} \boldsymbol{x}.
\]

Therefore, we complete the proof of Lemma \ref{lemma: hidden_sum_reg_c_expectation_bd}.
\end{proof}

\subsection{Proof of Lemma \ref{lemma: hidden_sum_reg_c_2}}
\label{sec: proof_lemma_hidden_sum_reg_c_2}

\begin{proof}[Proof of Lemma \ref{lemma: hidden_sum_reg_c_2}]
\begin{equation}
\label{eq: trunc_regret_r}
\begin{aligned}
 \sum_{t=M+1}^{\top} \sqrt{\frac{ \log \left( \frac{t}{\alpha} \right)}{N_{a_t,t}}}
& \underset{(a)}{\leq }
\sqrt{ \log \left( \frac{T}{\alpha} \right)} \sum_{i=1}^{K} \sum_{n=N_{i,M+1}}^{N_{i,T}} \sqrt{\frac{1}{n}}\\
& \underset{(b)}{\leq }
\sqrt{ \log \left( \frac{T}{\alpha} \right)} \sum_{i=1}^{K} \sum_{n=N_{i,1}}^{N_{i,T-M}} \sqrt{\frac{1}{n}}\\
& \underset{(c)}{\leq }
\sqrt{ \log \left( \frac{T}{\alpha} \right)} \sum_{i=1}^{K} \sum_{n=1}^{\frac{T-M}{K}} \sqrt{\frac{1}{n}}\\
& \underset{(d)}{\leq }
\sqrt{ \log \left( \frac{T}{\alpha} \right)} \sum_{i=1}^{K} 2 \sqrt{\frac{T-M}{K}}\\
& =
2 \sqrt{ K \log \left( \frac{T}{\alpha} \right) (T-M)}.
\end{aligned}
\end{equation}

Specifically, in step (a), we breakdown the totally truncated horizon by the episodes that each individual arm $i \in [K]$ was pulled, and replace $t$ with upper-bound $T$ in the original numerator.
Step (b) trivially holds since $\frac{1}{N_{i,t+M}} \leq \frac{1}{N_{i,t}}$ is strictly true for all $i \in [K]$. Step (c) follows from the fact that the entire sum is maximized when all arms are pulled an equal number of times. (d) holds since the fact that $ 2\sqrt{n}-2 \leq \sum_{x=1}^{n} \frac{1}{\sqrt{x}} \leq 2\sqrt{n} $.
\end{proof}


\subsection{Proof of Lemma \ref{lemma: hidden_sum_reg_c_1}}
\label{sec: proof_lemma_upbd_term1_hidden}

To derive the upper-bound of term 
$\sum_{t=M+1}^{\top}
\min \left(
\Vert \boldsymbol{\mu}_{a_t} \Vert_{\mathbb{E}\left[\boldsymbol{V}_{t-1}\right]^{-1}}, 
\frac{1}{2\sqrt{2}}
\right)$, 
we follow the similar techniques for analyzing the sum of instantaneous regret in OFUL~\citep{abbasi2011improved}. Specifically, we first show that the sum of squared terms $\min \left(
\Vert \boldsymbol{\mu}_{a_t} \Vert_{\mathbb{E}\left[\boldsymbol{V}_{t-1}\right]^{-1}}, 
\frac{1}{2\sqrt{2}}
\right)^2$ is optimal up to $\mathcal{O}(\log T)$, and then extend the result to the sum of $\min \left(
\Vert \boldsymbol{\mu}_{a_t} \Vert_{\mathbb{E}\left[\boldsymbol{V}_{t-1}\right]^{-1}}, 
\frac{1}{2\sqrt{2}}
\right)$ using Cauchy-Schwarz inequality.

We begin with stating the following lemmas for proof. 
\begin{lemma}
\label{lemma: iter_E_upsilon}
For any action sequence of ${a_1},...,{a_T}$ and any $M \in [0,1]$, we have
\[
\emph{det} \left( \mathbb{E}[\boldsymbol{V}_{T}] \right) 
\geq
\emph{det} \left( \mathbb{E}[\boldsymbol{V}_{M}] \right) 
\prod_{t=M+1}^{\top} 
\left(
1 + \frac{\emph{det} \left( \Sigma_{r,a_t} \right)}{\emph{det} \left( \mathbb{E}[\boldsymbol{V}_{t-1}] \right) }
+ 
\boldsymbol{\mu}_{a_t}^{\top} \mathbb{E}[\boldsymbol{V}_{t-1}]^{-1} \boldsymbol{\mu}_{a_t}
\right).
\]
\end{lemma}
Please see Appendix~\ref{sec: proof_lemma_iter_E_upsilon} for the detailed proof of Lemma~\ref{lemma: iter_E_upsilon}.

\begin{lemma}
\label{lemma: log_E_upsilon}
For any action sequence of ${a_1},...,{a_T}$, any weight $w_t = \frac{\omega}{\Vert \boldsymbol{r}_{a_t,t} \Vert_2^2}$, then for any $M \in [1,T]$, we have
\[
\log \left( \frac{ \emph{det} \left(\mathbb{E}[\boldsymbol{V}_{T}] \right) }{ \emph{det} \left(\mathbb{E}[\boldsymbol{V}_{M}] \right)} \right)
\leq 
D \log \left( 1 + \frac{\omega}{D \lambda}(T-M) \right).
\]
\end{lemma}

Please see Appendix~\ref{sec: proof_lemma_log_E_upsilon} for the detailed proof of Lemma~\ref{lemma: log_E_upsilon}.

\begin{proof}[Proof of Lemma\ref{lemma: hidden_sum_reg_c_1}]

\textbf{Step-1:}
We first show that the sum of squared terms is optimal up to $\mathcal{O}(\log(T-M))$. Specifically,

\begin{equation}
\begin{aligned}
\label{eq: term1_squared_hidden}
& \sum_{t=M+1}^{\top} \min \left(
\Vert \boldsymbol{\mu}_{a_t} \Vert_{\mathbb{E}\left[\boldsymbol{V}_{t-1}\right]^{-1}}, 
\frac{1}{2\sqrt{2}}
\right)^2 \\
& \qquad \qquad =
\sum_{t=M+1}^{\top} \min \left( \boldsymbol{\mu}_{a_t}^{\top} \mathbb{E}[\boldsymbol{V}_t]^{-1} \boldsymbol{\mu}_{a_t}, \frac{1}{8} \right) \\
& \qquad \qquad \underset{(a)}{\leq} 
\sum_{t=M+1}^{\top} \frac{1}{8 \log(9/8)} \log \left( 1 + \min \big( \boldsymbol{\mu}_{a_t}^{\top} \mathbb{E}[\boldsymbol{V}_t]^{-1} \boldsymbol{\mu}_{a_t}, \frac{1}{8} \big) \right) \\
& \qquad \qquad \leq 
\sum_{t=M+1}^{\top} \frac{1}{8 \log(9/8)} \log \left( 1 + \boldsymbol{\mu}_{a_t}^{\top} \mathbb{E}[\boldsymbol{V}_t]^{-1} \boldsymbol{\mu}_{a_t} \right) \\
\end{aligned}
\end{equation}

where (b) holds since the fact that $\log(1+x) \geq 8 \log \left( \frac{9}{8} \right)x$ for $x \leq \frac{1}{8}$.

On the other hand, Lemma~\ref{lemma: iter_E_upsilon} implies that 

\begin{equation}
\begin{aligned}
\log \left( \frac{ \text{det} \left(\mathbb{E}[\boldsymbol{V}_{T}] \right) }{ \text{det} \left(\mathbb{E}[\boldsymbol{V}_{M}] \right)} \right) 
\geq 
\sum_{t=M+1}^{\top} 
\log \left(
1 + \frac{\text{det} \left( \Sigma_{r,a_t} \right)}{\text{det} \left( \mathbb{E}[\boldsymbol{V}_{t-1}] \right) }
+ 
\boldsymbol{\mu}_{a_t}^{\top} \mathbb{E}[\boldsymbol{V}_{t-1}]^{-1} \boldsymbol{\mu}_{a_t}
\right).
\end{aligned}
\end{equation}

Additionally, since $\text{det} \left( \Sigma_{r,a_t} / \mathbb{E}[\boldsymbol{V}_{t-1}] \right) > 0$ and $\Vert \boldsymbol{\mu}_{i} \Vert_2^2 \leq \Vert \boldsymbol{\mu}_{i} \Vert_1^2 \leq D, \forall i \in [K]$, by Lemma~\ref{lemma: log_E_upsilon}, we have 

\begin{equation}
\begin{aligned}
\label{eq: term1_log_sum_hidden}
D \log \left( 1 + \frac{\omega}{\lambda}(T-M) \right)
\geq 
\log \left( \frac{ \text{det} \left(\mathbb{E}[\boldsymbol{V}_{T}] \right) }{ \text{det} \left(\mathbb{E}[\boldsymbol{V}_{M}] \right)} \right) 
\geq
\sum_{t=M+1}^{\top} 
\log \left(
1 + 
\boldsymbol{\mu}_{a_t}^{\top} \mathbb{E}[\boldsymbol{V}_t]^{-1} \boldsymbol{\mu}_{a_t}
\right).
\end{aligned}
\end{equation}

Plugging the above result back into Eq.~\ref{eq: term1_squared_hidden}, we can derive a bound up to $\mathcal{O}(\log(T-M))$ on the sum of squared instantaneous regrets as:

\begin{equation}
\begin{aligned}
\label{eq: term1_squared_upbd_hidden}
\sum_{t=M+1}^{\top} \min \left( \sqrt{ \boldsymbol{\mu}_{a_t}^{\top} \mathbb{E}[\boldsymbol{V}_t]^{-1} \boldsymbol{\mu}_{a_t}}, \frac{1}{2\sqrt{2}} \right)^2 
& \leq 
\frac{D}{8 \log(9/8)} \log \left( 1 + \frac{\omega}{ \lambda}(T-M) \right). \\
\end{aligned}
\end{equation}

\textbf{Step-2:}
Given the upper-bound on the sum of squared instantaneous regrets , we next extend it to the sum of instantaneous regrets by using Cauchy-Schwarz inequality.
Specifically, 

\begin{equation}
\begin{aligned}
\sum_{t=M+1}^{\top} \min \Big( \sqrt{ \boldsymbol{\mu}_{a_t}^{\top} \mathbb{E}[\boldsymbol{V}_t]^{-1} \boldsymbol{\mu}_{a_t}}, \frac{1}{2\sqrt{2}} \Big)
& \leq
\sqrt{(T-M) \sum_{t=M+1}^{\top} \min \Big( \sqrt{ \boldsymbol{\mu}_{a_t}^{\top} \mathbb{E}[\boldsymbol{V}_t]^{-1} \boldsymbol{\mu}_{a_t}}, \frac{1}{2\sqrt{2}} \Big)^2 } \\
& \leq
\sqrt{\frac{D}{8 \log(9/8)} (T-M) \log \Big( 1 + \frac{\omega}{ \lambda}(T-M) \Big)}. \\
\end{aligned}
\end{equation}

Therefore we complete the proof of Lemma \ref{lemma: hidden_sum_reg_c_1}.  
\end{proof}

\subsubsection{Proof of Lemma~\ref{lemma: iter_E_upsilon}}
\label{sec: proof_lemma_iter_E_upsilon}

We begin with a lemma that will be utilized in the derivations of Lemma~\ref{lemma: iter_E_upsilon}:

\begin{lemma}[Determinant of Symmetric PSD Matrices Sum]
\label{lemma: symmetric_det}
Let $\boldsymbol{A} \in \mathbb{R}^{n \times n}$ be a symmetric and positive definite matrix, and $\boldsymbol{B} \in \mathbb{R}^{n \times n}$ be a symmetric and positive (semi-) definite matrix. Then we have 

\[
\emph{det} \left( \boldsymbol{A} + \boldsymbol{B} \right) 
\geq
\emph{det} \left( \boldsymbol{A}\right) + \emph{det} \left( \boldsymbol{B}\right) 
\]
\end{lemma}

\begin{proof}

\begin{equation}
\begin{aligned}
\label{eq: symmetric_det1}
\text{det} \left( \boldsymbol{A} + \boldsymbol{B} \right) 
=
\text{det} \left( \boldsymbol{A}\right) 
\text{det} \left( \boldsymbol{I} + \boldsymbol{A}^{-\frac{1}{2}} \boldsymbol{B} \boldsymbol{A}^{-\frac{1}{2}} \right).
\end{aligned}
\end{equation}

Let $\lambda_1, ..., \lambda_n$ be the eigenvalues of $\boldsymbol{A}^{-\frac{1}{2}} \boldsymbol{B} \boldsymbol{A}^{-\frac{1}{2}}$. 
Since $\boldsymbol{A}^{-\frac{1}{2}} \boldsymbol{B} \boldsymbol{A}^{-\frac{1}{2}}$ is positive (semi-) definite, we have $\lambda_i \geq 0, \forall i \in [n]$, which implies

\begin{equation}
\begin{aligned}
\label{eq: symmetric_det2}
\text{det} \left( \boldsymbol{I} + \boldsymbol{A}^{-\frac{1}{2}} \boldsymbol{B} \boldsymbol{A}^{-\frac{1}{2}} \right)
=
\prod_{i=1}^{n} (1 + \lambda_i) 
\geq 
1 + \prod_{i=1}^{n} \lambda_i
= 
\text{det} (\boldsymbol{I}) + \text{det} \left( \boldsymbol{A}^{-\frac{1}{2}} \boldsymbol{B} \boldsymbol{A}^{-\frac{1}{2}} \right).
\end{aligned}
\end{equation}

Combining Eq.\ref{eq: symmetric_det1} with Eq.~\ref{eq: symmetric_det2} concludes the proof.

\end{proof}

\begin{proof}[Proof of Lemma~\ref{lemma: iter_E_upsilon}]

For $\boldsymbol{V}_t$ and $\mathbb{E}[\boldsymbol{V}_t]$, by definition and $w_t \geq 1$,

\[
\boldsymbol{V}_{t}
=
\boldsymbol{V}_{t-1} + w_{t} \boldsymbol{r}_{a_t, t} \boldsymbol{r}_{a_t, t}^{\top}
\quad \text{and} \quad
\boldsymbol{V}_{0} = \lambda \boldsymbol{I}
\]
\[
\mathbb{E}[\boldsymbol{V}_{t}]
=
\mathbb{E}[\boldsymbol{V}_{t-1} + w_{t} \boldsymbol{r}_{a_t, t} \boldsymbol{r}_{a_t, t}^{\top}] 
\succeq
\mathbb{E}[\boldsymbol{V}_{t-1} + \boldsymbol{r}_{a_t, t} \boldsymbol{r}_{a_t, t}^{\top}] 
=
\mathbb{E}[\boldsymbol{V}_{t-1}] 
+ \boldsymbol{\mu}_{a_t} \boldsymbol{\mu}_{a_t}^{\top} + \Sigma_{r, a_t}.
\]

Since $\mathbb{E}[\boldsymbol{V}_{t}]$ is symmetric and positive definite, we have 

\begin{small}
\begin{equation}
\begin{aligned}
\label{eq: iter_E_upsilon1}
\text{det} \left( \mathbb{E}[\boldsymbol{V}_{t}] \right)
& \geq 
\text{det} \left( \mathbb{E}[\boldsymbol{V}_{t-1}] + \boldsymbol{\mu}_{a_t} \boldsymbol{\mu}_{a_t}^{\top} + \Sigma_{r, a_t} \right) \\
& = 
\text{det} \left( 
\mathbb{E}[\boldsymbol{V}_{t-1}]^{\frac{1}{2}} \left( 
\boldsymbol{I} + \mathbb{E}[\boldsymbol{V}_{t-1}]^{-\frac{1}{2}} \left (\boldsymbol{\mu}_{a_t} \boldsymbol{\mu}_{a_t}^{\top} + \Sigma_{r,a_t} \right) \mathbb{E}[\boldsymbol{V}_{t-1}]^{-\frac{1}{2}}
\right)  \mathbb{E}[\boldsymbol{V}_{t-1}]^{\frac{1}{2}}
\right) \\
& = 
\text{det} \left( 
\mathbb{E}[\boldsymbol{V}_{t-1}] \right) 
\text{det} \left( 
\boldsymbol{I} + \mathbb{E}[\boldsymbol{V}_{t-1}]^{-\frac{1}{2}} \left (\boldsymbol{\mu}_{a_t} \boldsymbol{\mu}_{a_t}^{\top} + \Sigma_{r,a_t} \right) \mathbb{E}[\boldsymbol{V}_{t-1}]^{-\frac{1}{2}}
\right) \\
& \underset{(a)}{\geq}
\text{det} \left( 
\mathbb{E}[\boldsymbol{V}_{t-1}] \right) 
\left(
\text{det} \left( 
\boldsymbol{I} + \mathbb{E}[\boldsymbol{V}_{t-1}]^{-\frac{1}{2}} \boldsymbol{\mu}_{a_t} \boldsymbol{\mu}_{a_t}^{\top} \mathbb{E}[\boldsymbol{V}_{t-1}]^{-\frac{1}{2}}
\right)
+ \text{det} \left( 
\mathbb{E}[\boldsymbol{V}_{t-1}]^{-\frac{1}{2}}  \Sigma_{r,a_t} \mathbb{E}[\boldsymbol{V}_{t-1}]^{-\frac{1}{2}}
\right)
\right) \\
\end{aligned}
\end{equation}
\end{small}

where (a) holds since $\left( 
\boldsymbol{I} + \mathbb{E}[\boldsymbol{V}_{t-1}]^{-\frac{1}{2}} \boldsymbol{\mu}_{a_t} \boldsymbol{\mu}_{a_t}^{\top} \mathbb{E}[\boldsymbol{V}_{t-1}]^{-\frac{1}{2}}
\right)$ and $\left( 
\mathbb{E}[\boldsymbol{V}_{t-1}]^{-\frac{1}{2}}  \Sigma_{r,a_t} \mathbb{E}[\boldsymbol{V}_{t-1}]^{-\frac{1}{2}}
\right)$ are positive definite and applying Lemma~\ref{lemma: symmetric_det} yields the result.

Let $\mathbb{E}[\boldsymbol{V}_{t-1}]^{-\frac{1}{2}} \boldsymbol{\mu}_{a_t} = \boldsymbol{v}_t$, and we observe that 
\[
\left( \boldsymbol{I} + \boldsymbol{v}_t \boldsymbol{v}_t^{\top} \right) \boldsymbol{v}_t
=
\boldsymbol{v}_t + \boldsymbol{v}_t \left( \boldsymbol{v}_t^{\top} \boldsymbol{v}_t \right)
=
\left(1 + \boldsymbol{v}_t^{\top} \boldsymbol{v} \right) \boldsymbol{v}_t.
\]

Hence, $1 + \boldsymbol{v}_t^{\top} \boldsymbol{v}$ is an eigenvalue of $\boldsymbol{I} + \boldsymbol{v}_t \boldsymbol{v}_t^{\top}$. And since $\boldsymbol{v}_t \boldsymbol{v}_t^{\top}$ is a rank-1 matrix, all other eigenvalue of $\boldsymbol{I} + \boldsymbol{v}_t \boldsymbol{v}_t^{\top}$ equal to 1, implying

\begin{equation}
\begin{aligned}
\label{eq: iter_E_upsilon2}
\text{det} \left( 
\boldsymbol{I} + \mathbb{E}[\boldsymbol{V}_{t-1}]^{-\frac{1}{2}} \boldsymbol{\mu}_{a_t} \boldsymbol{\mu}_{a_t}^{\top} \mathbb{E}[\boldsymbol{V}_{t-1}]^{-\frac{1}{2}}
\right) 
& =
\text{det} \left( \boldsymbol{I} + \boldsymbol{v}_t \boldsymbol{v}_t^{\top} \right)\\
& =
1 + \boldsymbol{v}_t \boldsymbol{v}_t^{\top} \\
& =
1 + \left( \mathbb{E}[\boldsymbol{V}_{t-1}]^{-\frac{1}{2}} \boldsymbol{\mu}_{a_t} \right)^{\top} \left( \mathbb{E}[\boldsymbol{V}_{t-1}]^{-\frac{1}{2}} \boldsymbol{\mu}_{a_t} \right) \\
& = 
1 + \boldsymbol{\mu}_{a_t}^{\top} \mathbb{E}[\boldsymbol{V}_{t-1}]^{-1} \boldsymbol{\mu}_{a_t}.
\end{aligned}
\end{equation}

Combining Eq.~\ref{eq: iter_E_upsilon1} and Eq.~\ref{eq: iter_E_upsilon2}, we have

\[
\begin{aligned}
\text{det} \left( \mathbb{E}[\boldsymbol{V}_{t+1}] \right)
& \geq
\text{det} \left( 
\mathbb{E}[\boldsymbol{V}_{t-1}] \right) 
\left(
1 + \boldsymbol{\mu}_{a_t}^{\top} \mathbb{E}[\boldsymbol{V}_{t-1}]^{-1} \boldsymbol{\mu}_{a_t}
+ \text{det} \left( 
\mathbb{E}[\boldsymbol{V}_{t-1}]^{-\frac{1}{2}}  \Sigma_{r,a_t} \mathbb{E}[\boldsymbol{V}_{t-1}]^{-\frac{1}{2}}
\right)
\right) \\
\end{aligned}
\]

The solution of Lemma~\ref{lemma: iter_E_upsilon} follows from induction.
\end{proof}

\subsubsection{Proof of Lemma~\ref{lemma: log_E_upsilon}}
\label{sec: proof_lemma_log_E_upsilon}

\begin{proof}

For the proof of this lemma, we follow the main idea of Determinant-Trace Inequality in OFUL~\citep{abbasi2011improved} (Lemma 10).
Specifically, by the definition of $\boldsymbol{V}_{t}$, we have 

\begin{equation}
\begin{aligned}
\label{eq: log_E_upsilon1}
\log \left( \frac{ \text{det} \left(\mathbb{E}[\boldsymbol{V}_{T}] \right) }{ \text{det} \left(\mathbb{E}[\boldsymbol{V}_{M}] \right)} \right)
& =
\log \left( \text{det} \left( \frac{ \mathbb{E}[\boldsymbol{V}_{M}] + \sum_{t=M+1}^{\top} \mathbb{E}[ w_t \boldsymbol{r}_{a_t,t} \boldsymbol{r}_{a_t,t}^{\top}] }{ \mathbb{E}[\boldsymbol{V}_{M}] } \right) \right) \\
& \underset{(a)}{\leq}
\log \left( \text{det} \left( 1 + \frac{ \sum_{t=M}^{\top} \mathbb{E}[ w_t \boldsymbol{r}_{a_t,t} \boldsymbol{r}_{a_t,t}^{\top}] }{ \lambda \boldsymbol{I} } \right) \right) \\
& = 
\log \left( \text{det} \left( 1 +  \frac{1}{\lambda} \sum_{t=M+1}^{\top} \mathbb{E}[ w_t \boldsymbol{r}_{a_t,t} \boldsymbol{r}_{a_t,t}^{\top}]  \right) \right), \\
\end{aligned}
\end{equation}

where (a) holds since $\text{det}(\mathbb{E}[\boldsymbol{V}_{M}]) \geq \text{det}(\mathbb{E}[\boldsymbol{V}_{0}]) = \lambda \boldsymbol{I}$. Let $\xi_1, ... , \xi_D$ denote the eigenvalues of $\sum_{t=M+1}^{\top} \mathbb{E}[ w_t \boldsymbol{r}_{a_t,t} \boldsymbol{r}_{a_t,t}^{\top}]$, and note:

\begin{equation}
\begin{aligned}
\label{eq: log_E_upsilon2}
\sum_{d=1}^{D} \xi_d 
& =
\text{Trace} \left( \sum_{t=M+1}^{\top} \mathbb{E}[ w_t \boldsymbol{r}_{a_t,t} \boldsymbol{r}_{a_t,t}^{\top}] \right) \\
& =
\sum_{t=M+1}^{\top} \mathbb{E} \left[\text{Trace} \left( w_t \boldsymbol{r}_{a_t,t} \boldsymbol{r}_{a_t,t}^{\top} \right) \right] \\
& =
\sum_{t=M+1}^{\top} \mathbb{E} \left[\text{Trace} \left( \frac{\omega}{\Vert \boldsymbol{r}_{a_t,t} \Vert_2^2} \boldsymbol{r}_{a_t,t} \boldsymbol{r}_{a_t,t}^{\top} \right) \right] \\
& =
\sum_{t=M+1}^{\top} \mathbb{E} \left[ \frac{\omega \Vert \boldsymbol{r}_{a_t,t} \Vert_2^2}{\Vert \boldsymbol{r}_{a_t,t} \Vert_2^2} \right] \\
& = 
\omega (T-M).
\end{aligned}
\end{equation}

Combining Eq.~\ref{eq: log_E_upsilon1} and Eq.~\ref{eq: log_E_upsilon2} implies 

\[
\begin{aligned}
\log \left( \frac{ \text{det} \left(\mathbb{E}[\boldsymbol{V}_{T}] \right) }{ \text{det} \left(\mathbb{E}[\boldsymbol{V}_{M}] \right)} \right)
& \leq
\log \left( \text{det} \left( 1 +  \frac{1}{\lambda} \sum_{t=M+1}^{\top} \mathbb{E}[ w_t \boldsymbol{r}_{a_t,t} \boldsymbol{r}_{a_t,t}^{\top}]  \right) \right) \\
& =
\log \left( \prod_{i=1}^{D} \left( 1 + \frac{\xi_i}{\lambda} \right) \right) \\
& =
D \log \left( \prod_{i=1}^{D} \left( 1 + \frac{\xi_i}{\lambda} \right) \right)^{\frac{1}{D}} \\
& \underset{(a)}{\leq}
D \log \left( \frac{1}{D} \sum_{i=1}^{D} \left( 1 + \frac{\xi_i}{\lambda} \right) \right) \\
& \underset{(b)}{\leq}
D \log \left( 1 +  \frac{\omega (T-M)}{D \lambda} \right), \\
\end{aligned}
\]

where (a) follows from the inequality of arithmetic and geometric means, and (b) follows from Eq.~\ref{eq: log_E_upsilon2}.
\end{proof}

\section{Analyses for Section \ref{sec: knonw}}

\subsection{Proof of Theorem \ref{theorem:up_bd_stat} (Unknown Preference Case: Regret Analysis of Algorithm \ref{alg:PRUCB_UP})}
\label{sec:app_up_bd_stat}

The presented Theorem \ref{theorem:up_bd_stat} establishes the upper bound of regret $R(T)$ for PRUCB-HP under unknown preference environment.
For the convenience of the reader, we re-state some notations that will be used in the following before going to proof. 
Here we first focus on the user-specific regret $R^{n}(T)$. For notational simplicity, we fix a user $n$ and omit the user index $n$ as superscript.
Let $a^*_t$ be the dynamic oracle at time step $t$, $\Delta_{i,t} = \boldsymbol{\mu}_{a_{t}^{*}} - \boldsymbol{\mu}_{i,t} \in \mathbb{R}^D, \forall t \in [1,T]$ be the gap of expected rewards between suboptimal arm $i$ and best arm $a_t^*$ at time step $t$. 
Define $\mathcal{T}_{i} = \{ t \in [T] | i \in \mathcal{A}_t, a^*_t \neq i \}$ be the set of episodes when arm $i$ being located in $\mathcal{A}_t$ and serving as a suboptimal arm over $T$. 
Let $\eta_{i}^{\downarrow} = \min_{t \in \mathcal{T}_{i}}\{\boldsymbol{\overline{c}}_{t}^{\top} \Delta_{i,t} \}$ and $\eta_{i}^{\uparrow} = \min_{t \in \mathcal{T}_{i}}\{\boldsymbol{\overline{c}}_{t}^{\top} \Delta_{i,t} \}$ refer to the lower bound and upper bound of the expected gap of overall-rewards between $i$ and $a_t^*$ over $T$.
$\Vert \Delta_i^{\uparrow} \Vert_2 = \max_{\{t,j\} \in \mathcal{T}_i \times \mathcal{A}_t/i } \Vert \boldsymbol{\mu}_{i,t} - \boldsymbol{\mu}_{j,t} \Vert_2$ denotes the largest Euclidean distance between the expected rewards of arm $i$ and other arms over $\mathcal{T}_i$.

\subsubsection{Proof Sketch of Theorem \ref{theorem:up_bd_stat}}
\label{sec:app_pr_sketch_up_bd_stat}

We analyze the expected number of times in $T$ that one suboptimal arm $i \neq a^*_t$ is played, denoted by $\tilde{N}_{i,T}$.
Since regret performance is affected by both reward and preference estimates, we introduce a hyperparameter $\epsilon$ to quantify the accuracy of the empirical estimation $\boldsymbol{\hat{c}}_t$.

The key idea is that by using $\epsilon$ to measure the closeness of the preference estimation $\hat{\boldsymbol{c}}_t$ to the true expected vector $\overline{\boldsymbol{c}}$, the event of pulling a suboptimal arm can be decomposed into two disjoint sets based on whether $\hat{\boldsymbol{c}}_t$ is sufficiently accurate, as determined by  $\epsilon$. And the parameter $\epsilon$ can be tuned to optimize the final regret.
This decomposition allows us to address the problem of joint impact from the preference and reward estimate errors,  analyzing the undesirable behaviors of leaner caused by estimation errors of reward $\boldsymbol{\hat{r}}$ and preference $\boldsymbol{\hat{c}}$ independently. 

For suboptimal pulls induced by error of $\boldsymbol{\hat{r}}$, 
we show that the pseudo episode set $\mathcal{M}_i$ where the suboptimal arm $i$ is considered suboptimal under the preference estimate align with the true suboptimal episode set $\tilde{N}_{i,T}$, and the best arm within $\mathcal{M}_i$ is consistently identified as better than arm $i$. Using this insight, we show that this case can be transferred to a new preference known instance with a narrower overall-reward gap w.r.t $\epsilon$.

For suboptimal pulls due to error of $\boldsymbol{\hat{c}}$, 
we first relax the suboptimal event set to an overall-reward estimation error set, eliminating the joint dependency on reward and preference from action $a_t$. Then we develop a tailored-made error bound (Lemma \ref{lemma: error_distance}) on preference estimation, which transfers the original error set to a uniform imprecise estimation set on preference, such that a tractable formulation of the estimation deviation can be constructed.

\subsubsection{Proof of Theorem \ref{theorem:up_bd_stat}}
\label{sec:app_pr_up_bd_stat}

We begin with a more general upper bound (Proposition \ref{prop: N_known_changing}) for the learner's behavior using a policy that optimizes the inner product between the reward upper confidence bound (UCB) of arms and an arbitrary dynamic vector $\boldsymbol{b}_t$. It demonstrates that after a sufficiently large number of samples (on the order of $\mathcal{O}(\log T)$) for each arm $i$, for the episodes where the inner product of its rewards expectations with $\boldsymbol{b}_t$ is not highest, the expected number of times arm $i$ is pulled can be well controlled by a constant.
The proof of Proposition \ref{prop: N_known_changing} is provided in Appendix~\ref{sec: app_pr_prop_N_known_changing}.

\begin{proposition}
\label{prop: N_known_changing}
Let $\boldsymbol{b}_t \in \mathbb{R}^D$ be an arbitrary bounded vector at time step $t$ with $\Vert \boldsymbol{b}_t \Vert_1 \leq M$, define $\mathcal{M}_i := \{ t \in \mathcal{T}_i \mid i \neq \argmaxA_{j \in \mathcal{A}_t} \boldsymbol{b}_t^{\top} \boldsymbol{\mu}_j \}, \forall i \in [K]$. For the policy of $a_t = \argmaxA \Phi ( \boldsymbol{b}_t, \hat{\boldsymbol{r}}_{i,t}+\sqrt{\frac{\log(t/\alpha)}{\max \{1, N_{i, t } \} }} \boldsymbol{e})$, for any arm $i \in [K]$, any subset $\mathcal{M}^{o}_i \subseteq \mathcal{M}_i$, we have
\[
\mathbb{E} \left[ \sum_{t \in \mathcal{M}^{o}_i} \mathds{1}_{\{a_t = i \} } \right] 
\leq
\frac{4 M^2 \log{(\frac{T}{\alpha})}}{L_i^2}
+
\frac{ |\mathcal{B}^{+}_{T}| \pi^2 \alpha^2} {3},
\]
where $L_i = \min_{t \in \mathcal{M}^{o}_i} \{ \max_{j \in \mathcal{A}_t \setminus i } \{\boldsymbol{b}_t^{\top} (\boldsymbol{\boldsymbol{\mu}_j - \boldsymbol{\mu}_{i}
} )\} \}$, $\mathcal{B}^{+}_{T} := \{{[\boldsymbol{b}_{1}(d), \boldsymbol{b}_{2}(d), ..., \boldsymbol{b}_{T}(d)]} \neq \boldsymbol{0}, \forall d \in [D]\}$ is the collection set of non-zero $[\boldsymbol{b}(d)]^{\top}$ sequence. 
\end{proposition}

\begin{proof}[Proof of Theorem \ref{theorem:up_bd_stat}]
Here we focus on the user-specific regret $R^{n}(T)$. For notational simplicity, we fix a user $n$ and omit the user index $n$ as superscript.
Let $\tilde{N}_{i,T} = \sum_{t=1}^{T} \mathds{1}_{\{a_t = i \neq a_{t}^{*} \} }$. be the number of pulls of each arm $i$ when it serves as a suboptimal arm within horizon $T$. We first analyze $\tilde{N}_{i,T}$ and then extend to the final regret $R(T)$.
The proof consists of several steps.

\textbf{Step-1 ($N_{i,T}$ Decomposition with Parameter $\epsilon$):}

For any $i \in \mathcal{A}_t, i \neq a^*_t$, at any time step $t \in [T]$, with a hyper-parameter $0 < \epsilon < \eta_i^{\downarrow}$ introduced, we can formulate the the number of times the suboptimal arm $i$ is played as follows:

\begin{equation}
\label{eq: N_up_bd_stat}
\begin{aligned}
\tilde{N}_{i,T} &= \sum_{t=1}^{\top} \mathds{1}_{\{a_t = i \neq a_{t}^{*}\}} 
= 
\underbrace{\sum_{t \in \mathcal{T}_{i}} \mathds{1}_{\{a_t = i \neq a_{t}^{*}, \hat{\boldsymbol{c}}_{t}^{\top} \mu_{a_{t}^{*}} > \hat{\boldsymbol{c}}_{t}^{\top} \mu_{i} + \eta_i^{\downarrow} - \epsilon \}}}
_{\tilde{N}_{i,T}^{\tilde{\boldsymbol{r}}}: \text { \parbox{110pt}{\centering \emph{Suboptimal pulls caused by imprecise \textbf{reward estimation} }} } }
+ 
\underbrace{\sum_{t \in \mathcal{T}_{i}} \mathds{1}_{\{a_t = i \neq a_{t}^{*}, \hat{\boldsymbol{c}}_{t}^{\top} \mu_{a_{t}^{*}} \leq \hat{\boldsymbol{c}}_{t}^{\top} \mu_{i} + \eta_i^{\downarrow} - \epsilon \}}}
_{\tilde{N}_{i,T}^{\tilde{\boldsymbol{c}}}: \text {\parbox{120pt}{\centering \emph{Suboptimal pulls caused by imprecise \textbf{preference estimation} }} } }
\end{aligned}
\end{equation}

The technical idea behind is that by introducing $\epsilon$ to measure the closeness of the preference estimate $\hat{\boldsymbol{c}}_t$ to the true expected vector $\overline{\boldsymbol{c}}$, we can decouple the undesirable behaviors caused by either reward estimation error or preference estimation error.
Let $N_{i,T}^{\widetilde{\boldsymbol{r}}}$ and $N_{i,T}^{\widetilde{\boldsymbol{c}}}$ denote the times of suboptimal pulling induced by imprecise reward estimation and preference estimation (shown in Eq.~\ref{eq: N_up_bd_stat}).
We use $\mathbb{E}_{\epsilon}$ and $\mathbb{P}_{\epsilon}$ to denote the probability distribution and expectation under parameter $\epsilon$. Next, we will study these two terms separately.

\textbf{Step-2 (Bounding $N_{i,T}^{\hat{\boldsymbol{r}}}$):}

Define $\mathcal{M}_i$ as the set of episodes that arm $i$ achieves suboptimal expected overall-reward under preference estimation $\hat{\boldsymbol{c}}_t$, i.e., 
$\mathcal{M}_i := \{ t \in \mathcal{T}_i \mid i \neq \argmaxA_{j \in \mathcal{A}_t} \hat{\boldsymbol{c}}_t^{\top} \boldsymbol{\mu}_j \}$.
By the definition of event $N_{i,T}^{\hat{\boldsymbol{r}}}$, we have $\hat{\boldsymbol{c}}_{t}^{\top} \Delta_{i,t} > \eta_i^{\downarrow} - \epsilon > 0$ holds for all $t \in \mathcal{T}_i$, which implies that $a^{*}_{t}$ still yields a better overall reward than $i$ given the estimated preference coefficient $\hat{\boldsymbol{c}}_{t}$ over $\mathcal{T}_i$.
Thus the suboptimal pulling of arm $i$ is attributed to the imprecise rewards estimations of arms. 
Additionally, we have $\mathcal{M}_i = \mathcal{T}_i$ since arm $i$ is at least worse than $a^*_t$ under the preference estimation $\hat{\boldsymbol{c}}_{t}$ for all episode $t \in \mathcal{T}_i$. Hence for $N_{i,T}^{\widetilde{\boldsymbol{r}}}$ we have
\begin{equation}
N_{i,T}^{\widetilde{\boldsymbol{r}}} = 
\sum_{t \in \mathcal{T}_i} \mathds{1}_{\{a_t = i, \hat{\boldsymbol{c}}_{t}^{\top} \Delta_{i,t} > \eta_{i}^{\downarrow} - \epsilon\}}
=
\sum_{t \in \mathcal{M}_i} \mathds{1}_{\{a_t = i, \hat{\boldsymbol{c}}_{t}^{\top} \Delta_{i,t} > \eta_{i}^{\downarrow} - \epsilon\}}
\end{equation}

Moreover, recall that PRUCB-APM also leverages $\boldsymbol{\hat{c}}_t$ for optimistic arm selection, i.e., $a_t = \argmaxA f( \boldsymbol{\hat{c}}_t, \hat{\boldsymbol{r}}_{i,t}+\sqrt{\frac{ \log(t/\alpha)}{\max \{1, N_{i, t } \} }} \boldsymbol{e})$,  by Proposition \ref{prop: N_known_changing}, we have 

\begin{equation}
\label{eq:N_unkown_stat}
\mathbb{E}_{\epsilon} \left[ \sum_{t \in \mathcal{M}_i} \mathds{1}_{\{a_t = i, \hat{\boldsymbol{c}}_{t}^{\top} \Delta_{i,t} > \eta_{i}^{\downarrow} - \epsilon\}} \right] 
\leq
\mathbb{E} \left[ \sum_{t \in \mathcal{M}_i} \mathds{1}_{\{a_t = i \} } 
\right] 
\leq
\frac{4 \delta^2 \log{(\frac{T}{\alpha})}}{L_i^2}
+
\frac{ |\hat{\mathcal{C}}^{+}_{T}| \pi^2 \alpha^2 } {3}.
\end{equation}

Additionally, since $\hat{\boldsymbol{c}}_{t}^{\top} \Delta_{i,t} > \eta_i^{\downarrow} - \epsilon > 0$ holds for all $t \in \mathcal{T}_i$, it implies that

\[
L_i = \min_{t \in \mathcal{M}_i} \{ \max_{j \in \mathcal{A}_t \setminus i } \{\boldsymbol{\hat{c}}_t^{\top} (\boldsymbol{\boldsymbol{\mu}_j - \boldsymbol{\mu}_{i} } )\} \} 
\geq 
\min_{t \in \mathcal{M}_i} \hat{\boldsymbol{c}}_{t}^{\top} \Delta_{i,t} 
>
\eta_i^{\downarrow} - \epsilon
.
\]

Combining above results, and by $|\hat{\mathcal{C}}^{+}_{T}| \leq D$,
we have the expectation of $N_{i,T}^{\tilde{\boldsymbol{r}}}$ in Eq.~\ref{eq: N_up_bd_stat} can be upper-bounded as follows:

\begin{equation}
\begin{aligned}
\label{eq: upbd_term_1_stat}
\mathbb{E}_{\epsilon} \left[ \tilde{N}_{i,T}^{\tilde{\boldsymbol{r}}} \right]
& = 
\mathbb{E}_{\epsilon} \left[ \sum_{t \in \mathcal{T}_i} \mathds{1}_{ \{ a_t = i, \hat{\boldsymbol{c}}_{t}^{\top} \Delta_{i,t} > \eta_{i}^{\downarrow} - \epsilon \}} \right]  
\leq 
\frac{4 \delta^2 \log (T/\alpha)}{(\eta_{i}^{\downarrow} - \epsilon)^2}
+
D \frac{\pi^2 \alpha^2}{3}.
\end{aligned}
\end{equation}

\textbf{Step-3 (Bounding $N_{i,T}^{\widetilde{\boldsymbol{c}}}$):}

We begin with stating one tailored-made preference estimation error bound which will be utilized in our derivation.

\begin{lemma}
\label{lemma: error_distance}
For any non-zero vectors $\Delta, \boldsymbol{\overline{c}} \in \mathbb{R}^k$, and all $\epsilon \in \mathbb{R}$, if $\boldsymbol{\overline{c}}^{\top} \Delta > \epsilon$, then for any vector $\boldsymbol{c^{\prime}}$ s.t, $\boldsymbol{c^{\prime}}^{\top} \Delta = \epsilon$, we have
\[
\Vert \boldsymbol{\overline{c}}- \boldsymbol{c^{\prime}}  \Vert_2 \geq \frac{ \boldsymbol{\overline{c}}^{\top} \Delta  - \epsilon}{\Vert \Delta \Vert_2}.
\]
\end{lemma}
Please see Appendix \ref{sec: proof_lemma_error_distance} for the proof of Lemma \ref{lemma: error_distance}

Firstly we relax the instantaneous event set of $N_{i,T}^{\widetilde{\boldsymbol{c}}}$ in Eq.~\ref{eq: N_up_bd_stat} into a pure estimation error case as:
\begin{equation}
\begin{aligned}
\left\{ a_t = i \neq a^*_t, \hat{\boldsymbol{c}}_{t}^{\top} \mu_{a^{*}_t} \leq \hat{\boldsymbol{c}}_{t}^{\top} \mu_{i} + \eta_{i}^{\downarrow} - \epsilon \right\}
\subset
\left\{\hat{\boldsymbol{c}}_{t}^{\top} \mu_{a^{*}_t} \leq \hat{\boldsymbol{c}}_{t}^{\top} \mu_{i} + \eta_{i}^{\downarrow} - \epsilon \right\}
= 
\left\{\hat{\boldsymbol{c}}_{t}^{\top} \Delta_{i,t} \leq \eta_{i}^{\downarrow} - \epsilon \right\}.
\end{aligned}
\end{equation}

Then, according to Lemma \ref{lemma: error_distance} above, we can transfer the original overall-reward gap estimation error to the preference estimation error.
More specifically, since ${\boldsymbol{\overline{c}}^{\top} \Delta_{i,t} > \eta_{i}^{\downarrow} - \epsilon }$ always holds, for any $t \in \mathcal{T}_i$, by applying Lemma~\ref{lemma: error_distance}, we have 
\begin{equation}
\begin{aligned}
\label{eq: term_2_set_stat}
\left\{\hat{\boldsymbol{c}}_{t}^{\top} \Delta_{i,t} \leq \eta_{i}^{\downarrow} - \epsilon \right\} 
& \subset
\left\{ \Vert \overline{\boldsymbol{c}} - \hat{\boldsymbol{c}}_t \Vert_2 \geq \frac{\overline{\boldsymbol{c}}^{\top} \Delta_{i,t} - ( \eta_{i}^{\downarrow} - \epsilon ) }{ \Vert \Delta_{i,t} \Vert_2 } \right\} \\
& \subset
\left\{ \Vert \overline{\boldsymbol{c}} - \hat{\boldsymbol{c}}_t \Vert_2 
\geq
\frac{ \epsilon }{ \Vert \Delta_{i,t} \Vert_2 } \right\}.
\end{aligned}
\end{equation}

\begin{equation}
\begin{aligned}
\label{eq: term_2_set_prob_stat}
& \implies
\mathbb{P}_{\epsilon} \left( a_t = i \neq a^*_t, \hat{\boldsymbol{c}}_{t}^{\top} \mu_{a^{*}_t} \leq \hat{\boldsymbol{c}}_{t}^{\top} \mu_{i} + \eta_i^{\downarrow} - \epsilon \right)
\leq
\mathbb{P}_{\epsilon} \left( \Vert \overline{\boldsymbol{c}} - \hat{\boldsymbol{c}}_t \Vert_2 \geq \frac{\epsilon}{ \Vert \Delta_{i,t} \Vert_2 } \right). \\
\end{aligned}
\end{equation}

For the RHS term in Eq. (\ref{eq: term_2_set_prob_stat}), we have 
\begin{equation}
\begin{aligned}
\label{eq: term_2_prob_2_stat}
\mathbb{P}_{\epsilon} \bigg( 
\Vert \overline{\boldsymbol{c}} - \hat{\boldsymbol{c}}_t \Vert_2 
\geq \frac{\epsilon}{ \Vert \Delta_{i,t} \Vert_2 } \bigg) 
& =
\mathbb{P}_{\epsilon} \left( \sum_{d = 1}^{D} \left( \overline{\boldsymbol{c}} (d) - \hat{\boldsymbol{c}}_t (d) \right)^2 
\geq
\frac{\epsilon^2}{ \Vert \Delta_{i,t} \Vert_2^2 } \right) \\
& \underset{(a)}{\leq}
\sum_{d = 1}^{D} \mathbb{P}_{\epsilon} \left( | \overline{\boldsymbol{c}} (d) - \hat{\boldsymbol{c}}_t (d) | \geq 
\frac{\epsilon}{\sqrt{D} \Vert \Delta_{i,t} \Vert_2 } \right)
\end{aligned}
\end{equation}

where (a) holds by the union bound and the fact that there must be at least one objective $d \in [D]$ satisfying $\left( \overline{\boldsymbol{c}} (d) - \hat{\boldsymbol{c}}_t (d) \right)^2 \geq \frac{1}{D} \frac{\epsilon^2}{ \Vert \Delta_{i,t} \Vert_2^2 }$, otherwise the event would fail.
Note that for all $t \in (0, T]$, ${\boldsymbol{c}}_t$ follows same the distribution, and the deviation is exactly the radius of the preference confidence ellipse, thus we can use a tail bound for the confidence interval on empirical mean of i.i.d. sequence. 
Applying the the Hoeffding's inequality (Lemma~\ref{lemma: Hoeffding}), the probability for each objective $d \in [D]$ can be upper-bounded as follows:
\begin{equation}
\begin{aligned}
\mathbb{P}_{\epsilon} \left( | \overline{\boldsymbol{c}} (d) - \hat{\boldsymbol{c}}_t (d) | \geq 
\frac{\epsilon}{\sqrt{D} \Vert \Delta_{i,t} \Vert_2 } \right)
& \leq
2 \exp \left( - \frac{ 2 \epsilon^2 t^2 }{ D \Vert \Delta_{i,t} \Vert_2 \sum_{\tau=1}^{\top} \delta^2 } \right)
=
2 \exp \left( - \frac{ 2 \epsilon^2 }{ D \Vert \Delta_{i,t} \Vert_2^2 \delta^2 } t \right).
\end{aligned}
\end{equation}

Plugging above result back to Eq. \ref{eq: term_2_prob_2_stat} and combining with Eq. \ref{eq: term_2_set_prob_stat}, 
we can obtain the upper-bound for the expectation of $N_{i,T}^{\widetilde{\boldsymbol{c}}}$ in Eq.~\ref{eq: N_up_bd_stat} as follows:
\begin{equation}
\begin{aligned}
\label{eq: upbd_term_2_stat}
\mathbb{E}_{\epsilon} \left[ N_{i,T}^{\widetilde{\boldsymbol{c}}} \right]
& = 
\mathbb{E}_{\epsilon} \left[ \sum_{t=1}^{T} \mathds{1}_{\{a_t = i \neq a^*, \hat{\boldsymbol{c}}_{t}^{\top} \mu_{a^{*}} \leq \hat{\boldsymbol{c}}_{t}^{\top} \mu_{i} + \eta_i - \epsilon \}} \right] \\
& = 
\sum_{t=1}^{T} \mathbb{P}_{\epsilon} \left( a_t = i \neq a^*, \hat{\boldsymbol{c}}_{t}^{\top} \mu_{a^{*}} \leq \hat{\boldsymbol{c}}_{t}^{\top} \mu_{i} + \eta_i - \epsilon \right) \\
& \leq
\sum_{t=1}^{T}
\mathbb{P}_{\epsilon} \bigg( 
\Vert \overline{\boldsymbol{c}} - \hat{\boldsymbol{c}}_t \Vert_2 
\geq \frac{\epsilon}{ \Vert \Delta_{i}^{\uparrow} \Vert_2 } \bigg) \\ 
& \leq
2 D \sum_{t=1}^{T} \exp \left( - \frac{ 2 \epsilon^2 }{ D \Vert \Delta_{i}^{\uparrow} \Vert_2^2 \delta^2 } t \right) \\
& \underset{(a)}{\leq}
\frac{2 D}{ \exp \left( \frac{ 2 \epsilon^2 }{ D \Vert \Delta_{i}^{\uparrow} \Vert_2^2 \delta^2 } \right) -1 }\\
& \leq
\frac{ D^2 \Vert \Delta_{i}^{\uparrow} \Vert_2^2 \delta^2 }{ \epsilon^2 },
\quad (\text{by } e^x \geq x+1, \forall x \geq 0 )
\end{aligned}
\end{equation} 

where (a) holds since for any $a>0$, we have
\begin{equation}
\begin{aligned}
\label{eq: geometric_converge}
\sum_{t=1}^{T} \left( e^{-a} \right)^{t} 
= 
\sum_{t=0}^{T-1} e^{-a} \cdot \left( e^{-a} \right)^{t} 
& \leq 
\sum_{t=0}^{\infty} e^{-a} \cdot \left( e^{-a} \right)^{t} \\
& = \frac{e^{-a}}{1 - e^{-a}} \quad \text{(by closed form of the geometric series)} \\
& = 
\frac{1}{e^a-1}.
\end{aligned}
\end{equation}

\textbf{Step-4 (Final $R(T)$ Derivation and Optimization over $\epsilon$):}

Combining Eq.\ref{eq: N_up_bd_stat} with the corresponding upper-bounds of $\mathbb{E}_{\epsilon} \left[ N_{i,T}^{\widetilde{\boldsymbol{r}}} \right]$ (Eq.\ref{eq: upbd_term_1_stat}) and $\mathbb{E}_{\epsilon} \left[ N_{i,T}^{\widetilde{\boldsymbol{c}}} \right]$ (Eq.\ref{eq: upbd_term_2_stat}), we can get
\begin{equation}
\begin{aligned}
\label{eq: E_N_up_bd_stat}
\mathbb{E} [\tilde{N}_{i,T}]
\leq
\frac{4 \delta^2 \log (T/\alpha)}{(\eta_i^{\downarrow} - \epsilon )^2}
+
D \frac{\pi^2 \alpha^2}{3}
+
\frac{ D^2 \Vert \Delta_{i}^{\uparrow} \Vert_2^2 \delta^2 }{ \epsilon^2 }.
\end{aligned}
\end{equation}

Note that for any $i \neq a^*_t$, the parameter $\epsilon \in (0, \eta_{i}^{\downarrow})$ can be optimally selected so as to minimize the RHS of Eq. \ref{eq: E_N_up_bd_stat}. 
For simplicity, taking $\epsilon = \frac{1}{2} \eta_{i}^{\downarrow}$ yields
\[
\mathbb{E}[N_{i,T}] 
\leq
\frac{16 \delta^2 \log{(\frac{T}{\alpha})}}{{\eta_{i}^{\downarrow}}^2} 
+
\frac{D \pi^2 \alpha^2} {3}
+
\frac{ 4 D^2 \Vert \Delta_{i}^{\uparrow} \Vert_2^2 \delta^2 }{ {\eta_{i}^{\downarrow}}^2 }
.
\]

Multiplying the results above by the expected overall-reward gap $\eta_i^{\downarrow}$ for all suboptimal arms $i \neq a^*_t$ and summing them up over the all $N$ users, we can derive the regret of PRUCB-HP in Theorem \ref{theorem:up_bd_stat}.
\end{proof}

\subsubsection{Proof of Proposition \ref{prop: N_known_changing}}
\label{sec: app_pr_prop_N_known_changing}

\begin{proof}

Define $\tilde{a}_{t}^{*} = \argmaxA_{j \in \mathcal{A}_t} \boldsymbol{b}_t^{\top} \boldsymbol{\mu}_j, \forall t \in (0,T]$, for any $\beta \in (0, T]$, we have

\begin{equation}
\begin{aligned}
\label{eq: term_1_stat}
\sum_{t \in \mathcal{M}^{o}_i} \mathds{1}_{\{a_t = i \} }
& \leq 
\sum_{t \in \mathcal{M}^{o}_i} \mathds{1}_{\{a_t = i, N_{i,t} \leq \beta \}}
+
\sum_{t \in \mathcal{M}^{o}_i} \mathds{1}_{\{a_t = i, N_{i,t} > \beta \}} \\
& \leq 
\beta
+
\sum_{t \in [T]} \mathds{1}_{\{a_t = i \neq \tilde{a}_{t}^{*}, N_{i,t} > \beta \}}.
\end{aligned}
\end{equation}

where the first term refers to the event of insufficient sampling (quantified by $\beta$) of arm $i$. 
, then for the event of second term, we have 
\begin{equation}
\begin{aligned}
\label{eq: event_ABC_t_stat1}
&{\{a_t = i \neq \tilde{a}_{t}^{*}, N_{i,t} > \beta \}} \\
& \subset
\Bigg\{ 
\underbrace{
\boldsymbol{b}_{t}^{\top} \hat{\boldsymbol{r}}_{i,t} > \boldsymbol{b}_{t}^{\top} \boldsymbol{\mu}_{i} + \boldsymbol{b}_{t}^{\top} \boldsymbol{e} \sqrt{\frac{ \log(t/\alpha)}{N_{i,t}}}
}_{\tilde{A}_t}, 
N_{i,t} > \beta \Bigg\} \\
& \qquad \qquad \cup
\Bigg\{ 
\underbrace{
\boldsymbol{b}_{t}^{\top} \hat{\boldsymbol{r}}_{\tilde{a}_{t}^{*},t} < \boldsymbol{b}_{t}^{\top} \boldsymbol{\mu}_{\tilde{a}_{t}^{*}} - \boldsymbol{b}_{t}^{\top} \boldsymbol{e} \sqrt{\frac{ \log(t/\alpha)}{N_{\tilde{a}_{t}^{*},t}}}
}_{\tilde{B}_t}, 
N_{i,t} > \beta \Bigg\} \\
& \qquad \qquad \cup
\Bigg\{ 
\underbrace{
\tilde{A}_t^{\mathsf{c}}, \tilde{B}_t^{\mathsf{c}}, 
\boldsymbol{b}_{t}^{\top} \hat{\boldsymbol{r}}_{i,t} + \boldsymbol{b}_{t}^{\top} \boldsymbol{e} \sqrt{\frac{ \log(t/\alpha)}{N_{i,t}}} 
\geq
\boldsymbol{b}_{t}^{\top} \hat{\boldsymbol{r}}_{\tilde{a}_{t}^{*},t} + \boldsymbol{b}_{t}^{\top} \boldsymbol{e} \sqrt{\frac{ \log(t/\alpha)}{N_{\tilde{a}_{t}^{*},t}}} , N_{i,t} > \beta }_{\tilde{\Gamma}_t} \Bigg\}.  \\
\end{aligned}
\end{equation}

Specifically, $\tilde{A}_t$ and $\tilde{B}_t$ denote the events where the constructed upper confidence bounds (UCBs) for arm $i$ or the optimal arm $a^{}$ fail to accurately bound their true expected rewards, indicating imprecise rewards estimation. Meanwhile, $\tilde{\Gamma}_t$ represents the event where the UCBs for both arms effectively bound their expected rewards, yet the UCB of arm $i$ still exceeds that of the arm $\tilde{a}_{t}^{*}$ though it yields the maximum value of $\boldsymbol{b}_t^{\top} \boldsymbol{\mu}_{\tilde{a}_{t}^{*}}$, leading to pulling of arm $i$. According to~\citep{auer2002finite}, at least one of these events must occur for an pulling of arm $i$ to happen at time step $t$.

For event $\tilde{\Gamma}_t$, the $\tilde{A}_t^{\mathsf{c}}$ and $\tilde{B}_t^{\mathsf{c}}$ imply
\[
\boldsymbol{b}_{t}^{\top} \boldsymbol{\mu}_{i} + \boldsymbol{b}_{t}^{\top} \boldsymbol{e} \sqrt{\frac{ \log(t/\alpha)}{N_{i,t}}}
\geq
\boldsymbol{b}_{t}^{\top} \hat{\boldsymbol{r}}_{i,t}
\quad \text{and} \quad
\boldsymbol{b}_{t}^{\top} \hat{\boldsymbol{r}}_{\tilde{a}_{t}^{*},t} \geq \boldsymbol{b}_{t}^{\top} \boldsymbol{\mu}_{\tilde{a}_{t}^{*}} - \boldsymbol{b}_{t}^{\top} \boldsymbol{e} \sqrt{\frac{ \log(t/\alpha)}{N_{\tilde{a}_{t}^{*},t}}}, \\
\]

indicating 
\[
\begin{aligned}
& \boldsymbol{b}_{t}^{\top} \boldsymbol{\mu}_{i} + 2 \boldsymbol{b}_{t}^{\top} \boldsymbol{e} \sqrt{\frac{ \log(t/\alpha)}{N_{i,t}}}
\geq
\boldsymbol{b}_{t}^{\top} \hat{\boldsymbol{r}}_{i,t} + \boldsymbol{b}_{t}^{\top} \boldsymbol{e} \sqrt{\frac{ \log(t/\alpha)}{N_{i,t}}} 
\geq
\boldsymbol{b}_{t}^{\top} \hat{\boldsymbol{r}}_{\tilde{a}_{t}^{*},t} + \boldsymbol{b}_{t}^{\top} \boldsymbol{e} \sqrt{\frac{ \log(t/\alpha)}{N_{\tilde{a}_{t}^{*},t}}}
\geq
\boldsymbol{b}_{t}^{\top} \boldsymbol{\mu}_{\tilde{a}_{t}^{*}} \\
& \qquad \qquad \implies
2 \boldsymbol{b}_{t}^{\top} \boldsymbol{e} \sqrt{\frac{ \log(t/\alpha)}{N_{i,t}}} 
\geq
\boldsymbol{b}_{t}^{\top} \boldsymbol{\mu}_{\tilde{a}_{t}^{*}} - \boldsymbol{b}_{t}^{\top} \boldsymbol{\mu}_{i}.
\end{aligned}
\]

Combining above result and relaxing the first and second union sets in Eq.~\ref{eq: event_ABC_t_stat1} gives:
\begin{equation}
\begin{aligned}
&{\{a_t = i \neq \tilde{a}_{t}^{*}, N_{i,t} > \beta \}} \\
& \qquad \subset
\left\{ \boldsymbol{b}_{t}^{\top} \hat{\boldsymbol{r}}_{i,t} > \boldsymbol{b}_{t}^{\top} \boldsymbol{\mu}_{i} + \boldsymbol{b}_{t}^{\top} \boldsymbol{e} \sqrt{\frac{ \log(t/\alpha)}{N_{i,t}}} \right\} 
\cup
\left\{ \boldsymbol{b}_{t}^{\top} \hat{\boldsymbol{r}}_{\tilde{a}_{t}^{*},t} < \boldsymbol{b}_{t}^{\top} \boldsymbol{\mu}_{\tilde{a}_{t}^{*}} - \boldsymbol{b}_{t}^{\top} \boldsymbol{e} \sqrt{\frac{ \log(t/\alpha)}{N_{\tilde{a}_{t}^{*},t}}} \right\} \\
& \qquad \qquad \cup
\left\{ \boldsymbol{b}_{t}^{\top} (\boldsymbol{\mu}_{\tilde{a}_{t}^{*}}-\boldsymbol{\mu}_{i}) < 2 \Vert \boldsymbol{b}_{t} \Vert_1 \sqrt{\frac{ \log(t/\alpha)}{N_{i,t}}}, N_{i,t} > \beta \right\}  \\
& \qquad \subset
\underbrace{
\left\{ 
\underset{d \in \mathcal{D}^{+}_{T}}{\cup} \left\{ \boldsymbol{b}_{t}(d) \hat{\boldsymbol{r}}_{i,t}(d) > \boldsymbol{b}_{t}(d) \boldsymbol{\mu}_{i}(d) + \boldsymbol{b}_{t}(d) \sqrt{\frac{ \log(t/\alpha)}{N_{i,t}}} \right\} \right\}}_{A_t} \\
& \qquad \qquad \cup
\underbrace{
\left \{
\underset{d \in \mathcal{D}^{+}_{T}}{\cup} \left\{ \boldsymbol{b}_{t}(d) \hat{\boldsymbol{r}}_{\tilde{a}_{t}^{*},t}(d) < \boldsymbol{b}_{t}(d) \boldsymbol{\mu}_{\tilde{a}_{t}^{*}}(d) - \boldsymbol{b}_{t}(d) \sqrt{\frac{ \log(t/\alpha)}{N_{\tilde{a}_{t}^{*},t}}}\right\}\right\}}_{B_t} \\
& \qquad \qquad \cup
\underbrace{
\left\{ \boldsymbol{b}_{t}^{\top} (\boldsymbol{\mu}_{\tilde{a}_{t}^{*}}-\boldsymbol{\mu}_{i}) < 2 \Vert \hat{\boldsymbol{c}}_t \Vert_1 \sqrt{\frac{ \log(t/\alpha)}{N_{i,t}}}, N_{i,t} > \beta, \boldsymbol{b}_{t}^{\top} \Delta_{i} > \eta_{i} - \epsilon\right\}}_{\Gamma_t},
\end{aligned}
\end{equation}

where $\mathcal{D}^{+}_{T} :=\left\{ d| [\boldsymbol{b}_{1}, \boldsymbol{b}_{2}, ..., \boldsymbol{b}_{T}](d) \in \mathcal{B}^{+}_{T} \right\}$, and $\mathcal{B}^{+}_{T}:= \{{[\boldsymbol{b}_{1}(d), \boldsymbol{b}_{2}(d), ..., \boldsymbol{b}_{T}(d)]} \neq \boldsymbol{0}, \forall d \in [D]\}$ is the collection set of non-zero $[\boldsymbol{b}(d)]^{\top}$ sequence.

Then on event $A_t$, by applying Hoeffding’s Inequality (Lemma~\ref{lemma: Hoeffding}), for any $d \in [D]$, we have

\begin{equation}
\begin{aligned}
\mathbb{P} \left( \boldsymbol{b}_{t}(d) \hat{\boldsymbol{r}}_{i,t}(d) > \boldsymbol{b}_{t}(d) \boldsymbol{\mu}_{i}(d) + \boldsymbol{b}_{t}(d) \sqrt{\frac{ \log(t/\alpha)}{N_{i,t}}}\right) 
& = 
\mathbb{P} \left( \hat{\boldsymbol{r}}_{i,t}(d) - \boldsymbol{\mu}_{i}(d) > \sqrt{\frac{ \log(t/\alpha)}{N_{i,t}}} \right) \\
& \leq
\exp \left( \frac{-2 N_{i,t}^2 \log(t/\alpha) }{ N_{i,t} \sum_{\iota=1}^{N_{i,t}}(1-0)^2 } \right) \\
& =
\exp \left( -2 \log(t/\alpha) \right) 
=
\left( \frac{\alpha}{t} \right)^2,
\end{aligned}
\end{equation}

which yields the upper bound of $\mathbb{P} (A_t)$ as 
\begin{equation}
\begin{aligned}
\label{eq: P_A_t_stat}
\mathbb{P} (A_t) \leq \sum_{d \in \mathcal{D}^{+}_{T}}
\mathbb{P} \left( \boldsymbol{b}_{t}(d) \hat{\boldsymbol{r}}_{i,t}(d) > \boldsymbol{b}_{t}(d) \boldsymbol{\mu}_{i}(d) + \boldsymbol{b}_{t}(d) \sqrt{\frac{ \log(t/\alpha)}{N_{i,t}}}\right) 
& \leq 
|\mathcal{B}^{+}_{T}| \left( \frac{\alpha}{t} \right)^2,
\end{aligned}
\end{equation}

and similarly,
\begin{equation}
\begin{aligned}
\label{eq: P_B_t_stat}
\mathbb{P} (B_t) \leq \sum_{d \in \mathcal{D}^{+}_{T}}
\mathbb{P} \left( \boldsymbol{b}_{t}(d) \hat{\boldsymbol{r}}_{\tilde{a}_{t}^{*},t}(d) < \boldsymbol{b}_{t}(d) \boldsymbol{\mu}_{\tilde{a}_{t}^{*}}(d) - \boldsymbol{b}_{t}(d) \sqrt{\frac{ \log(t/\alpha)}{N_{\tilde{a}_{t}^{*},t}}} \right) 
& \leq 
|\mathcal{B}^{+}_{T}| \left( \frac{\alpha}{t} \right)^2.
\end{aligned}
\end{equation}

Next we investigate the event $\Gamma_t := \left\{ \boldsymbol{b}_{t}^{\top} \Delta_i < 2 \Vert \boldsymbol{b}_{t} \Vert_1 \sqrt{\frac{ \log(t/\alpha)}{N_{i,t}}}, N_{i,t} > \beta \right\}$. 
Let $\beta = \frac{4 M^2 \log (T/\alpha)}{L_i^2}$. Since $N_{i,t} \geq \beta$ and recall that $\boldsymbol{b}_t^{\top} (\boldsymbol{\mu}_{\tilde{a}^{*}_{t}} - \boldsymbol{\mu}_{i}) \geq L_i$, we have,

\begin{equation}
\begin{aligned}
2 \Vert \boldsymbol{b}_t \Vert_1 \sqrt{\frac{ \log(t/\alpha)}{N_{i,t}}}
\leq
2 \Vert \boldsymbol{b}_t \Vert_1 \sqrt{\frac{ \log(t/\alpha)}{\beta}}
\leq
2 M \sqrt{\frac{ \log(T/\alpha)}{\beta}} = L_i 
\leq 
\boldsymbol{b}_t^{\top} (\boldsymbol{\mu}_{\tilde{a}^{*}_{t}} - \boldsymbol{\mu}_{i}), 
\end{aligned}
\end{equation}

implying that the event $\Gamma_t$ has $\mathbb{P}$-probability 0.
By combining Eq.~\ref{eq: term_1_stat} with Eq.~\ref{eq: event_ABC_t_stat1}, \ref{eq: P_A_t_stat} and \ref{eq: P_B_t_stat}, the expectation of LHS term in Eq.~\ref{eq: N_up_bd_stat} can be upper-bounded as follows:

\begin{equation}
\begin{aligned}
\mathbb{E} \left[\sum_{t \in \mathcal{M}^{o}_i} \mathds{1}_{\{a_t = i \} } \right]
& \leq 
\mathbb{E} \left[ \sum_{t=1}^{T} \mathds{1}_{\{a_t = i \neq \tilde{a}^{*}_{t} \}} \right] \\
& \leq 
\frac{4 M^2 \log (T/\alpha)}{L_i^2}
+
|\mathcal{B}^{+}_{T}| \alpha^2 \sum_{t=1}^{T} t^{-2}  \\
& \underset{(a)}{\leq}
\frac{4 M^2 \log (T/\alpha)}{L_i^2}
+
|\mathcal{B}^{+}_{T}| \frac{\pi^2 \alpha^2}{3},
\end{aligned}
\end{equation}

where (a) holds by the convergence of sum of reciprocals of squares that 
\begin{equation}
\label{eq: riemann_zeta}
\sum_{t=1}^{\infty} t^{-2} = \frac{\pi^2}{6}.
\end{equation}
This concludes the proof.
\end{proof}

\subsubsection{Proof of Lemma~\ref{lemma: error_distance}}
\label{sec: proof_lemma_error_distance}

\begin{proof}[Proof of Lemma~\ref{lemma: error_distance}]

Let $\phi_\epsilon$ be the set of solution such that $\boldsymbol{x}^{\top} \Delta = \epsilon$,
$\phi_{\boldsymbol{\overline{c}}^{\top} \Delta}$ be the solution set of $\boldsymbol{x}^{\top} \Delta = \boldsymbol{\overline{c}}^{\top} \Delta$,
i.e., 
\[
\begin{aligned}
\phi_{\epsilon} & :=\left\{ \boldsymbol{x} \mid \boldsymbol{x}^{\top} \Delta = \epsilon\right\}\\
\phi_{\boldsymbol{\overline{c}}^{\top} \Delta} &:=\left\{ \boldsymbol{x} \mid \boldsymbol{x}^{\top} \Delta = \boldsymbol{\overline{c}}^{\top} \Delta \right\},
\end{aligned}
\]

where $\phi_{\epsilon}$ and $\phi_{\boldsymbol{\overline{c}}^{\top} \Delta}$ can be viewed as two hyperplanes share the same normal vector of $\Delta$. Let $\boldsymbol{\overline{c}}_{\phi_{\epsilon}}$ be the projection of vector $\boldsymbol{\overline{c}}$ on hyperplane $\phi_{\epsilon}$. Apparently, $\left( \boldsymbol{\overline{c}}_{\phi_{\epsilon}} - \boldsymbol{\overline{c}} \right) \perp \phi_{\epsilon}$, and thus we have
\begin{equation}
\begin{aligned}
\Vert \boldsymbol{\overline{c}}_{\phi_{\epsilon}} - \boldsymbol{\overline{c}} \Vert_2 = \frac{\boldsymbol{\overline{c}}^{\top} \Delta}{\Vert \Delta \Vert_2}
-
\frac{\epsilon}{\Vert \Delta \Vert_2},
\end{aligned}
\end{equation}

which is also the distance between the parallel hyperplanes $\phi_{\epsilon}$ and $\phi_{ \boldsymbol{\overline{c}}^{\top} \Delta }$.
By the principle of distance between points on parallel hyperplanes, we have for any $\boldsymbol{\hat{c}} \in \phi_{\epsilon}$, the distance between 
$\boldsymbol{\hat{c}}$ and $\boldsymbol{\overline{c}}$ is always greater than or equal to the shortest distance between the hyperplanes $\phi_{\epsilon}$ and $\phi_{\boldsymbol{\overline{c}}^{\top} \Delta }$, i.e.,
\begin{equation}
\begin{aligned}
\Vert \boldsymbol{\hat{c}} - \boldsymbol{\overline{c}} \Vert_2
\geq
\Vert \boldsymbol{\overline{c}}_{\phi_{\epsilon}} - \boldsymbol{\overline{c}} \Vert_2 = \frac{\boldsymbol{\overline{c}}^{\top} \Delta - \epsilon }{\Vert \Delta \Vert_2}
\end{aligned}
\end{equation}
\end{proof}





\subsection{Known Preference as A Special Case }
\label{sec:app_up_bd_stat_known}

\begin{algorithm}[H]
\caption{Preference UCB with Known Preference}
\label{alg:PRUCB_KP}
\begin{algorithmic}
\STATE \textbf{Parameters:} $\alpha$.
\STATE \textbf{Initialization:} 
$N_{i, 1} \!\leftarrow\! 0$; $\boldsymbol{\hat{r}}_{i,1} \!\leftarrow\! [0]^{D}, \forall i \!\in\! [K]$.
\FOR{$t=1,\cdots,T$ }
    \STATE Receive user preference expectation $\boldsymbol{c}_t$
    \STATE $\hat{\boldsymbol{c}}_{t} \leftarrow \boldsymbol{c}_t$
    \STATE Perform one step of PRUCB-UP.
\ENDFOR
\end{algorithmic}
\end{algorithm}

In this section, we use a simple variant of Algorithm \ref{alg:PRUCB_UP} to solve PAMO-MAB with known preference case and show the regret upper-bound. 

Specifically, the known preference environment can be \emph{viewed as a special case} of unknown preference environment, where the preference estimation $\boldsymbol{\hat{c}}_t$ is exactly the user's preference expectation $\boldsymbol{\overline{c}}$ provided before hand. Hence, we can simply replace the estimation value $\boldsymbol{\hat{c}}_t$ in Algorithm \ref{alg:PRUCB_UP} with $\overline{\boldsymbol{c}}$ we obtained in advance. We present this variant in Algorithm \ref{alg:PRUCB_KP}.

Define $\tilde{a}_t = \argmaxA_{k \in \mathcal{A}_t} \boldsymbol{c}_t^{\top} \boldsymbol{\mu}_k$ and $\mathcal{T}_{i} = \{ t \in [T] \mid \tilde{a}_t \neq i \}$ be the set of episodes when $i$ serving as a suboptimal arm conditioned on the given preference $\boldsymbol{c}_t$ over $T$ horizon.
Let 
$\tilde{\Delta}_{i,t} = \mu_{\tilde{a}_t} - \mu_{i} \in \mathbb{R}^D, \forall t \in [1,T]$ 
be the gap of expected rewards between arm $i$ and best $\boldsymbol{c}_t$ conditioned arm $\tilde{a}_t$ at time step $t$, 
$\eta_{i}^{\downarrow} = \min_{t \in \mathcal{T}_{i}}\{\boldsymbol{c}_{t}^{\top} \tilde{\Delta}_{i,t} \}$ and 
$\eta_{i}^{\uparrow} = \max_{t \in \mathcal{T}_{i}}\{\boldsymbol{c}_{t}^{\top} \tilde{\Delta}_{i,t} \}$ refer to the lower and upper bounds of the expected overall-reward gap between $i$ and $a_t^*$ over $T$ when $i$ serving as a suboptimal arm.
For preference known case with Algorithm \ref{alg:PRUCB_KP}, the following corollary of Theorem \ref{theorem:up_bd_stat} characterizes the performance. 

\begin{corollary}
Assume $\overline{\boldsymbol{c}}$ is given before decision making, Algorithm \ref{alg:PRUCB_KP} has 
\[
R(T) 
=
O
\Bigg(
\sum_{i \neq a^{*}}
\Big(
\frac{\delta^2 \eta_{i}^{\uparrow} \log T}{{\eta_{i}^{\downarrow}}^2} 
+
D \pi^2 \alpha^2 \eta_{i}^{\uparrow}
\Big)
\Bigg).
\]
\end{corollary}

\begin{proof}
The proof follows the same path of Theorem \ref{theorem:up_bd_stat} but with some slight modifications.
Let $\tilde{N}_{i,T}$ denotes the number of times that arm $i$ is played as a $\boldsymbol{c}_t$ conditioned suboptimal arm, i.e.,
$\tilde{N}_{i,T} = \sum_{t \in \mathcal{T}_i} \mathds{1}_{\{a_t = i \neq \tilde{a}_t\}}$.

Then we can apply Proposition \ref{prop: N_known_changing} on $\tilde{N}_{i,T}$ for analysis. 
Specifically, by directly substituting $\boldsymbol{b}_t$ with $\boldsymbol{c}_t$, the policy of $a_t$ aligns with that of Algorithm \ref{alg:PRUCB_KP}, and it is easy to verify that $\mathcal{M}_i = \mathcal{T}_i$, $L_i = \eta_{i}^{\downarrow}$. And thus by Proposition \ref{prop: N_known_changing}, we have

\[
\mathbb{E} [\tilde{N}_{i,T}] 
=
\mathbb{E} \left[ \sum_{t \in \mathcal{T}_i} \mathds{1}_{\{a_t = i \} } \right] 
\leq
\frac{4 \delta^2 \log{(\frac{T}{\alpha})}}{\eta_{i}^{\downarrow 2}}
+
\frac{D \pi^2 \alpha^2} {3}.
\]
Thus we have 
\[
\begin{aligned}
R(T) & = \sum_{t\in[T]} \max_{i\in[K]}\mathbb{E}[\boldsymbol{c}_t^{\top} (\boldsymbol{r}_i - \boldsymbol{r}_{a_t} ) ]
\leq 
\sum_{t\in[T]} \mathbb{E}[\max_{i\in[K]} \boldsymbol{c}_t^{\top} (\boldsymbol{r}_i - \boldsymbol{r}_{a_t} ) ]
=
\sum_{i \in [K]} \sum_{t \in \mathcal{T}_{i}}
\mathbb{E} \left[
\mathds{1}_{\{a_t = i \neq \tilde{a}_t\}} \boldsymbol{c}_{t}^{\top} \tilde{\Delta}_{i,t}
\right] \\
& \leq
\sum_{i \in [K]} 
\mathbb{E} [\tilde{N}_{i,T}] \eta_{i}^{\uparrow}
\leq
\sum_{i \in [K]} 
\eta_{i}^{\uparrow} 
( \frac{4 \delta^2 \log{(\frac{T}{\alpha})}}{\eta_{i}^{\downarrow 2}}
+
\frac{D \pi^2 \alpha^2} {3} ).
\end{aligned}
\]

\end{proof}

\section{Supplemental Lemmas}

\begin{lemma}[Hoeffding’s inequality for general bounded random variables~\citep{vershynin2018high} (Theorem
2.2.6)]
\label{lemma: Hoeffding}
 Given independent random variables $\{X_1, ..., X_m \}$ where $a_i \leq X_i \leq b_i$ almost surely (with probability 1) we have:
\[
\mathbb{P} \left(\frac{1}{m} \sum_{i=1}^{m} X_i - \frac{1}{m} \sum_{i=1}^{m} \mathbb{E}[X_i] \geq \epsilon \right) \leq \exp \left(\frac{-2 \epsilon^2 m^2}{\sum_{i=1}^{m} (b_i-a_i)^2} \right).
\]
\end{lemma}

\begin{lemma}[Tower]
\label{lemma: Tower}
If $\mathcal{H} \subseteq \mathcal{G}$ is a $\sigma$-field, then $\mathbb{E} [ \mathbb{E} [ X \mid \mathcal{G}] \mid \mathcal{H}] = \mathbb{E}[X \mid \mathcal{H}]$.
In particular $\mathbb{E}[\mathbb{E}[X \mid \mathcal{G}]] = \mathbb{E}[X]$.
\end{lemma}

\end{document}